\documentclass[11pt]{article}
\usepackage[margin=1in]{geometry}

\usepackage{algorithm,algpseudocode}
\usepackage{url}
\usepackage[utf8]{inputenc} % allow utf-8 input
\usepackage[T1]{fontenc}    % use 8-bit T1 fonts
\usepackage{booktabs}       % professional-quality tables
\usepackage{amsfonts}       % blackboard math symbols
\usepackage{nicefrac}       % compact symbols for 1/2, etc.
\usepackage{microtype}      % microtypography

\usepackage{color}
\usepackage{diagbox}
\usepackage{array,multirow,graphicx}
\usepackage{caption}
\usepackage{mathtools}

\usepackage{subfig}

\usepackage{amsthm,amsmath,amssymb,graphicx}%  
\usepackage{enumitem}
\setlist{nosep}
\usepackage{authblk}

\newtheorem{definition}{Definition}
\newtheorem{lemma}{Lemma}
\newtheorem{corollary}{Corollary}
\newtheorem{proposition}{Proposition}

\newtheorem{theorem}{Theorem}

\def\Do{\text{do}}
\newcommand{\pr}[1]{\mathbb{P}(#1)}
\newcommand{\prr}[1]{\mathbb{P}_r(#1)}
\newcommand{\prg}[1]{\mathbb{P}_g(#1)}

\newcommand{\mat}[1]{\mathbf{#1}}

\newcommand{\expect}[2]{\mathbb{E}_{#1}\left[#2\right]}
\newcommand{\expectd}[1]{\mathbb{E}_{x\sim p_{\text{data}}(x)}\left[#1\right]}
\newcommand{\expectg}[1]{\mathbb{E}_{x\sim p_{\text{g}}(x)}\left[#1\right]}
\newcommand{\expectdz}[1]{\mathbb{E}_{x\sim p_{\text{data}}^0(x)}\left[#1\right]}
\newcommand{\expectdo}[1]{\mathbb{E}_{x\sim p_{\text{data}}^1(x)}\left[#1\right]}
\newcommand{\expectgo}[1]{\mathbb{E}_{x\sim p_{\text{g}}^1(x)}\left[#1\right]}
\newcommand{\expectgz}[1]{\mathbb{E}_{x\sim p_{\text{g}}^0(x)}\left[#1\right]}

\DeclarePairedDelimiterX{\infdivx}[2]{(}{)}{ #1\;\delimsize\|\;#2}
\newcommand{\KL}{KL\infdivx}
\newcommand{\norm}[1]{\left\lVert#1\right\rVert}

\usepackage{soul,xcolor}
\setstcolor{red}

\begin{document} 
\date{\today}
\title{CausalGAN: Learning Causal Implicit Generative Models \\ with Adversarial Training}
\author[1,a]{Murat Kocaoglu \footnote{Equal contribution.}}
\newcommand\CoAuthorMark{\footnotemark[\arabic{footnote}]}
\author[1,b]{Christopher Snyder \protect\CoAuthorMark}
\author[1,c]{Alexandros G. Dimakis}
\author[1,d]{Sriram Vishwanath}
\affil[1]{\small Department of Electrical and Computer Engineering, The University of Texas at Austin, USA}
\affil[ ]{\small \textit \textsuperscript{a} mkocaoglu@utexas.edu \textsuperscript{b}
22csnyder@gmail.com \textsuperscript{c}
 dimakis@austin.utexas.edu \textsuperscript{d} sriram@austin.utexas.edu}

\renewcommand\Authands{ and }

\maketitle

\begin{abstract}
We propose an adversarial training procedure for learning a causal implicit generative model for a given causal graph. We show that adversarial training can be used to learn a generative model with true observational and interventional distributions if the generator architecture is consistent with the given causal graph. We consider the application of generating faces based on given binary labels where the dependency structure between the labels is preserved with a causal graph. This problem can be seen as learning a causal implicit generative model for the image and labels. We devise a two-stage procedure for this problem. First we train a causal implicit generative model over binary labels using a neural network consistent with a causal graph as the generator. We empirically show that Wasserstein GAN can be used to output discrete labels. Later we propose two new conditional GAN architectures, which we call CausalGAN and CausalBEGAN. We show that the optimal generator of the CausalGAN, given the labels, samples from the image distributions conditioned on these labels. The conditional GAN combined with a trained causal implicit generative model for the labels is then an implicit causal generative network over the labels and the generated image. We show that the proposed architectures can be used to sample from observational and interventional image distributions, even for interventions which do not naturally occur in the dataset.
%COMMENT{We demonstrate each model has distinct conditional and interventional image sampling modes that respect the given causal graph. Images generated from interventions on labels remain high-quality, plausible, and label-consistent. For CausalBEGAN we show this remains true even for label combinations realized under intervention that never occur during training.} \st{ We test our framework on the CelebA dataset.}
\end{abstract}

\section{Introduction}
Generative adversarial networks are neural generative models that can be trained using backpropagation to mimick sampling from very high dimensional nonparametric distributions \cite{Goodfellow2014}. A \emph{generator} network models the sampling process through feedforward computation. The generator output is constrained and refined through the feedback by a competitive "adversary network", that attempts to discriminate between the generated and real samples. In the application of sampling from a distribution over images, a generator, typically a neural network, outputs an image given independent noise variables. The objective of the generator is to maximize the loss of the discriminator (convince the discriminator that it outputs images from the real data distribution). GANs have shown tremendous success in generating samples from distributions such as image and video \cite{Vondrick2016nips} and have even been proposed for language translation \cite{Wu2017}. 

One extension idea for GANs is to enable sampling from the class conditional data distributions by feeding labels to the generator. Various neural network architectures have been proposed for solving this problem \cite{Mirza2014,Odena2016,Antipov2017}. As far as we are aware of, in all of these works, the class labels are chosen independently from one another. Therefore, choosing one label does not affect the distribution of the other labels. As a result, these architectures do not provide the functionality to condition on a label, and sample other labels and the image. For concreteness consider a generator trained to output images of birds when given the \emph{color} and \emph{species} labels. On one hand, if we feed the generator 
\emph{color=blue}, since \emph{species} label is independent from the \emph{color} label, we are likely to see blue eagles as well as blue jays. However, we do not expect to see any blue eagles when conditioned on \emph{color=blue} in any dataset of real bird images. Similarly, consider a generator trained to output face images given the \emph{gender} and \emph{mustache} labels. When labels are chosen independently from one another, images generated under \emph{mustache = 1} should contain both males and females, which is clearly different than conditioning on \emph{mustache = 1}. The key for understanding and unifying these two notions, conditioning and being able to sample from distributions different than the dataset's is to use \emph{causality}.

We can think of generating an image conditioned on labels as a causal process: Labels determine the image distribution. The generator is a functional map from labels to image distributions. This is consistent with a simple causal graph \emph{"Labels cause the Image"}, represented with the graph $L\rightarrow G$, where $L$ is the set of labels and $G$ is the generated image. Using a finer model, we can also include the causal graph between the labels. Using the notion of causal graphs, we are interested in extending the previous work on conditional image generation by 
\begin{enumerate}[label=(\roman*)]
\item capturing the dependence and
\item capturing the causal effect
\end{enumerate}
between labels and the image.

As an example, consider the causal graph between gender ($G$) and mustache ($M$) labels. The causal relation is clearly \emph{gender causes mustache}\footnote{In reality, there may be confounder variables, i.e., variables that affect both, which are not observable. In this work, we ignore this effect by assuming the graph has causal sufficiency, i.e., there does not exist unobserved variables that cause more than one observable variable.}, shown with the graph $G\rightarrow M$. Conditioning on \emph{gender=male}, we expect to see males with or without mustaches, based on the fraction of males with mustaches in the population. When we condition on \emph{mustache = 1}, we expect to sample from males only since the population does not contain females with mustaches. In addition to sampling from conditional distributions, causal models allow us to sample from various different distributions called \emph{interventional distributions}, which we explain next.

\begin{figure}[t!]
\centering

\begingroup
\captionsetup{width=0.45\linewidth }

\subfloat[Top: Intervened on Bald=1. Bottom: Conditioned on Bald = 1. $Male\rightarrow Bald$.]{
\includegraphics[width=0.45\linewidth]{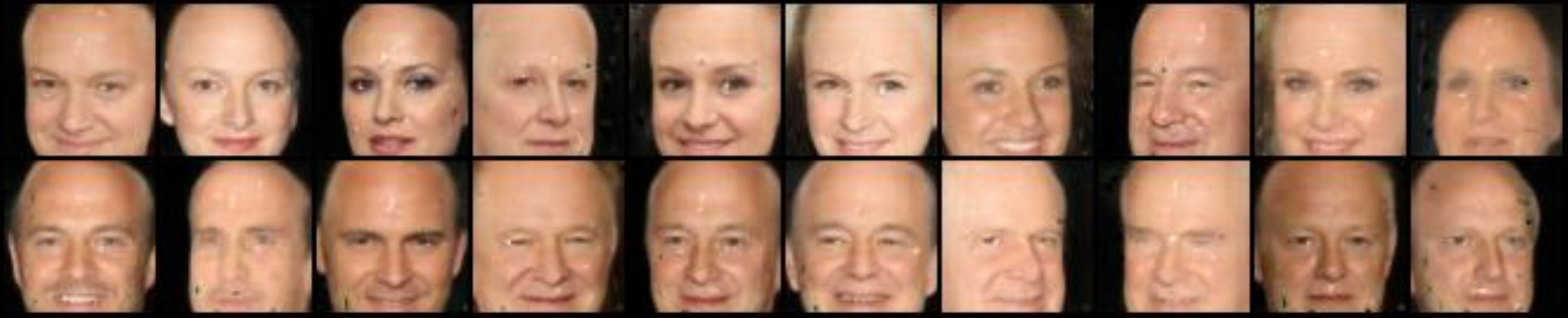}
\label{fig:bald_d1c1}
}
\subfloat[Top: Intervened on Mustache=1. Bottom: Conditioned on  Mustache = 1. $Male\rightarrow Mustache$.]{
\includegraphics[width=0.45\linewidth]{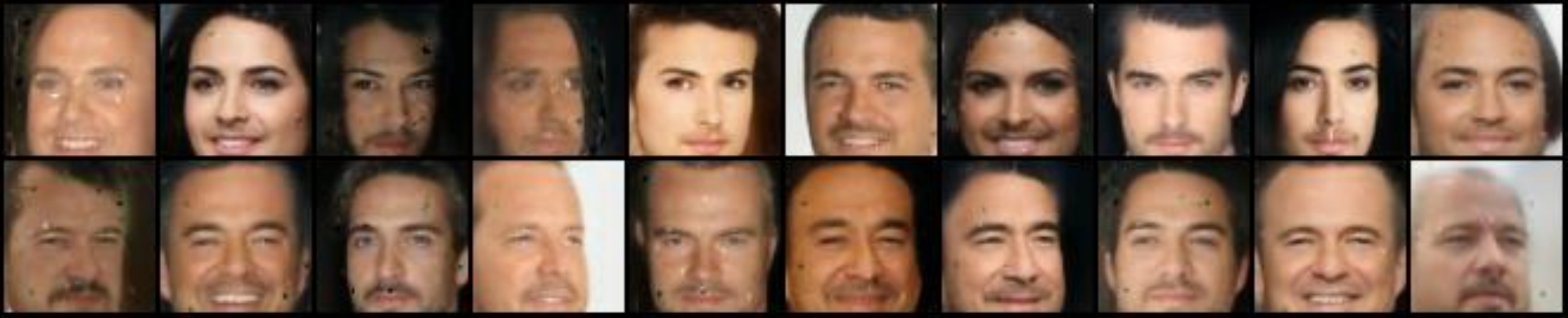}
\label{fig:mustache_d1c1}
}
\endgroup

\caption{Observational and interventional image samples from CausalBEGAN. Our architecture can be used to sample not only from the joint distribution (conditioned on a label) but also from the interventional distribution, e.g., under the intervention do$(Mustache=1)$. The resulting distributions are clearly different, as is evident from the samples outside the dataset, e.g., \emph{females with mustaches}.}
\label{fig:introductory}
\end{figure}

From a causal lens, using independent labels corresponds to using an empty causal graph between the labels. However in practice the labels are not independent and even have clear causal connections (e.g., \emph{gender} causes \emph{mustache}). Using an empty causal graph instead of the true causal graph, and setting a label to a particular value is equivalent to \emph{intervening} on that label in the original causal graph, but also ignoring the way it affects other variables. An intervention is an experiment which fixes the value of a variable, without affecting the rest of the causal mechanism, which is different from conditioning. An intervention on a variable affects its descendant variables in the causal graph. But unlike conditioning, it does not affect the distribution of its ancestors. For example, instead of the causal graph \emph{Gender causes Mustache}, if we used the empty causal graph between the same labels, intervening on \emph{Gender = Female} would create females with mustaches, whereas with the correct causal graph, it should only yield females without mustaches since setting  the \emph{Gender} variable will affect all the variables that are downstream, e.g., \emph{mustache}. See Figure \ref{fig:introductory} for a sample of our results which illustrate this concept on the \emph{bald} and \emph{mustache} variables. Similarly, for generating birds with the causal graph \emph{Species causes color}, intervening on \emph{color = blue} allows us to sample blue eagles (which do not exist) whereas conditioning on \emph{color = blue} does not.

An \textit{implicit} generative model \cite{Mohamed2016} is a mechanism that can sample from a probability distribution but cannot provide likelihoods for data points. In this work we propose \textit{causal implicit} generative models (CiGM): mechanisms that can sample not only from probability distributions but also from \textit{conditional} and \textit{interventional} distributions.  We show that when the generator structure inherits its neural connections from the causal graph, GANs can be used to train causal implicit generative models. We use WassersteinGAN to train a causal implicit generative model for image labels, as part of a two-step procedure for training a causal implicit generative model for the images and image labels. For the second step, we propose a novel conditional GAN architecture and loss function called the CausalGAN. We show that the optimal generator can sample from the correct conditional and interventional distributions, which is summarized by the following theorem.
\begin{theorem}[Informal]
\label{thm:informal}
Let $G(l,z)$ be the output of the generator for a given label $l$ and latent vector $z$. Let $G^*$ be the global optimal generator for the loss function in (\ref{eq:gen_loss}), when the rest of the network is trained to optimality. Then the generator samples from the conditional image distribution given the label, i.e., $p_g(G(l,Z)=x) = p_{data}{(X=x|L=l)}$, where $p_{data}$ is the data probability density function over the image and the labels, $p_g$ is the probability density function induced by the random variable $Z$, and $X$ is the image random variable.
\end{theorem}

The following corollary states that the trained causal implicit generative model for the labels concatenated with CausalGAN is a causal implicit generative model for the labels and image.
\begin{corollary}
Suppose $C:\mathcal{Z}_1\rightarrow \mathcal{L}$ is a causal implicit generative model for the causal graph $D=(\mathcal{V},E)$ where $\mathcal{V}$ is the set of image labels and the observational joint distribution over these labels is strictly positive. Let $G:\mathcal{L}\times \mathcal{Z}_2\rightarrow \mathcal{I}$ be the class conditional generator that can sample from the image distribution conditioned on the given label combination $L\in \mathcal{L}$. Then $G(C(Z_1),Z_2)$ is a causal implicit generative model for the causal graph $D' = (\mathcal{V} \cup \{Image\}, E \cup \{(V_1,Image),(V_2,Image),\hdots, (V_n,Image)\})$.
\end{corollary}
In words, the corollary states the following: Consider a causal graph $D'$ on the image \emph{labels} and the \emph{image} variable, where every \emph{label} causes the \emph{image}. Then combining an implicit causal generative model for the induced subgraph on the \emph{labels} with a conditional generative model for the \emph{image} given the \emph{labels} yields a causal implicit generative model for $D'$.

Our contributions are as follows:
\begin{itemize}
\item We observe that adversarial training can be used after simply structuring the generator architecture based on the causal graph to train a causal implicit generative model.
\item We empirically show how simple GAN training can be adapted using WassersteinGAN to learn a graph-structured generative model that outputs \emph{essentially discrete}\footnote{Each of the generated labels are sharply concentrated around 0 and 1.} labels. 
\item We consider the problem of conditional and interventional sampling of images given a causal graph over binary labels. We propose a two-stage procedure to train a causal implicit generative model over the binary labels and the image. As part of this procedure, we propose a novel conditional GAN architecture and loss function. We show that the global optimal generator\footnote{Global optimal after the remaining network is trained to optimality.} provably samples from the class conditional distributions. 
\item We propose a natural but nontrivial extension of BEGAN to accept labels: using the same motivations for margins as in BEGAN \cite{Berthelot2017}, we arrive at a "margin of margins" term, which cannot be neglected. We show empirically that this model, which we call CausalBEGAN, produces high quality images that capture the image labels.
\item We evaluate our causal implicit generative model training framework on the labeled CelebA data \cite{liu2015faceattributes}. We show that the combined architecture generates images that can capture both the observational and interventional distributions over images and labels jointly \footnote{Our code is available at \url{https://github.com/mkocaoglu/CausalGAN}}. We show the surprising result that CausalGAN and CausalBEGAN can produce high-quality label-consistent images \emph{even for label combinations realized under interventions that never occur during training}, e.g., "woman with mustache". 
\end{itemize}

\section{Related Work}
Using a generative adversarial network conditioned on the image labels has been proposed before: In \cite{Mirza2014}, authors propose to extend generative adversarial networks to the setting where there is extra information, such as labels. The label of the image is fed to both the generator and the discriminator. This architecture is called conditional GAN. In \cite{Chen2016}, authors propose a new architecture called InfoGAN, which attempts to maximize a variational lower bound of mutual information between the labels given to the generator and the image. In \cite{Odena2016}, authors propose a new conditional GAN architecture, which performs well on higher resolution images. A class label is given to the generator. Image from the dataset is also chosen conditioned on this label. In addition to deciding if the image is real or fake, the discriminator has to also output an estimate of the class label.

Using causal principles for deep learning and using deep learning techniques for causal inference has been recently gaining attention. In \cite{Paz2016}, authors observe the connection between conditional GAN layers, and structural equation models. Based on this observation, they use CGAN \cite{Mirza2014} to learn the causal direction between two variables from a dataset. In \cite{Paz2017}, the authors propose using a neural network in order to discover the causal relation between image class labels based on static images. In \cite{Bahadori2017}, authors propose a new regularization for training a neural network, which they call causal regularization, in order to assure that the model is predictive in a causal sense. In a very recent work \cite{Besserve2017}, authors point out the connection of GANs to causal generative models. However they see image as a cause of the neural net weights, and do not use labels.

BiGAN \cite{Donahue2016} and ALI \cite{Dumoulin2017} improve the standard GAN framework to provide the functionality of learning the mapping from image space to latent space. In CoGAN \cite{Liu2016} the authors learn a joint distribution given samples from marginals by enforcing weight sharing between generators. This can, for example, be used to learn the joint distribution between image and labels. It is not clear, however, if this approach will work when the generator is structured via a causal graph. SD-GAN \cite{Donahue2017} is an architecture which splits the latent space into "Identity" and "Observation" portions. To generate faces of the same person, one can then fix the identity portion of the latent code. This works well for datasets where each identity has multiple observations. Authors in \cite{Antipov2017} use conditional GAN of \cite{Mirza2014} with a one-hot encoded vector that encodes the age interval. A generator conditioned on this one-hot vector can then be used for changing the age attribute of a face image. Another application of generative models is in compressed sensing: Authors in \cite{Bora2017} give compressed sensing guarantees for recovering a vector, if the data lies close to the output of a trained generative model.

\section{Background}
\label{sec:background}
\subsection{Causality Basics}
In this section, we give a brief introduction to causality. Specifically, we use Pearl's framework \cite{Pearl2009}, i.e., structural causal models, which uses structural equations and directed acyclic graphs between random variables to represent a causal model. We explain how causal principles apply to our framework through examples. For a more detailed treatment of the subject with more of the technical details, see \cite{Pearl2009}.

Consider two random variables $X, Y$. Within the structural causal modeling framework and under the causal sufficiency assumption\footnote{In a causally sufficient system, every unobserved variable affects no more than a single observed variable.}, $X$ \emph{causes} $Y$ simply means that there exists a function $f$ and some unobserved random variable $E$, independent from $X$, such that $Y = f(X,E)$. Unobserved variables are also called \emph{exogenous}. The causal graph that represents this relation is $X\rightarrow Y $. In general, a causal graph is a directed acyclic graph implied by the structural equations: The parents of a node in the causal graph represent the \emph{causes} of that variable. The causal graph can be constructed from the structural equations as follows: The parents of a variable are those that appear in the structural equation that determines the value of that variable.

Formally, a structural causal model is a tuple $\mathcal{M} = (\mathcal{V}, \mathcal{E}, \mathcal{F},\mathcal{P}_E(.))$ that contains a set of functions $\mathcal{F} = \{f_1,f_2,\hdots,f_n\}$, a set of random variables $V = \{X_1,X_2,\hdots, X_n\}$, a set of exogenous random variables $\mathcal{E}=\{E_1,E_2,\hdots, E_n\}$, and a probability distribution over the exogenous variables $\mathcal{P}_{\mathcal{E}}$ \footnote{The definition provided here assumes causal sufficiency, i.e., there are no exogenous variables that affect more than one observable variable. Under causal sufficiency, Pearl's model assumes that the distribution over the exogenous variables is a product distribution, i.e., exogenous variables are mutually independent.}. The set of observable variables $\mathcal{V}$ has a joint distribution implied by the distributions of $\mathcal{E}$, and the functional relations $\mathcal{F}$. This distribution is the projection of $\mathcal{P}_{\mathcal{E}}$ onto the set of variables $\mathcal{V}$ and is shown by $\mathcal{P}_\mathcal{V}$. The causal graph $D$ is then the directed acyclic graph on the nodes $\mathcal{V}$, such that a node $X_j$ is a parent of node $X_i$ if and only if $X_j$ is in the domain of $f_i$, i.e., $X_i = f_i(X_j, S, E_i)$, for some $S\subset V$. The set of parents of variable $X_i$ is shown by $Pa_i$. $D$ is then a Bayesian network for the induced joint probability distribution over the observable variables $\mathcal{V}$. We assume causal sufficiency: Every exogenous variable is a direct parent of at most one observable variable.

An \emph{intervention}, is an operation that changes the underlying causal mechanism, hence the corresponding causal graph. An intervention on $X_i$ is denoted as $do(X_i=x_i)$. It is different from conditioning on $X_i=x$ in the following way: An intervention removes the connections of node $X_i$ to its parents, whereas conditioning does not change the causal graph from which data is sampled. The interpretation is that, for example, if we set the value of $X_i$ to 1, then it is no longer determined through the function $f_i(Pa_i, E_i)$. An intervention on a set of nodes is defined similarly. The joint distribution over the variables after an intervention (post-interventional distribution) can be calculated as follows: Since $D$ is a Bayesian network for the joint distribution, the observational distribution can be factorized as $P(x_1,x_2,\hdots x_n) = \prod_{i\in [n]} \Pr(x_i|Pa_i)$, where the nodes in $Pa_i$ are assigned to the corresponding values in $\{x_i\}_{i\in [n]}$. After an intervention on a set of nodes $X_S \coloneqq \{X_i\}_{i\in S}$, i.e., $do(X_S = \mat{s})$, the post-interventional distribution is given by $\prod_{i \in [n]\backslash S }\Pr(x_i|Pa_i^S)$, where $Pa_i^S$ is the shorthand notation for the following assignment: $X_j = x_j $ for $X_j\in Pa_i$ if $j \notin S$ and $X_j = \mat{s}(j)$ if $j\in S$\footnote{With slight abuse of notation, we use $\mat{s}(j)$ to represent the value assigned to variable $X_j$ by the intervention rather than the $j$th coordinate of $\mat{s}$}.

In general it is not possible to identify the true causal graph for a set of variables without performing experiments or making additional assumptions. This is because there are multiple causal graphs that lead to the same joint probability distribution even for two variables \cite{Spirtes2001}. This paper does not address the problem of learning the causal graph: We assume the causal graph is given to us, and we learn a causal model, i.e.,  the functions and the distributions of the exogenous variables comprising the structural equations\footnote{Even when the causal graph is given, there will be many different sets of functions and exogenous noise distributions that explain the observed joint distribution for that  causal graph. We are learning one such model.}. There is significant prior work on learning causal graphs that could be used before our method, see e.g. \cite{EberhardtThesis, Hoyer2008, Hyttinen2013, hauser2014two, Shanmugam2015, Paz2015a, Etesami2016, Quinn2015, Kontoyiannis2016, Kocaoglu2017, Kocaoglu2017a}. When the true causal graph is unknown we can use any feasible graph, i.e., any Bayesian network that respects the conditional independencies present in the data. If only a few conditional independencies are known, a richer model (i.e., a denser Bayesian network) can be used, although a larger number of functional relations should be learned in that case. We explore the effect of the used Bayesian network in Section \ref{sec:results}. If the used Bayesian network has edges that are inconsistent with the true causal graph, our conditional distributions will be correct, but the interventional distributions will be different.

\section{Causal Implicit Generative Models}
\label{sec:implicit}
Implicit generative models \cite{Mohamed2016} are used to sample from a probability distribution without an explicit parameterization. Generative adversarial networks are arguably one of the most successful examples of implicit generative models. Thanks to an adversarial training procedure, GANs are able to produce realistic samples from distributions over a very high dimensional space, such as images. To sample from the desired distribution, one samples a vector from a known distribution, such as Gaussian or uniform, and feeds it into a feedforward neural network which was trained on a given dataset. Although implicit generative models can sample from the data distribution, they do not provide the functionality to sample from interventional distributions. \emph{Causal implicit generative models} provide a way to sample from both observational and interventional distributions. 

We show that generative adversarial networks can also be used for training causal implicit generative models. Consider the simple causal graph $X\rightarrow Z\leftarrow Y$. Under the causal sufficiency assumption, this model can be written as $X = f_X(N_X), Y = f_Y(N_Y), Z = f_Z(X,Y,N_Z)$, where $f_X,f_Y,f_Z$ are some functions and $N_X, N_Y, N_Z$ are jointly independent variables. The following simple observation is useful: \emph{In the GAN training framework, generator neural network connections can be arranged to reflect the causal graph structure}. Consider Figure \ref{fig:XYZ}. The feedforward neural networks can be used to represent the functions $f_X,f_Y,f_Z$. The noise terms can be chosen as independent, complying with the condition that $(N_X,N_Y,N_Z)$ are jointly independent. Hence this feedforward neural network can be used to represents the causal graph $X\rightarrow Z\leftarrow Y$ if $f_X, f_Y, f_Z$ are within the class of functions that can be represented with the given family of neural networks.

The following proposition is well known in the causality literature. It shows that given the true causal graph, two causal models that have the same observational distribution have the same interventional distribution for any intervention.
\begin{proposition}
\label{prop:causal}
Let $\mathcal{M}_1 = (D_1=(V,E), N_1, \mathcal{F}_1,\mathcal{P}_{N_1}(.)), \mathcal{M}_2 = (D_2=(V,E), N_2, \mathcal{F}_2,\mathcal{Q}_{N_2}(.))$ be two causal models. If $\mathcal{P}_V(.) = \mathcal{Q}_V(.)$, then $\mathcal{P}_V(.|\Do (S)) = \mathcal{Q}_V(. | \Do (S))$
\end{proposition}
\begin{proof}
Note that $D_1$ and $D_2$ are the same causal Bayesian networks \cite{Pearl2009}. Interventional distributions for causal Bayesian networks can be directly calculated from the conditional probabilities and the causal graph. Thus, $\mathcal{M}_1$ and $\mathcal{M}_2$ have the same interventional distributions.
\end{proof}

We have the following definition, which ties a feedforward neural network with a causal graph:
\begin{definition}
\label{def:consistent}
Let $Z = \{Z_1, Z_2, \hdots, Z_m\}$ be a set of mutually independent random variables. A feedforward neural network $G$ that outputs the vector $G(Z) = [G_1(Z), G_2(Z), \hdots, G_n(Z)]$ is called \textbf{consistent} with a causal graph $D=([n],E)$, if $\forall i \in [n]$, $\exists$ a set of layers $f_i$ such that $G_i(Z)$ can be written as $G_i(Z) = f_i(\{G_j(Z)\}_{j\in Pa_i}, Z_{S_i})$, where $Pa_i$ are the set of parents of $i$ in $D$, and $Z_{S_i} \coloneqq \{Z_j\ : j\in S_i\}$ are collections of subsets of $Z$ such that $\{S_i:i\in [n]\}$ is a partition of [m]. 
\end{definition}
\begin{figure}[t]
\centering
\subfloat[Standard generator architecture and the causal graph it represents]
{\label{fig:causal_interpretation}
\includegraphics[width=0.45\linewidth]{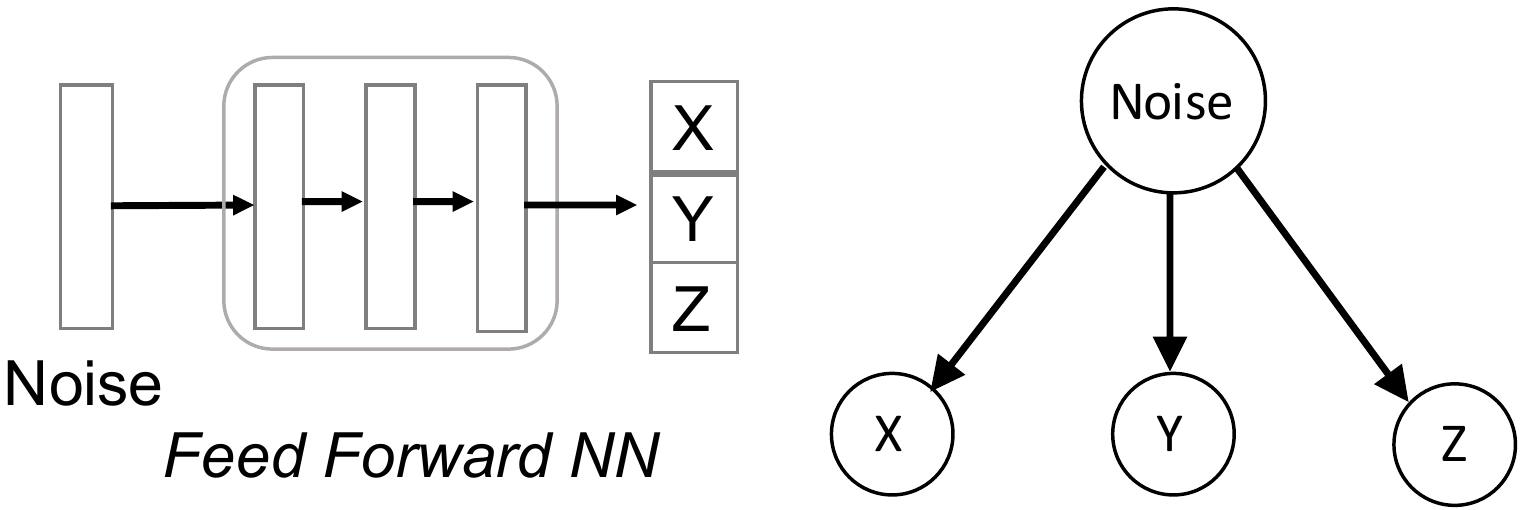}}
\hspace{0.25in}
\subfloat[Generator neural network architecture that represents the causal graph $X\rightarrow Z \leftarrow Y$]
{\label{fig:XYZ}
\includegraphics[width=0.23\linewidth]{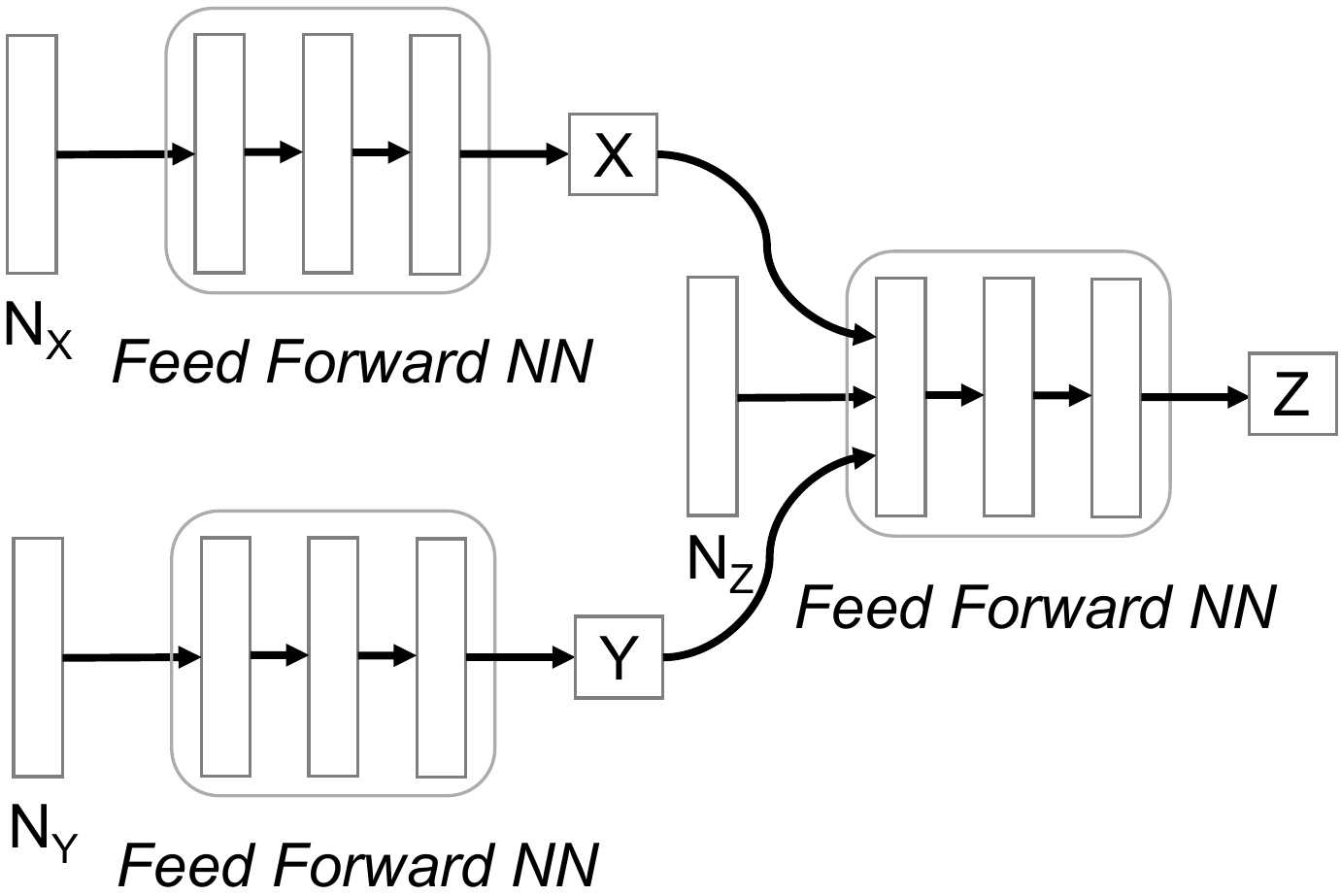}}
\caption{(a) The causal graph implied by the standard generator architecture, feedforward neural network. (b) A neural network implementation of the causal graph $X\rightarrow Z \leftarrow Y$: Each feed forward neural net captures the function $f$ in the structural equation model $V = f(Pa_V, E)$. }
\end{figure}
Based on the definition, we say a feedforward neural network $G$ with output 
\begin{equation}
G(Z) = [G_1(Z), G_2(Z), \hdots, G_n(Z)],
\end{equation} 
is a causal implicit generative model for the causal model $\mathcal{M}=(D=([n],E), N, \mathcal{F},\mathcal{P}_{N}(.))$ if $G$ is consistent with the causal graph $D$ and $\Pr(G(Z) = \mat{x}) = \mathcal{P}_V(\mat{x}), \forall \mat{x}$.

We propose using adversarial training where the generator neural network is consistent with the causal graph according to Definition \ref{def:consistent}. This notion is illustrated in Figure \ref{fig:XYZ}. 

\section{Causal Generative Adversarial Networks}
\label{sec:causalGAN}
Causal implicit generative models can be trained given a causal graph and samples from a joint distribution. However, for the application of image generation with binary labels, we found it difficult to simultaneously learn the joint label and image distribution \footnote{Please see the Appendix for our primitive result using this naive attempt.}. For these applications, we focus on dividing the task of learning a causal implicit generative causal model into two subtasks: First, learn the causal implicit generative model over a small set of variables. Then, learn the remaining set of variables conditioned on the first set of variables using a conditional generative network. For this training to be consistent with the causal structure, every node in the first set should come before any node in the second set with respect to the partial order of the causal graph. We assume that the problem of generating images based on the image labels inherently contains a causal graph similar to the one given in Figure \ref{fig:image_gen}, which makes it suitable for a two-stage training: First, train a generative model over the labels, then train a generative model for the images conditioned on the labels. As we show next, our new architecture and loss function (CausalGAN) assures that the optimum generator outputs the label conditioned image distributions. Under the assumption that the joint probability distribution over the labels is strictly positive\footnote{This assumption does not hold in the CelebA dataset: $\Pr{(Male=0,Mustache=1)}=0$. However, we will see that the trained model is able to extrapolate to these interventional distributions when the CausalGAN model is not trained for very long.}, combining pretrained causal generative model for labels with a label-conditioned image generator gives a causal implicit  generative model for images. The formal statement for this corollary is postponed to Section \ref{sec:theoretical}.

\begin{figure}[t]
\centering
\includegraphics[width=0.41\linewidth]{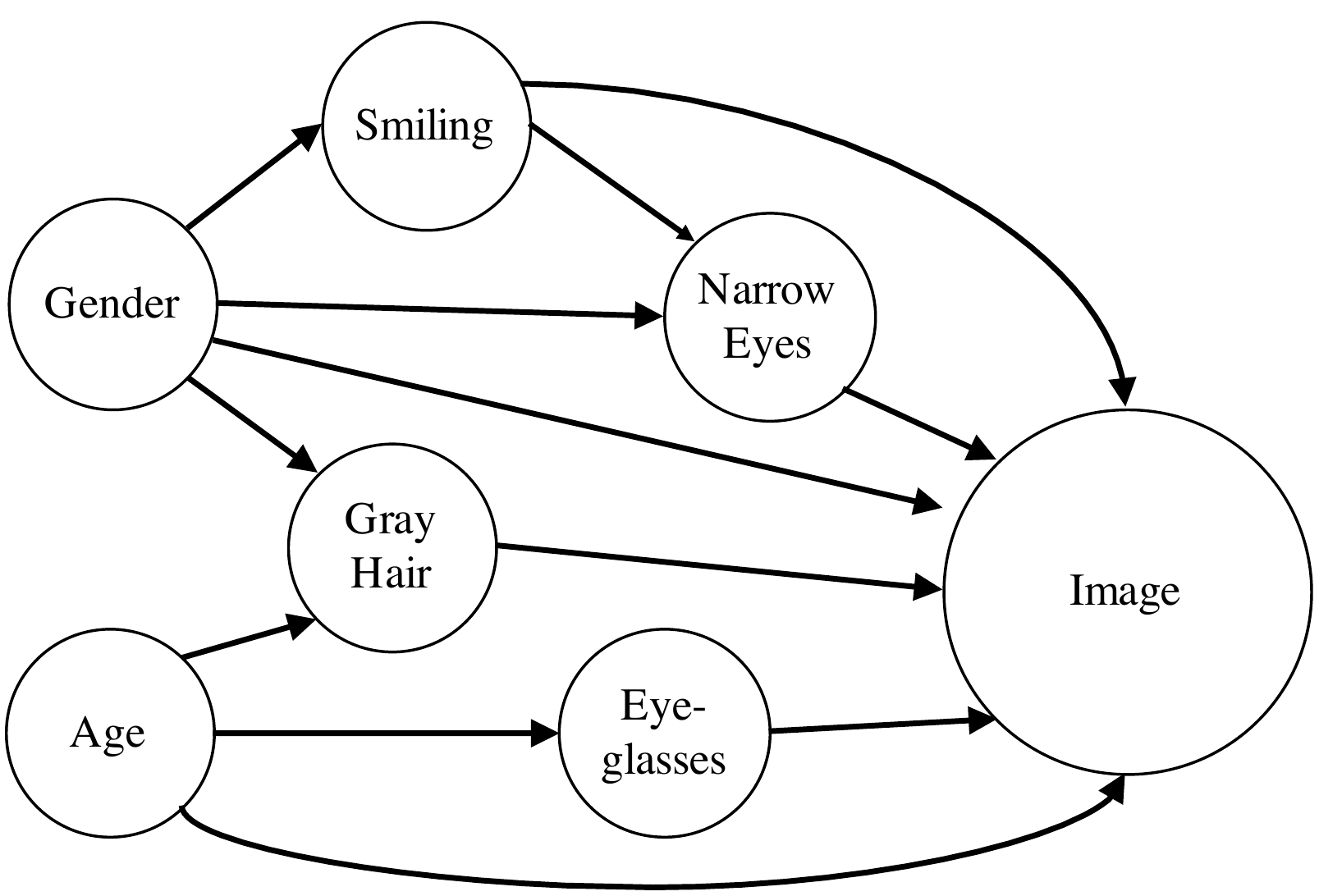}
\caption{A plausible causal model for image generation.}
\label{fig:image_gen}
\end{figure}
\subsection{Causal Implicit Generative Model for Binary Labels}
\label{subsec:cc}
Here we describe the adversarial training of a causal implicit generative model for binary labels. This generative model, which we call the \emph{Causal Controller}, will be used for controlling which distribution the images will be sampled from when intervened or conditioned on a set of labels. As in Section \ref{sec:implicit}, we structure the Causal Controller network to sequentially produce labels according to the causal graph.

Since our theoretical results hold for binary labels, we prefer a generator which can sample from an essentially discrete label distribution \footnote{Ignoring the theoretical considerations, adding noise to transform the labels artificially into continuous targets also works. However we observed better empirical convergence with this technique.}. However, the standard GAN training is not suited for learning a discrete distribution due to the properties of Jensen-Shannon divergence. To be able to sample from a discrete distribution, we employ WassersteinGAN \cite{Arjovsky2017}. We used the model of \cite{Gulrajani2017}, where the Lipschitz constraint on the gradient is replaced by a penalty term in the loss. 

\subsection{CausalGAN Architecture}
\label{causalgan:architecture}
As part of the two-step process proposed in Section \ref{sec:implicit} of learning a causal implicit generative model over the labels \emph{and} the image variables, we design a new conditional GAN architecture to generate the images based on the labels of the Causal Controller. Unlike previous work, our new architecture and loss function assures that the optimum generator outputs the label conditioned image distributions. We use a pretrained Causal Controller which is not further updated.

\textbf{Labeler and Anti-Labeler:}
We have two separate labeler neural networks. \emph{The Labeler} is trained to estimate the labels of images in the dataset. \emph{The Anti-Labeler} is trained to estimate the labels of the images which are sampled from the generator. The label of a generated image is the label produced by the Causal Controller.

\begin{figure}[t]
\centering
\includegraphics[width=0.75\linewidth]{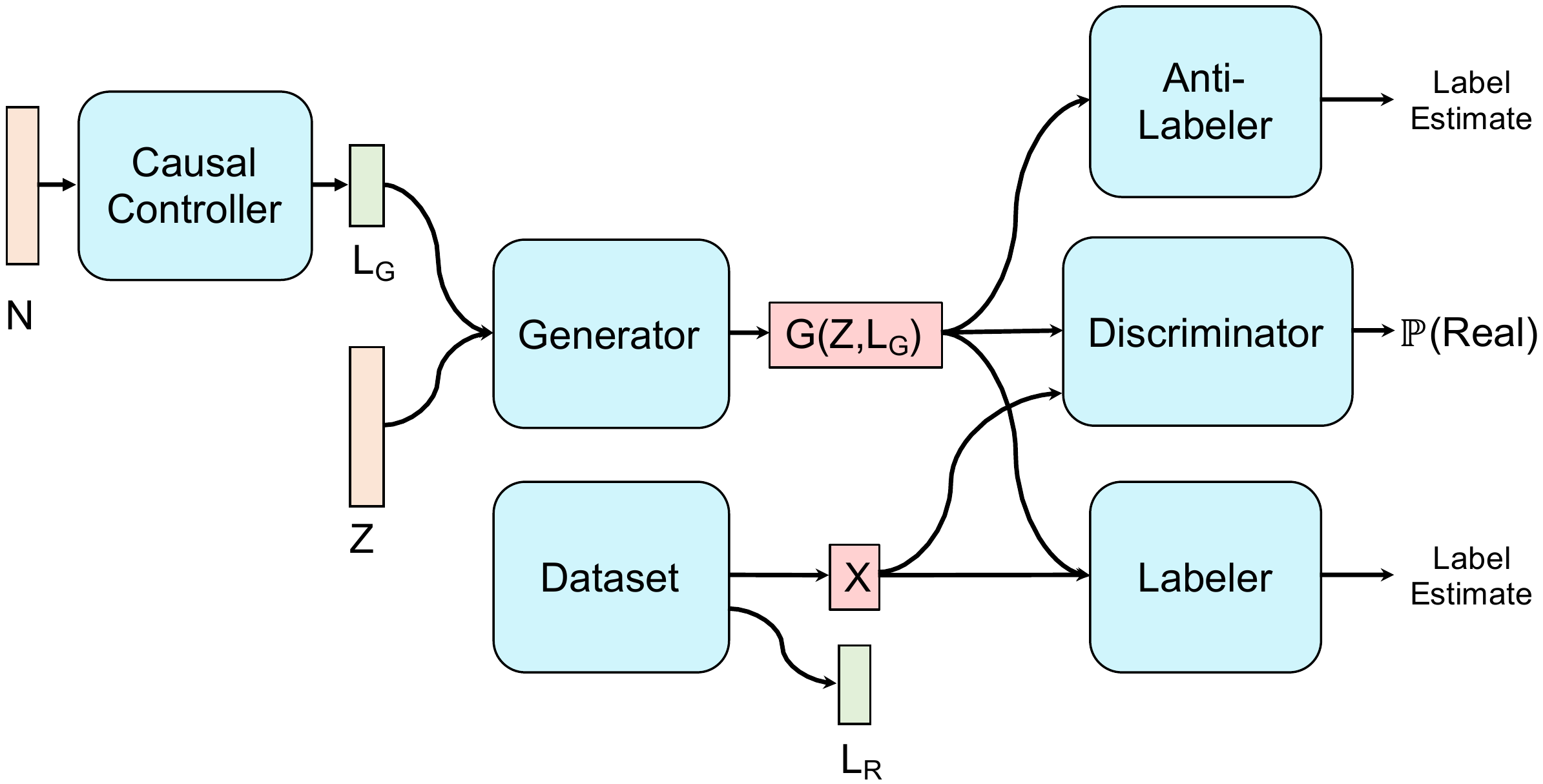}
\caption{CausalGAN architecture.}
\label{fig:architecture}
\end{figure}

\textbf{Generator:} The objective of the generator is 3-fold: producing realistic images by competing with the discriminator, capturing the labels it is given in the produced images by minimizing the Labeler loss, and avoiding drifting towards unrealistic image distributions that are easy to label by maximizing the Anti-Labeler loss. For the optimum Causal Controller, Labeler, and Anti-Labeler, we will later show that the optimum generator samples from the same distribution as the class conditional images.

The most important distinction of CausalGAN with the existing conditional GAN architectures is that it uses an Anti-Labeler network in addition to a Labeler network. Notice that the theoretical guarantee we develop in Section \ref{sec:theoretical} does not hold when Anti-Labeler network is not used. Intuitively, the Anti-Labeler loss discourages the generator network to generate only few typical faces for a fixed label combination. This is a phenomenon that we call \emph{label-conditioned mode collapse}. In the literature, minibatch-features are one of the most popular techniques used to avoid mode-collapse \cite{Salimans2016}. However, the diversity within a batch of images due to different label combinations can make this approach ineffective for combatting label-conditioned mode collapse. We observe that this intuition carries over to practice.

\subsubsection*{Loss Functions}
We present the results for a single binary label $l$. For the more general case of $d$ binary labels, we have an extension where the labeler and the generator losses are slightly modified. We explain this extension in the supplementary material in Section \ref{sec:generalized_proof} along with the proof that the optimal generator samples from the class conditional distribution given the $d-$dimensional label vector. Let $\pr{l=1}=\rho$. We use $p_{g}^0(x) \coloneqq \pr{G(z,l) = x | l=0}$ and $p_{g}^1(x) \coloneqq \pr{G(z,l) = x | l=1}$. $G(.), D(.), D_{LR}(.)$, and $D_{LG}(.)$ are the mappings due to generator, discriminator, Labeler, and Anti-Labeler respectively.

The generator loss function of CausalGAN contains label loss terms, the GAN loss in \cite{Goodfellow2014}, and an added loss term due to the discriminator. With the addition of this term to the generator loss, we will be able to prove that the optimal generator outputs the class conditional image distribution. This result will also be true for multiple binary labels.

For a fixed generator, Anti-Labeler solves the following optimization problem:
\begin{equation}
\label{eq:labeler_g_loss}
\max\limits_{D_{LG}} \rho\expectgo{\log(D_{LG}(x))} + (1-\rho)\expectgz{\log(1-D_{LG}(x)}.
\end{equation}

\noindent
The Labeler solves the following optimization problem:
\begin{equation}
\label{eq:labeler_r_loss}
\max\limits_{D_{LR}} \rho\expectdo{\log(D_{LR}(x))} + (1-\rho)\expectdz{\log(1-D_{LR}(x)}.
\end{equation}

\noindent
For a fixed generator, the discriminator solves the following optimization problem:
\begin{align}
\label{eq:discLoss}
\max\limits_{D} \expectd{\log(D(x))} + \expectg{\log\left(\frac{1-D(x)}{D(x)}\right)}.
\end{align}

\noindent
For a fixed discriminator, Labeler and Anti-Labeler, generator solves the following optimization problem:
\begin{align}
\label{eq:gen_loss}
&\min\limits_{G} \expectd{\log(D(x))} + \expectg{\log\left(\frac{1-D(x)}{D(x)}\right)} \nonumber\\
&- \rho\expect{x\sim p_g^1(x)}{ \log(D_{LR}(X))} - (1-\rho)\expect{x\sim p_g^0(x)}{\log(1-D_{LR}(X))}\nonumber \\
+& \rho \expect{x\sim p_g^1(x)}{  \log(D_{LG}(X))} + (1-\rho) \expect{x\sim p_g^0(x)}{ \log(1-D_{LG}(X))}.
\end{align}

\noindent
\textbf{Remark:} Although the authors in \cite{Goodfellow2014} have the additive term $\expectg{\log(1-D(X))}$ in the definition of the loss function, in practice they use the term $\expectg{-\log(D(X))}$. It is interesting to note that this is the extra loss terms we need for the global optimum to correspond to the class conditional image distributions under a label loss.

\subsection{CausalBEGAN Architecture}
\label{causalbegan:architecture}
In this section, we propose a simple, but non-trivial extension of BEGAN where we feed image labels to the generator. One of the central contributions of BEGAN \cite{Berthelot2017} is a control theory-inspired boundary equilibrium approach that encourages generator training only when the discriminator is near optimum and its gradients are the most informative. The following observation helps us carry the same idea to the case with labels: Label gradients are most informative when the image quality is high. Here, we introduce a new loss and a set of margins that reflect this intuition.

Formally, let $\mathcal{L}(x)$ be the average $L_1$ pixel-wise autoencoder loss for an image $x$, as in BEGAN. Let $\mathcal{L}_{sq}(u,v)$ be the squared loss term, i.e., $\norm{u-v}_2^2$. Let $(x,l_x)$ be a sample from the data distribution, where $x$ is the image and $l_x$ is its corresponding label. Similarly, $G(z,l_g)$ is an image sample from the generator, where $l_g$ is the label used to generate this image. Denoting the space of images by $\mathcal{I}$, let $G:\mathbb{R}^n \times  \{0,1\}^m \mapsto \mathcal{I}$ be the generator. As a naive attempt to extend the original BEGAN loss formulation to include the labels, we can write the following loss functions:
\begin{align}
Loss_D &= \mathcal{L}(x) - \mathcal{L}(Labeler(G(z,l))) + \mathcal{L}_{sq}(l_x,Labeler(x)) - \mathcal{L}_{sq}(l_g,Labeler(G(z,l_g))), \nonumber\\
Loss_G &= \mathcal{L}(G(z,l_g)) + \mathcal{L}_{sq}(l_g,Labeler(G(z,l_g))). \label{eq:naiveBEGAN}
\end{align}

However, this naive formulation does not address the use of margins, which is extremely critical in the BEGAN formulation. Just as a better trained BEGAN discriminator creates more useful gradients for image generation, a better trained Labeler is a prerequisite for meaningful gradients. This motivates an additional margin-coefficient tuple $(b_2,c_2)$, as shown in (\ref{eq:b_2},\ref{eq:c_2}). 

The generator tries to jointly minimize the two loss terms in the formulation in (\ref{eq:naiveBEGAN}). We empirically observe that occasionally the image quality will suffer because the images that best exploit the Labeler network are often not obliged to be realistic, and can be noisy or misshapen. Based on this, label loss seems unlikely to provide useful gradients unless the image quality remains good. Therefore we encourage the generator to incorporate label loss only when the \emph{image quality margin} $b_1$ is large compared to the \emph{label margin} $b_2$. To achieve this, we introduce a new \emph{ margin of margins} term, $b_3$. As a result, the margin equations and update rules are summarized as follows, where $\lambda_1,\lambda_2,\lambda_3$ are learning rates for the coefficients.
\begin{align}
b_1 &= \gamma_1*\mathcal{L}(x) - \mathcal{L}(G(z,l_g)).\nonumber\\
b_2 &= \gamma_2*\mathcal{L}_{sq}(l_x,Labeler(x)) - \mathcal{L}_{sq}(l_g,Labeler(G(z,l_g))).\label{eq:b_2}\\
b_3 &= \gamma_3*relu(b_1) - relu(b_2).\nonumber\\
c_1 &\gets clip_{[0,1]}(c_1 + \lambda_1*b_1).\nonumber\\
c_2 &\gets clip_{[0,1]}(c_2 + \lambda_2*b_2).\label{eq:c_2}\\
c_3 &\gets clip_{[0,1]}(c_3+ \lambda_3*b_3).\nonumber\\
Loss_D &= \mathcal{L}(x) - c_1*\mathcal{L}(G(z,l_g)) + \mathcal{L}_{sq}(l_x,Labeler(x)) - c_2*\mathcal{L}_{sq}(l_g,G(z,l_g)). \label{eq:z_t}\\
Loss_G &= \mathcal{L}(G(z,l_g)) + c_3*\mathcal{L}_{sq}(l_g,Labeler(G(z,l_g))). \nonumber
\end{align}

One of the advantages of BEGAN is the existence of a monotonically decreasing scalar which can track the convergence of the gradient descent optimization. Our extension preserves this property as we can define
\begin{equation}
\mathcal{M}_{complete} = \mathcal{L}(x) + |b_1| + |b_2| + |b_3|,
\label{eqn:M}
\end{equation}
and show that $\mathcal{M}_{complete}$ decreases progressively during our optimizations. See Figure \ref{fig:convergence_of_causalBEGAN} in the Appendix.

\section{Theoretical Guarantees for CausalGAN}
\label{sec:theoretical}
In this section, we show that the best CausalGAN generator for the given loss function outputs the class conditional image distribution when Causal Controller outputs the real label distribution and labelers operate at their optimum. We show this result for the case of a single binary label $l\in\{0,1\}$. The proof can be extended to multiple binary variables, which we explain in the supplementary material in Section \ref{sec:generalized_proof}. As far as we are aware of, this is the first conditional generative adversarial network architecture with this guarantee.

\subsection{CausalGAN with Single Binary Label}
First, we find the optimal discriminator for a fixed generator. Note that in (\ref{eq:discLoss}), the terms that the discriminator can optimize are the same as the GAN loss in \cite{Goodfellow2014}. Hence the optimal discriminator behaves the same as in the standard GAN. Then, the following lemma from \cite{Goodfellow2014} directly applies to our discriminator:
\begin{proposition}[\cite{Goodfellow2014}]
For fixed $G$, the optimal discriminator $D$ is given by
\begin{equation}
D_G^*(x)=\frac{p_{data}(x)}{p_{data}(x)+p_{g}(x)}.
\end{equation}
\end{proposition}

Second, we identify the optimal Labeler and Anti-Labeler. We have the following lemma:
\begin{lemma}
\label{lem:labeler_r}
The optimum Labeler has $D_{LR}(x) = \prr{l=1|x}$.
\end{lemma}
\begin{proof}
Please see the supplementary material.
\end{proof}

Similarly, we have the corresponding lemma for Anti-Labeler:
\begin{lemma}
For a fixed generator with $x \sim p_g $, the optimum Anti-Labeler has $D_{LG}(x) = \prg{l=1|x}$.
\end{lemma}
\begin{proof}
Proof is the same as the proof of Lemma \ref{lem:labeler_r}.
\end{proof}

Define $C(G)$ as the generator loss for when discriminator, Labeler and Anti-Labeler are at their optimum. Then, we show that the generator that minimizes $C(G)$ outputs class conditional image distributions. 
\begin{theorem}[Theorem \ref{thm:informal} formal for single binary label]
\label{thm:main}
The global minimum of the virtual training criterion $C(G)$ is achieved if and only if $p_g^0=p_{\text{data}}^0$ and $p_g^1=p_\text{data}^1$, i.e., if and only if given a label $l$, generator output $G(z,l)$ has the class conditional image distribution $p_{data}(x|l)$.
\end{theorem}
\begin{proof}
Please see the supplementary material.
\end{proof}

Now we can show that our two stage procedure can be used to train a causal implicit generative model for any causal graph where the \emph{Image} variable is a sink node, captured by the following corollary:
\begin{corollary}
\label{cor:main}
Suppose $C:\mathcal{Z}_1\rightarrow \mathcal{L}$ is a causal implicit generative model for the causal graph $D=(\mathcal{V},E)$ where $\mathcal{V}$ is the set of image labels and the observational joint distribution over these labels are strictly positive. Let $G:\mathcal{L}\times \mathcal{Z}_2\rightarrow \mathcal{I}$ be the class conditional GAN that can sample from the image distribution conditioned on the given label combination $L\in \mathcal{L}$. Then $G(C(Z_1),Z_2)$ is a causal implicit generative model for the causal graph $D' = (\mathcal{V} \cup \{Image\}, E \cup \{(V_1,Image),(V_2,Image),\hdots (V_n,Image)\})$.
\end{corollary}
\begin{proof}
Please see the supplementary material.
\end{proof}

\subsection{Extensions to Multiple Labels}
In Theorem \ref{thm:main} we show that the optimum generator samples from the class conditional distributions given a single binary label. Our objective is to extend this result to the case with $d$ binary labels.

First we show that if the Labeler and Anti-Labeler are trained to output $2^d$ scalars, each interpreted as the posterior probability of a particular label combination given the image, then the minimizer of $C(G)$ samples from the class conditional distributions \emph{given d  labels}. This result is shown in Theorem \ref{thm:main_generalized} in the supplementary material. However, when $d$ is large, this architecture may be hard to implement. To resolve this, we propose an alternative architecture, which we implement for our experiments: We extend the single binary label setup and use cross entropy loss terms for each label. This requires Labeler and Anti-Labeler to have only $d$ outputs. However, although we need the generator to capture the joint label posterior given the image, this only assures that the generator captures each label's posterior distribution, i.e.,  $p_r(l_i|x)=p_g(l_i|x)$ (Proposition \ref{prop:alternate}). This, in general, does not guarantee that the class conditional distributions will be true to the data distribution. However, for many joint distributions of practical interest, where \emph{the set of labels are completely determined by the image}\footnote{The dataset we are using arguably satisfies this condition.}, we show that this guarantee implies that the joint label posterior will be true to the data distribution, implying that the optimum generator samples from the class conditional distributions. Please see Section \ref{sec:alternate_practical_architecture} for the formal results and more details.

%We observe that this guarantee is sufficient to show that if 
%
%When there are $d$ binary labels, the straightforward extension is to simply add a label loss term for each binary label. 
%
%
%, extending this result to n labels requires access to knowledge of the label posterior over all labels jointly P(Y|x), whereas a network (with n outputs) trained to predict these labels with cross entropy loss only provides access to P(yi|x) for each i. We provide two remedies to this. 
%
%. However, this can be difficult to meet in practice as 2^n grows very quickly with the number of labels. We further show that if their exists a deterministic function f(x)[i]=yi which can with certainty decide the label given the image, then we show that minimizing C(G) ensures correct class conditional sampling even when the labeler is trained with only n-outputs (theorem 4)\footnote{The existence of such a function seems especially likely in any hand-labeled or machine labeled dataset. (The person or machine that carried out that labeling is then that function)}.

\section{Implementation}
In this section, we explain the differences between implementation and theory, along with other implementation details for both CausalGAN and CausalBEGAN.

\subsection{Pretraining Causal Controller for Face Labels}
\label{subsec:pretrain}

In this section, we explain the implementation details of the Wasserstein Causal Controller for generating face labels. We used the total variation distance (TVD) between the distribution of generator and data distribution as a metric to decide the success of the models.

The gradient term used as a penalty is estimated by evaluating the gradient at points interpolated between the real and fake batches. Interestingly, this Wasserstein approach gives us the opportunity to train the Causal Controller to output (almost) discrete labels (See Figure \ref{fig:cc_discrete}). In practice though, we still found benefit in rounding them before passing them to the generator. 

The generator architecture is structured in accordance with Section \ref{sec:implicit} based on the causal graph in Figure \ref{fig:big_causal_graph}, using uniform noise as exogenous variables and 6 layer neural networks as functions mapping parents to children. For the training, we used 25 Wasserstein discriminator (critic) updates per generator update, with a learning rate of 0.0008.

\subsection{Implementation Details for CausalGAN}
In practice, we use stochastic gradient descent to train our model. We use \emph{DCGAN} \cite{Radford2015}, a convolutional neural net-based implementation of generative adversarial networks, and extend it into our Causal GAN framework. We have expanded it by adding our Labeler networks, training a Causal Controller network and modifying the loss functions appropriately. Compared to DCGAN an important distinction is that we make 6 generator updates for each discriminator update on average. The discriminator and labeler networks are concurrently updated in a single iteration. 

Notice that the loss terms defined in Section \ref{causalgan:architecture} contain a single binary label. In practice we feed a $d$-dimensional label vector and need a corresponding loss function. We extend the Labeler and Anti-Labeler loss terms by simply averaging the loss terms for every label. The $i^{th}$ coordinates of the $d$-dimensional vectors given by the labelers determine the loss terms for label $i$. Note that this is different than the architecture given in Section \ref{sec:generalized_proof}, where the discriminator outputs a length-$2^d$ vector and estimates the probabilities of all label combinations given the image. Therefore this approach does not have the guarantee to sample from the class conditional distributions, if the data distribution is not restricted. However, for the type of labeled image dataset we use in this work, where labels seem to be completely determined given an image, this architecture is sufficient to have the same guarantees.  For the details, please see Section \ref{sec:alternate_practical_architecture} in the supplementary material.

Compared to the theory we have, another difference in the implementation is that we have swapped the order of the terms in the cross entropy expressions for labeler losses. This has provided sharper images at the end of the training.

\subsection{Usage of Anti-Labeler in CausalGAN}
An important challenge that comes with gradient-based training is the use of Anti-Labeler. We observe the following:  In the early stages of the training, Anti-Labeler can very quickly minimize its loss, if the generator falls into label-conditioned mode collapse. Recall that we define label-conditioned mode-collapse as the problem of generating few typical faces when a label is fixed. For example, the generator can output the same face when \emph{Eyeglasses} variable is set to 1. This helps generator to easily satisfy the label loss term we add to our loss function. Notice that however, if label-conditioned mode collapse occurs, Anti-Labeler will very easily estimate the true labels given an image, since it is always provided with the same image. Hence, maximizing the Anti-Labeler loss in the early stages of the  training helps generator to avoid label-conditioned mode collapse with our loss function.

In the later stages of the training, due to the other loss terms, generator outputs realistic images, which drives Anti-Labeler to act similar to Labeler. Thus, maximizing Anti-Labeler loss and minimizing Labeler loss become contradicting tasks. This moves the training in a direction where labels are captured less and less by the generator, hence losing the conditional image generation property. 

Based on these observations, we employ the following loss function for the generator in practice:
\begin{equation}
\mathcal{L}_G = \mathcal{L}_{GAN} + \mathcal{L}_{LabelerR} - e^{-t/T}\mathcal{L}_{LabelerG},
\end{equation}
where the terms are GAN loss term, loss of Labeler and loss of Anti-Labeler respectively (see first second and third lines of (\ref{eq:gen_loss})). $t$ is the number of iterations in the training and $T$ is the time constant of the exponential decaying coefficient for the Anti-Labeler loss. $T=3000$ is chosen for the experiments, which corresponds to roughly 1 epoch of training.

\begin{figure}[ht!]
\centering
\includegraphics[width=0.5\linewidth]{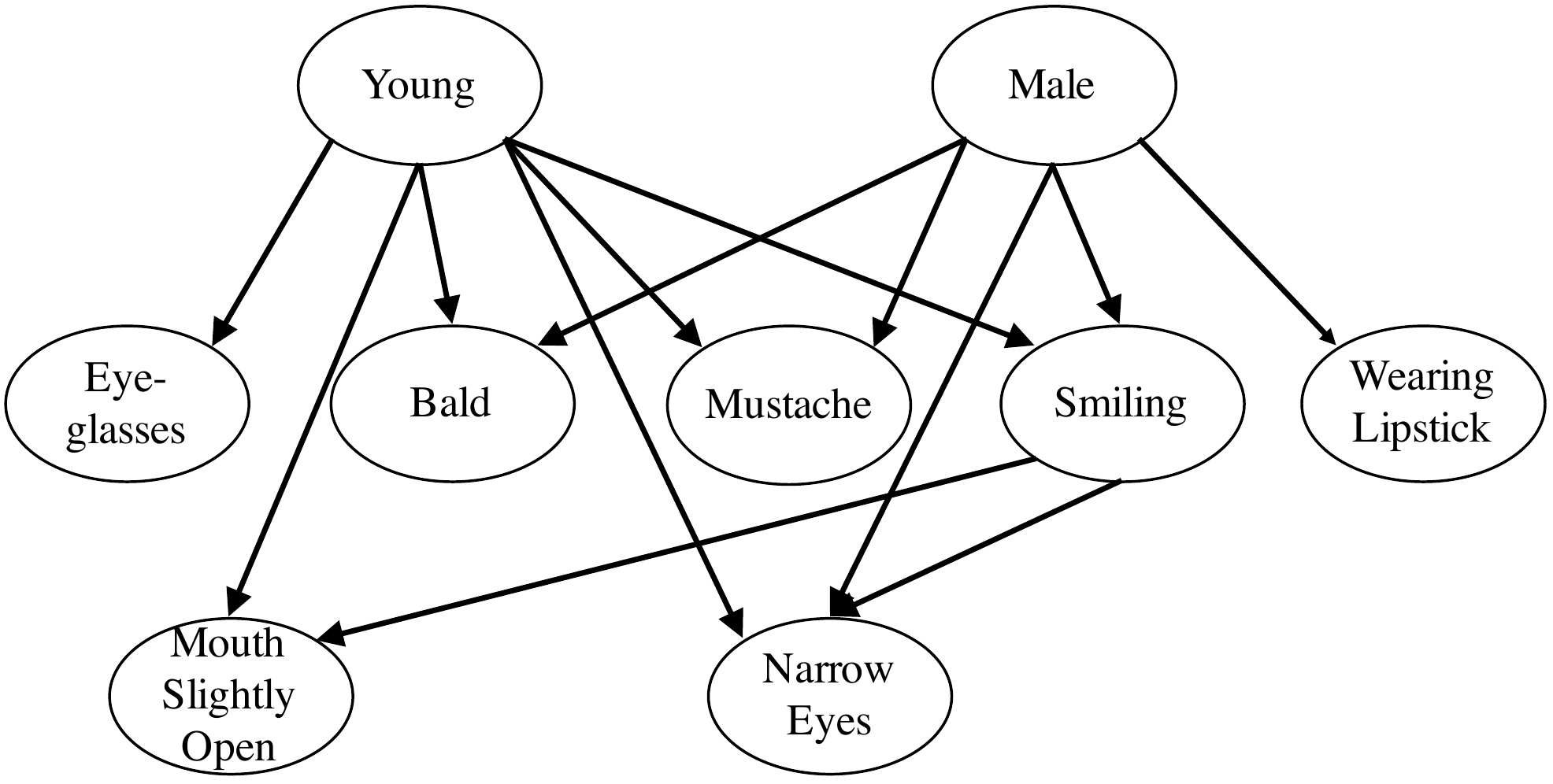}
\caption{The causal graph used for simulations for both CausalGAN and CausalBEGAN, called Causal Graph 1 (G1). We also add edges (see Appendix Section \ref{app:cc}) to form the complete graph "cG1". We also make use of the graph rcG1, which is obtained by reversing the direction of every edge in cG1.}
\label{fig:big_causal_graph}
\end{figure}

\subsection{Conditional Image Generation for CausalBEGAN}
The labels input to CausalBEGAN are taken from the Causal Controller. We use very few parameter tunings. We use the same learning rate (0.00008) for both the generator and discriminator and do 1 update of each simultaneously (calculating the for each before applying either). We simply use $\gamma_1=\gamma_2=\gamma_3=0.5$. We do not expect the model to be very sensitive to these parameter values, as we achieve good performance without hyperparameter tweaking. We do use customized margin learning rates $\lambda_1=0.001, \lambda_2=0.00008, \lambda_3=0.01$, which reflect the asymmetry in how quickly the generator can respond to each margin. For example $c_2$ can have much more "spiky", fast responding behavior compared to others even when paired with a smaller learning rate, although we have not explored this parameter space in depth. In these margin behaviors, we observe that the best performing models have all three margins "active": near $0$ while frequently taking small positive values.

\section{Results}
\label{sec:results}
\subsection{Dependence of GAN Behavior on Causal Graph}
\label{sec:synthetic}
In Section \ref{sec:implicit} we showed how a GAN could be used to train a causal implicit generative model by incorporating the causal graph into the generator structure. Here we investigate the behavior and convergence of causal implicit generative models when the true data distribution arises from another (possibly distinct) causal graph.

We consider causal implicit generative model convergence on synthetic data whose three features $\{X,Y,Z\}$ arise from one of three causal graphs: "line" $X\rightarrow Y\rightarrow Z$ , "collider" $X\rightarrow Y\leftarrow Z$, and "complete" $X\rightarrow Y\rightarrow Z,X\rightarrow Z$. For each node a (randomly sampled once) cubic polynomial in $n+1$ variables computes the value of that node given its $n$ parents and 1 uniform exogenous variable. We then repeat, creating a new synthetic dataset in this way for each causal model and report the averaged results of 20 runs for each model.

For each of these data generating graphs, we compare the convergence of the joint distribution to the true joint in terms of the total variation distance, when the generator is structured according to a line, collider, or complete graph. For completeness, we also include generators with no knowledge of causal structure: $\{fc3, fc5, fc10\}$ are fully connected neural networks that map uniform random noise to 3 output variables using either 3,5, or 10 layers respectively.

\begin{figure}[t]
%\captionsetup{justification=centering}
\subfloat[$X\rightarrow Y \rightarrow Z$]
{\label{fig:linear}
\includegraphics[width=0.32\linewidth]{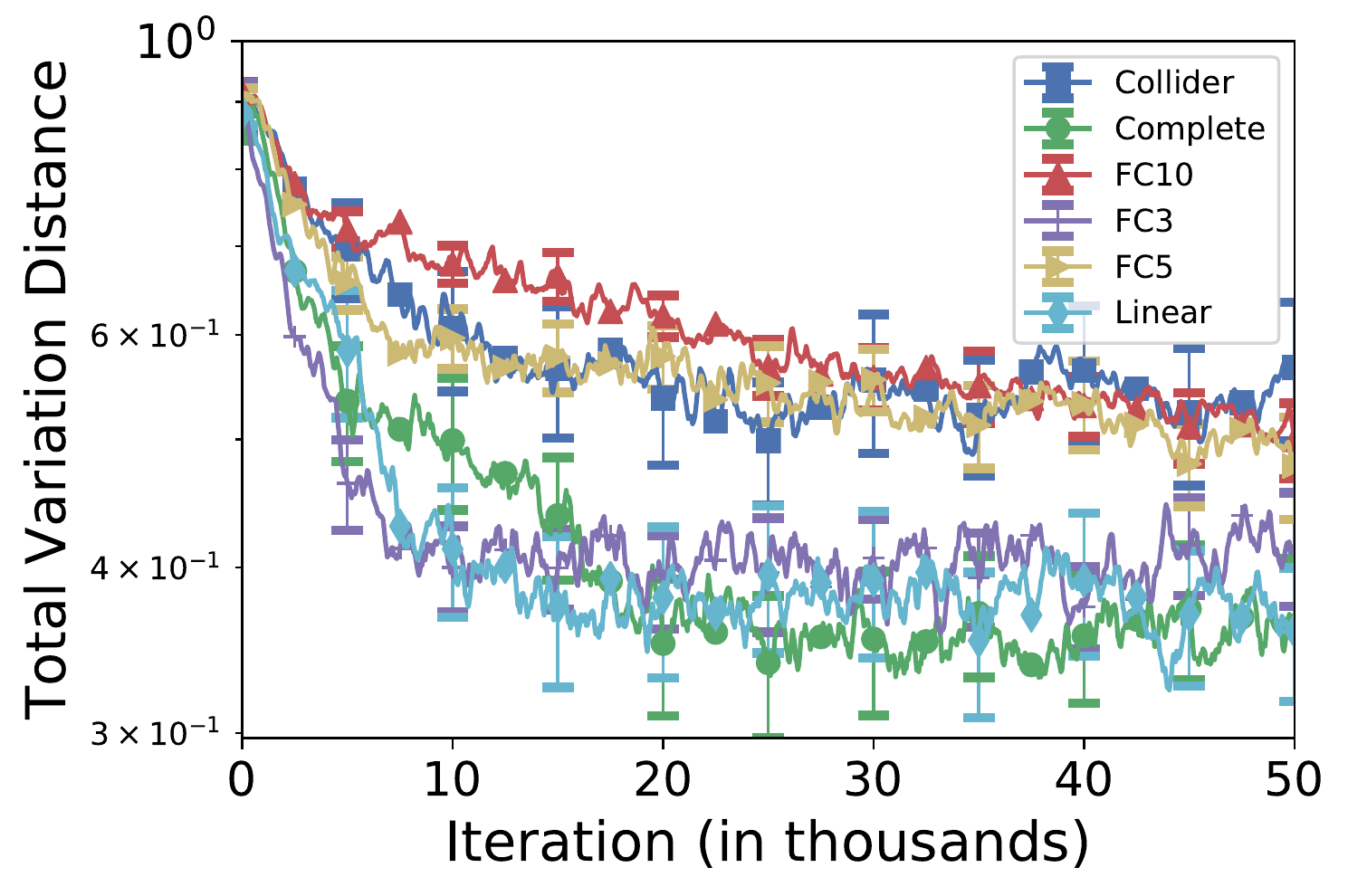}}
%\endgroup
\subfloat[$X\rightarrow Y\leftarrow Z$]
{\label{fig:collider}
\includegraphics[width=0.32\linewidth]{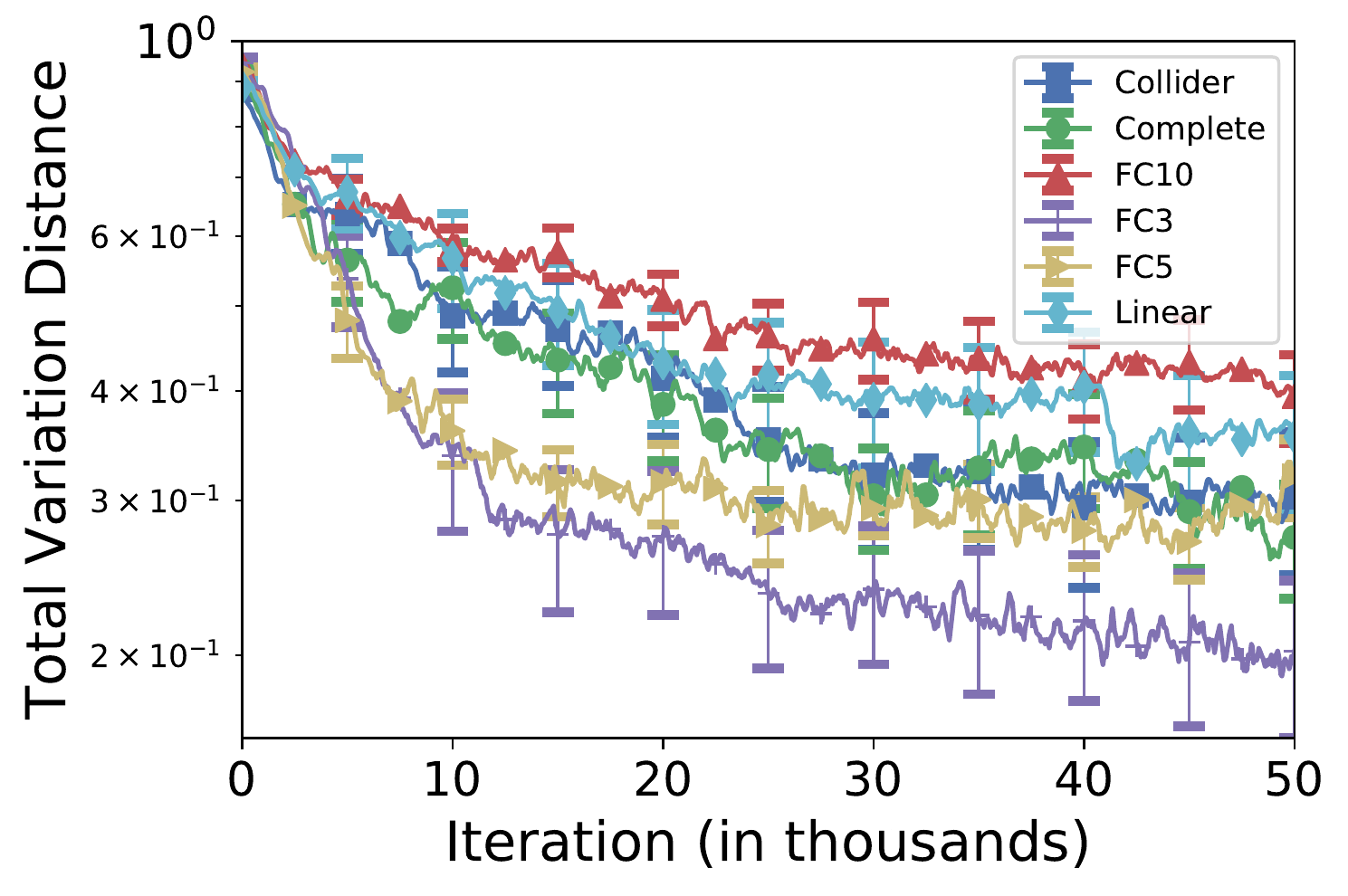}}
\subfloat[$X\rightarrow Y\rightarrow Z$, $X\rightarrow Z$]
{\label{fig:complete}
\includegraphics[width=0.32\linewidth]{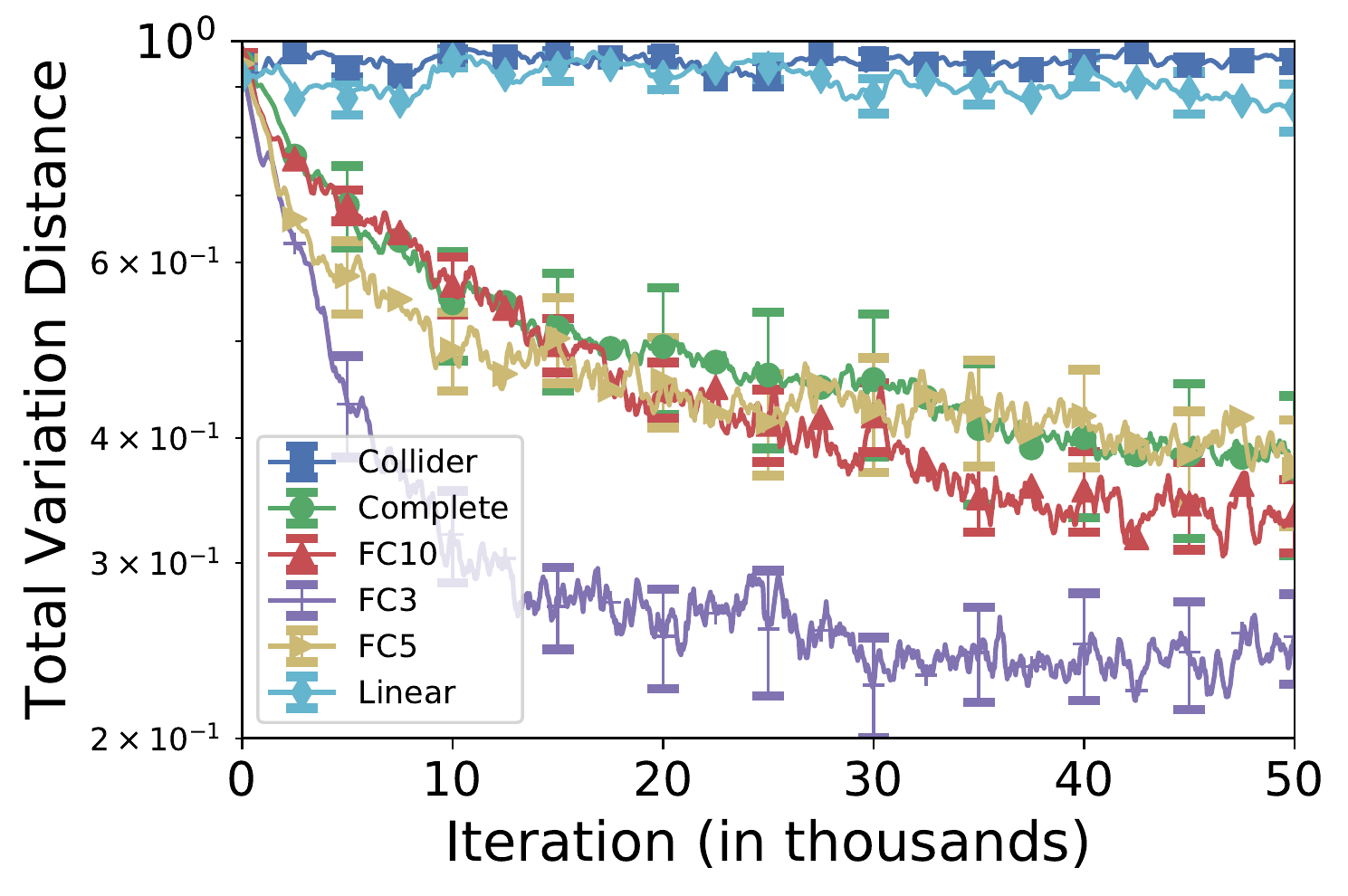}}
\caption{Convergence in total variation distance of generated distribution to the true distribution for causal implicit generative model, when the generator is structured based on different causal graphs. (a) Data generated from line graph $X\rightarrow Y\rightarrow Z$. The best convergence behavior is observed when the true causal graph is used in the generator architecture. (b) Data generated from collider graph $X\rightarrow Y \leftarrow Z$. Fully connected layers may perform better than the true graph depending on the number of layers. Collider and complete graphs performs better than the line graph which implies the wrong Bayesian network. (c) Data generated from complete graph $X\rightarrow Y\rightarrow Z$, $X\rightarrow Z$. Fully connected with 3 layers performs the best, followed by the complete and fully connected with 5 and 10 layers. Line and collider graphs, which implies the wrong Bayesian network does not show convergence behavior.}
\label{fig:synthetic_tvd}
\end{figure}

The results are given in Figure \ref{fig:synthetic_tvd}. Data is generated from line causal graph $X\rightarrow Y \rightarrow Z$ (left panel), collider causal graph $X \rightarrow Y \leftarrow$ (middle panel), and complete causal graph $X\rightarrow Y \rightarrow Z, X\rightarrow Z$ (right panel). Each curve shows the convergence behavior of the generator distribution, when generator is structured based on each one of these causal graphs. We expect convergence when the causal graph used to structure the generator is capable of generating the joint distribution due to the true causal graph: as long as we use the correct Bayesian network, we should be able to fit to the true joint. For example, complete graph can encode all joint distributions. Hence, we expect complete graph to work well with all data generation models. Standard fully connected layers correspond to the causal graph with a latent variable causing all the observable variables. Ideally, this model should be able to fit to any causal generative model. However, the convergence behavior of adversarial training across these models is unclear, which is what we are exploring with Figure \ref{fig:synthetic_tvd}.

For the line graph data $X\rightarrow Y\rightarrow Z$, we see that the best convergence behavior is when line graph is used in the generator architecture. As expected, complete graph also converges well, with slight delay. Similarly, fully connected network with 3 layers show good performance, although surprisingly fully connected with 5 and 10 layers perform much worse. It seems that although fully connected can encode the joint distribution in theory, in practice with adversarial training, the number of layers should be tuned to achieve the same performance as using the true causal graph. Using the wrong Bayesian network, the collider, also yields worse performance.

For the collider graph, surprisingly using a fully connected generator with 3 and 5 layers shows the best performance. However, consistent with the previous observation, the number of layers is important, and using 10 layers gives the worst convergence behavior. Using complete and collider graphs achieves the same decent performance, whereas line graph, a wrong Bayesian network, performs worse than the two. 

For the complete graph, fully connected 3 performs the best, followed by fully connected 5, 10 and the complete graph. As we expect, 	line and collider graphs, which cannot encode all the distributions due to a complete graph, performs the worst and does not actually show any convergence behavior. 

\subsection{Wasserstein Causal Controller on CelebA Labels}
\label{cc:results}
We test the performance of our Wasserstein Causal Controller on a subset of the binary labels of CelebA datset. We use the causal graph given in Figure \ref{fig:big_causal_graph}.

For causal graph training, first we verify that our Wasserstein training allows the generator to learn a mapping from continuous uniform noise to a discrete distribution. Figure \ref{fig:cc_discrete} shows where the samples, averaged over all the labels in Causal Graph 1, from this generator appears on the real line. The result emphasizes that the proposed Causal Controller outputs an almost discrete distribution: $96\%$ of the samples appear in $0.05-$neighborhood of $0$ or $1$. Outputs shown are \emph{unrounded} generator outputs. 

\begin{figure}
\subfloat[Essentially Discrete Range of Causal Controller]
{\label{fig:cc_discrete}
\includegraphics[width=0.45\linewidth]{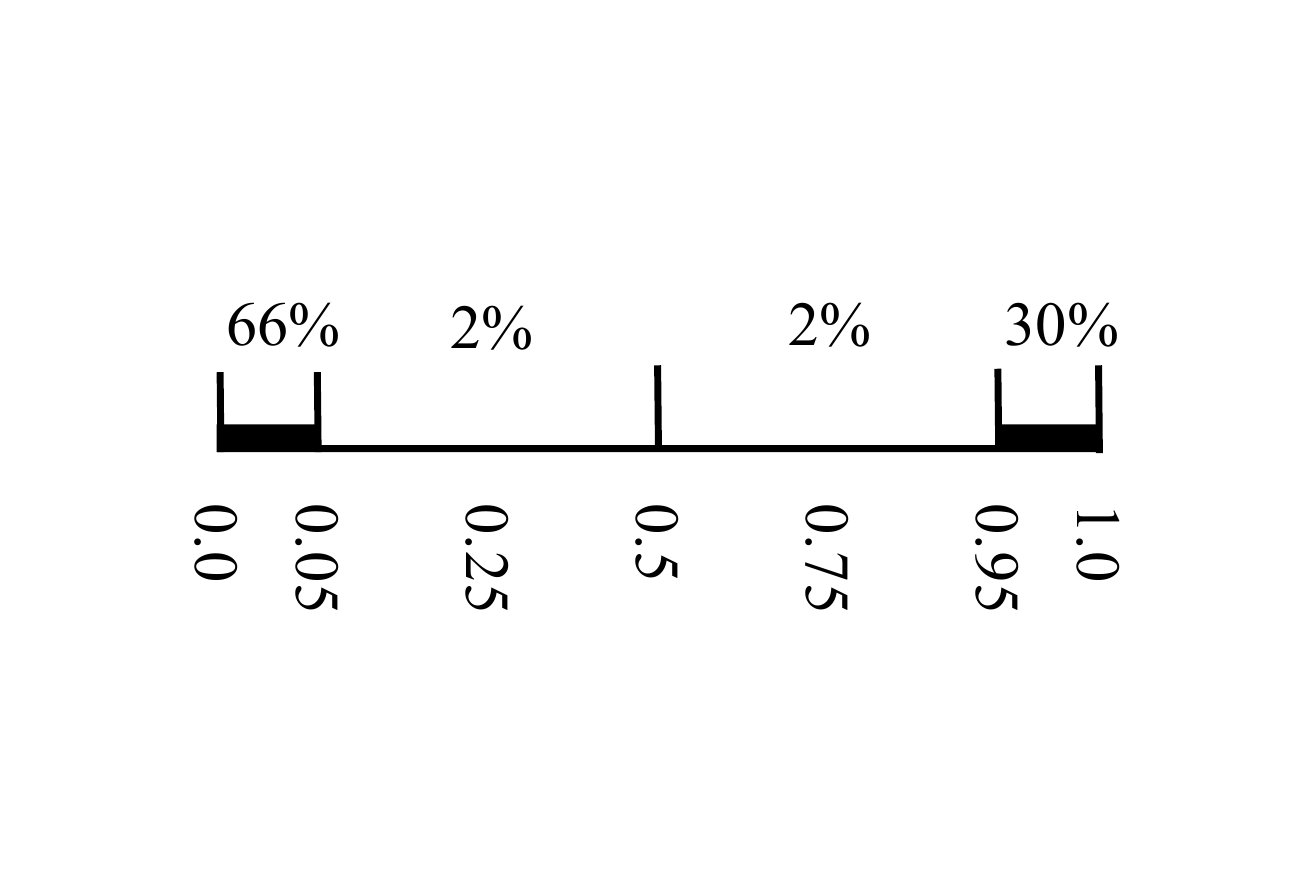}}
\hspace{0.2in}
\subfloat[TVD vs. No. of Iters in CelebA Labels]
{\label{fig:TVDiters}
\includegraphics[width=0.4\linewidth]{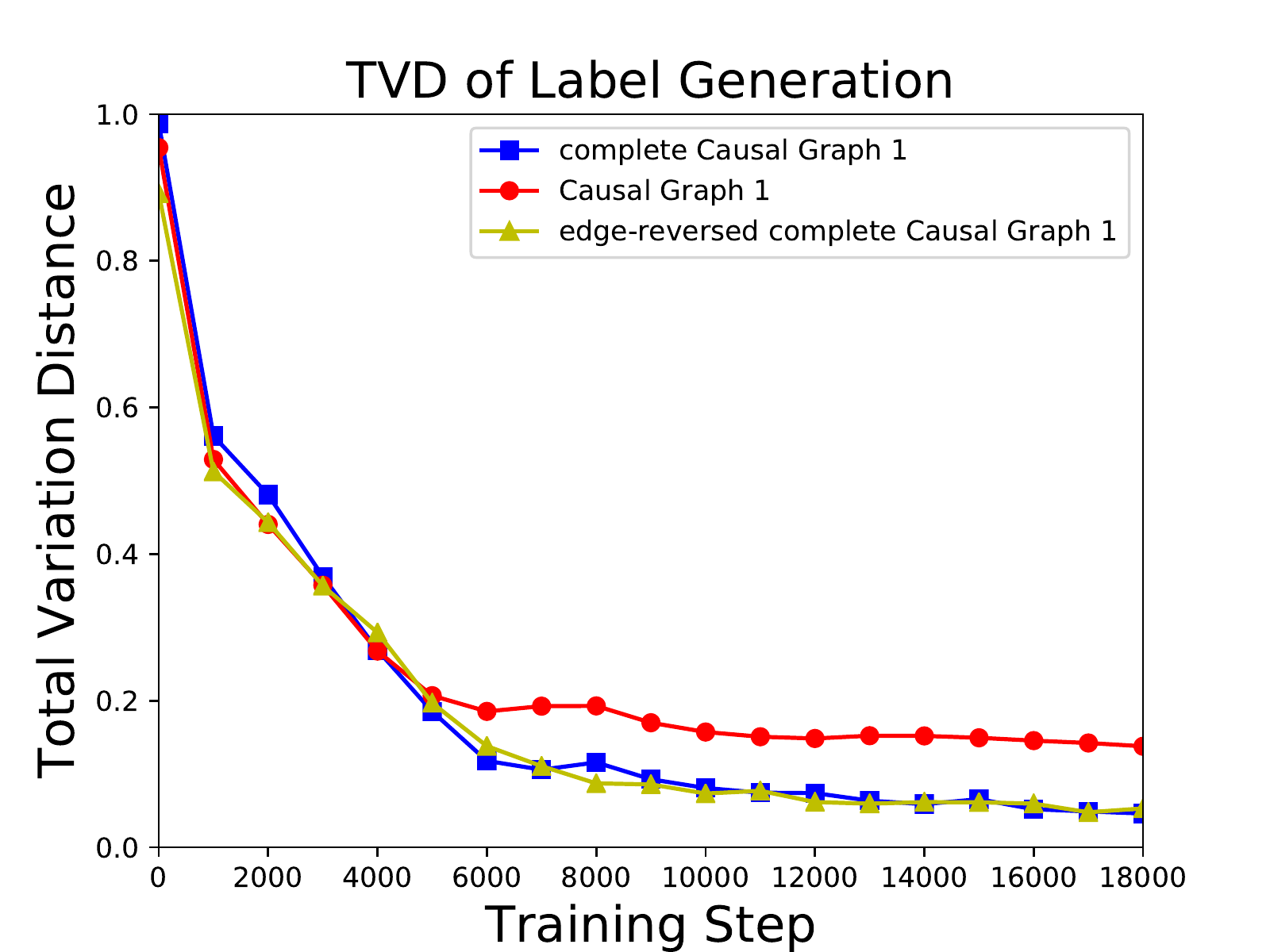}}
\caption{(a) A number line of unit length binned into 4 unequal bins along with the percent of Causal Controller ($G1$) samples in each bin. Results are obtained by sampling the joint label distribution 1000 times and forming a histogram of the scalar outputs corresponding to any label. Note that our Causal Controller output labels are approximately discrete even though the input is a continuum (uniform). The 4\% between 0.05 and 0.95 is not at all uniform and almost zero near 0.5. (b) Progression of total variation distance between the Causal Controller output with respect to the number of iterations: Causal Graph 1 is used in the training with Wasserstein loss.}
\label{asdf}
\end{figure}

A stronger measure of convergence is the total variational distance (TVD). For Causal Graph 1 (G1), our defined completion (cG1), and cG1 with arrows reversed (rcG1), we show convergence of TVD with training (Figure \ref{fig:TVDiters}). Both cG1 and rcG1 have TVD decreasing to 0, and TVD for G1 assymptotes to around 0.14 which corresponds to the incorrect conditional independence assumptions that G1 makes.  This suggests that any given complete causal graph will lead to a nearly perfect implicit causal generator over labels and that bayesian partially incorrect causal graphs can still give reasonable convergence.

\subsection{CausalGAN Results}
In this section, we train the whole CausalGAN together using a pretrained Causal Controller network. The results are given in Figures \ref{fig:dcgan_mustache1}-\ref{fig:dcgan_narrow1}. The difference between intervening and conditioning is clear through certain features. We implement conditioning through rejection sampling. See \cite{Naesseth2017, Graham2017} for other works on conditioning for implicit generative models.

\begin{figure}[h!]
\centering
\subfloat[Intervening vs Conditioning on Mustache, Top: Intervene Mustache=1, Bottom: Condition Mustache=1 ]
{\label{fig:dcgan_mustache1}
\includegraphics[width=0.8\linewidth]%{./dcgan_pictures/celebA_0829_110623_no_goz/label_intv_v_conditioning/45507_intvcond_Mustache=1_2x10.pdf}}\\
{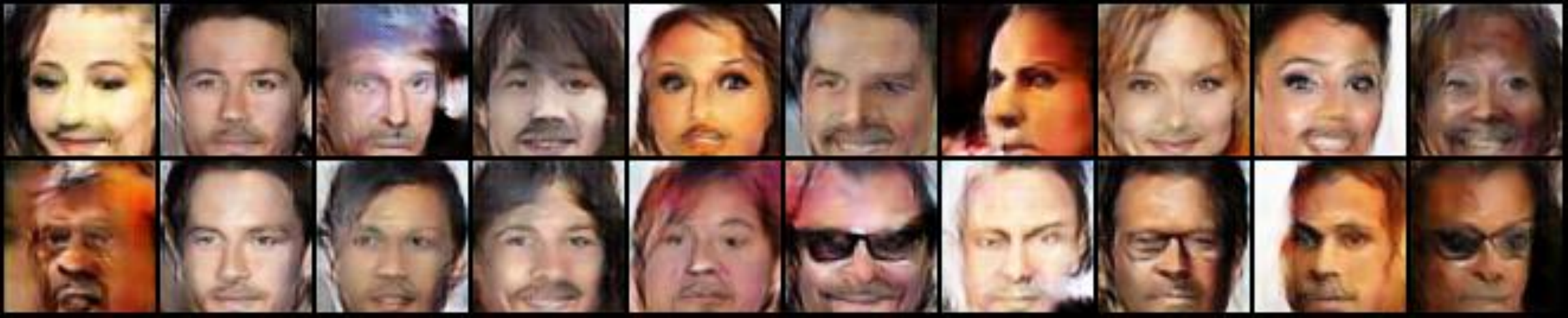}}\\
\caption{Intervening/Conditioning on Mustache label in Causal Graph 1. Since $Male\rightarrow Mustache$ in Causal Graph 1, we do not expect $do(Mustache=1)$ to affect the probability of $Male=1$, i.e., $\pr{Male=1|do(Mustache=1)}=\pr{Male=1} = 0.42$. Accordingly, the top row shows both males and females with mustaches, even though the generator never sees the label combination $\{Male=0, Mustache=1\}$ during training. The bottom row of images sampled from the conditional distribution $\pr{.|Mustache=1}$ shows only male images because in the dataset $\pr{Male=1|Mustache=1}\approx 1$.}
\end{figure}

\begin{figure}[ht!]
\centering
\subfloat[ Intervening vs Conditioning on Bald, Top: Intervene Bald=1, Bottom: Condition Bald=1 ]
{\label{fig:dcgan_bald1}
\includegraphics[width=0.8\linewidth]%{./dcgan_pictures/celebA_0829_110623_no_goz/label_intv_v_conditioning/45507_intvcond_Bald=1_2x10.pdf}}\\
{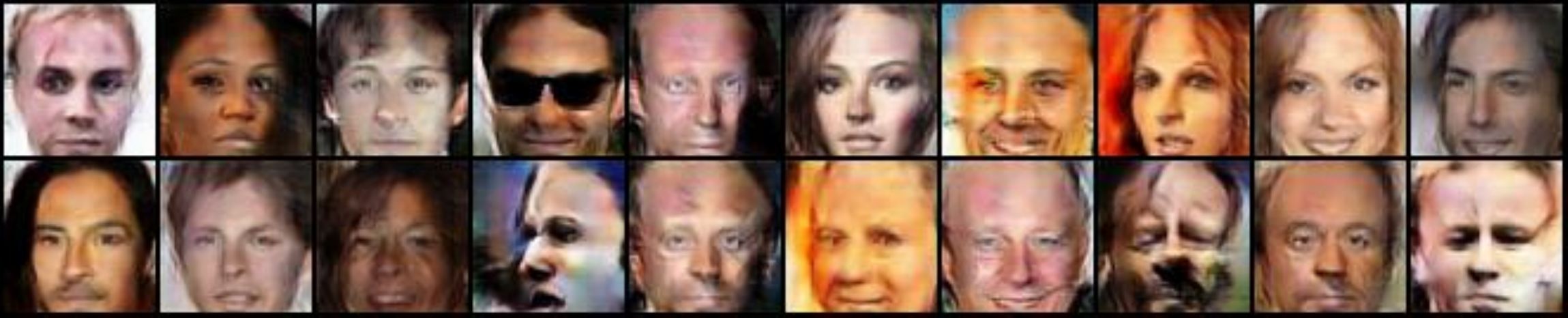}}\\
\caption{Intervening/Conditioning on Bald label in Causal Graph 1. Since $Male\rightarrow Bald$ in Causal Graph 1, we do not expect $do(Bald=1)$ to affect the probability of $Male=1$, i.e., $\pr{Male=1|do(Bald=1)}=\pr{Male=1} = 0.42$. Accordingly, the top row shows both bald males and bald females. The bottom row of images sampled from the conditional distribution $\pr{.|Bald=1}$ shows only male images because in the dataset $\pr{Male=1|Bald=1}\approx 1$.}
\end{figure}

\begin{figure}[ht!]
\centering
\subfloat[ Intervening vs Conditioning on Wearing Lipstick, Top: Intervene Wearing Lipstick=1, Bottom: Condition Wearing Lipstick=1 ]
{\label{fig:dcgan_lipstick1}
\includegraphics[width=0.8\linewidth]%{./dcgan_pictures/celebA_0829_110623_no_goz/label_intv_v_conditioning/45507_intvcond_Wearing_Lipstick=1_2x10.pdf}}\\
{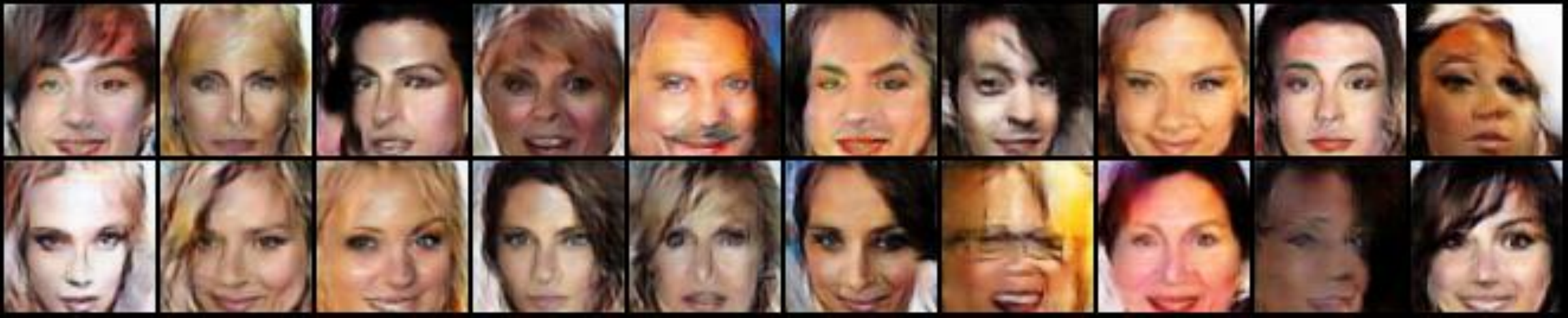}}\\
\caption{Intervening/Conditioning on Wearing Lipstick label in Causal Graph 1. Since $Male\rightarrow Wearing Lipstick$ in Causal Graph 1, we do not expect $do(\textit{Wearing Lipstick}=1)$ to affect the probability of $Male=1$, i.e., $\pr{Male=1|do(\textit{Wearing Lipstick}=1)}=\pr{Male=1} = 0.42$. Accordingly, the top row shows both males and females who are wearing lipstick. However, the bottom row of images sampled from the conditional distribution $\pr{.|\textit{Wearing Lipstick}=1}$ shows only female images because in the dataset $\pr{Male=0|\textit{Wearing Lipstick}=1}\approx 1$.}
\end{figure}

\begin{figure}[ht!]
\centering
\subfloat[ Intervening vs Conditioning on Mouth Slightly Open, Top: Intervene Mouth Slightly Open=1, Bottom: Condition Mouth Slightly Open=1 ]
{\label{fig:dcgan_mso1}
\includegraphics[width=0.8\linewidth]%{./dcgan_pictures/celebA_0829_110623_no_goz/label_intv_v_conditioning/45507_intvcond_Mouth_Slightly_Open=1_2x10.pdf}}\\
{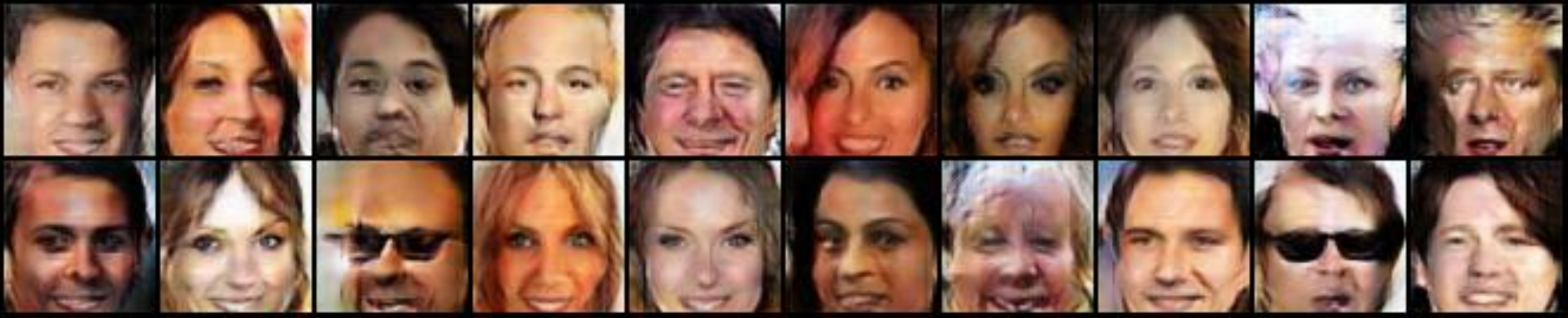}}\\
\caption{Intervening/Conditioning on Mouth Slightly Open label in Causal Graph 1. Since $Smiling\rightarrow Mouth Slightly Open$ in Causal Graph 1, we do not expect $do(\textit{Mouth Slightly Open}=1)$ to affect the probability of $Smiling=1$, i.e., $\pr{Smiling=1|do(\textit{Mouth Slightly Open}=1)}=\pr{Smiling=1}=0.48$. However on the bottom row, conditioning on $\textit{Mouth Slightly Open}=1$ increases the proportion of smiling images ($0.48\rightarrow 0.76$ in the dataset), although 10 images may not be enough to show this difference statistically.}
\end{figure}

\begin{figure}[ht!]
\centering
\subfloat[ Intervening vs Conditioning on Narrow Eyes, Top: Intervene Narrow Eyes=1, Bottom: Condition Narrow Eyes=1 ]
{\label{fig:dcgan_narrow1}
\includegraphics[width=0.8\linewidth]%{./dcgan_pictures/celebA_0829_110623_no_goz/label_intv_v_conditioning/45507_intvcond_Narrow_Eyes=1_2x10.pdf}}\\
{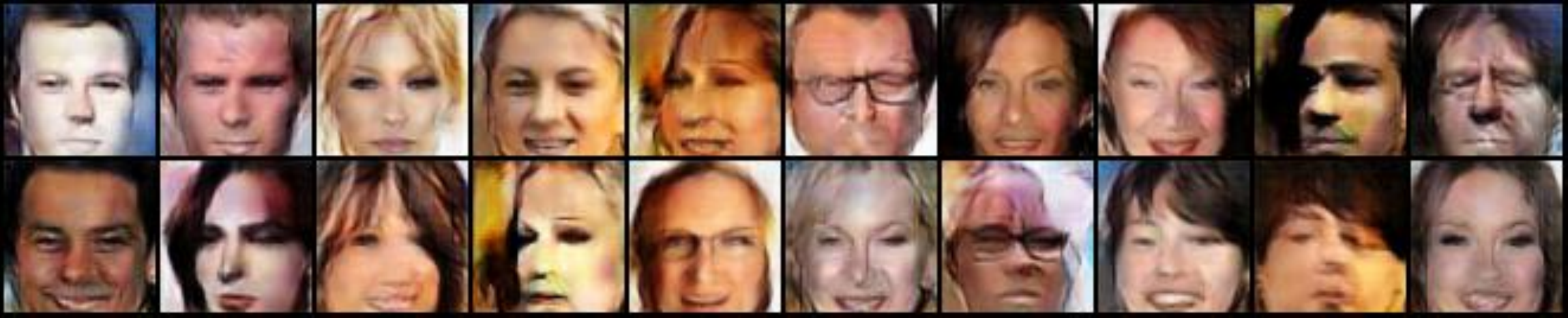}}\\
\caption{Intervening/Conditioning on Narrow Eyes label in Causal Graph 1. Since $Smiling\rightarrow \textit{Narrow Eyes}$ in Causal Graph 1, we do not expect $do(\textit{Narrow Eyes}=1)$ to affect the probability of $Smiling=1$, i.e., $\pr{Smiling=1|do(\textit{Narrow Eyes}=1)}=\pr{Smiling=1}=0.48$. However on the bottom row, conditioning on $\textit{Narrow Eyes}=1$ increases the proportion of smiling images ($0.48\rightarrow 0.59$ in the dataset), although 10 images may not be enough to show this difference statistically.}
\end{figure}

%data path began_pictures/celebA_0810_224425_beganworking_bcg/image_diversity/190001_G_diversity.pdf
% BEGAN Simulations
\subsection{CausalBEGAN Results}
In this section, we train CausalBEGAN on CelebA dataset using Causal Graph 1. The Causal Controller is pretrained with a Wasserstein loss and used for training the CausalBEGAN.

To first empirically justify the need for the margin of margins we introduced in (\ref{eq:z_t}) ($c_3$ and $b_3$), we train the same CausalBEGAN model setting $c_3=1$, removing the effect of this margin. We show that the image quality for rare labels deteriorates. Please see Figure \ref{fig:no3margin} in the appendix. Then for the labels \emph{Wearing Lipstick, Mustache, Bald}, and \emph{Narrow Eyes}, we illustrate the difference between interventional and conditional sampling when the label is 1. (Figures \ref{fig:began_mustache1}-\ref{fig:began_narrow1}).

\begin{figure}[h!]
\centering
\subfloat[Intervening vs Conditioning on Mustache, Top: Intervene Mustache=1, Bottom: Condition Mustache=1 ]
{\label{fig:began_mustache1}
\includegraphics[width=0.8\linewidth]%{./began_pictures/celebA_0810_224425_beganworking_bcg/label_intv_v_conditioning/190001_intvcond_Mustache=1_2x10.pdf}}\\
{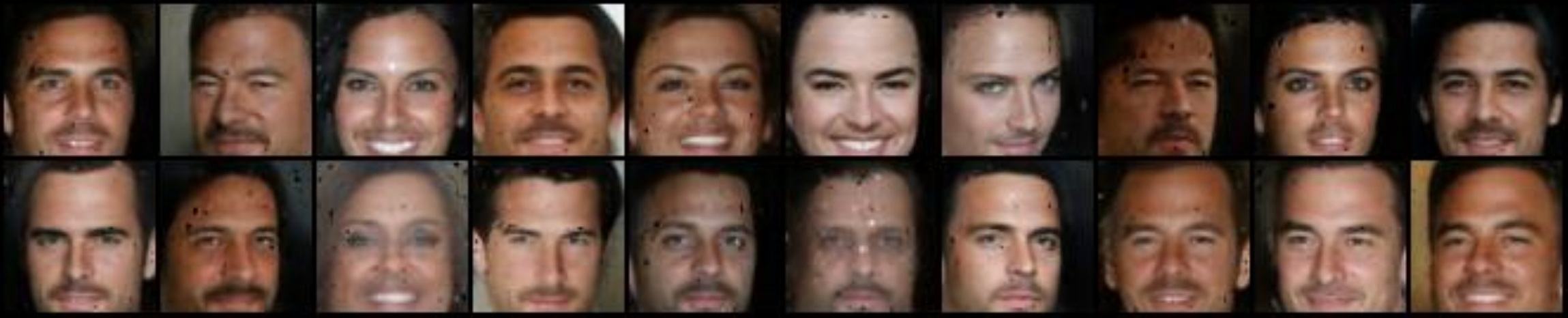}}\\
\caption{Intervening/Conditioning on Mustache label in Causal Graph 1. Since $Male\rightarrow Mustache$ in Causal Graph 1, we do not expect $do(Mustache=1)$ to affect the probability of $Male=1$, i.e., $\pr{Male=1|do(Mustache=1)}=\pr{Male=1} = 0.42$. Accordingly, the top row shows both males and females with mustaches, even though the generator never sees the label combination $\{Male=0, Mustache=1\}$ during training. The bottom row of images sampled from the conditional distribution $\pr{.|Mustache=1}$ shows only male images because in the dataset $\pr{Male=1|Mustache=1}\approx 1$.}
\end{figure}

\begin{figure}[ht!]
\centering
\subfloat[ Intervening vs Conditioning on Bald, Top: Intervene Bald=1, Bottom: Condition Bald=1 ]
{\label{fig:began_bald1}
\includegraphics[width=0.8\linewidth]%{./began_pictures/celebA_0810_224425_beganworking_bcg/label_intv_v_conditioning/190001_intvcond_Bald=1_2x10.pdf}}\\
{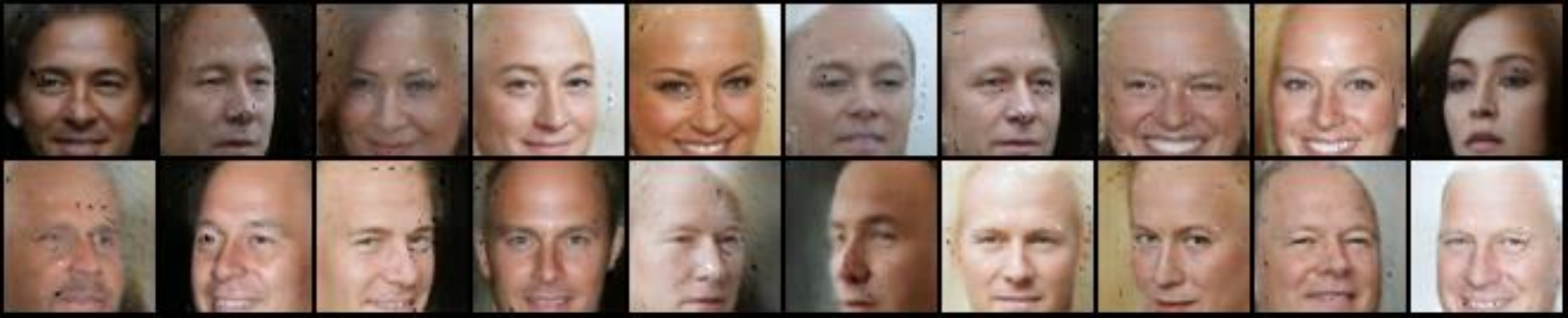}}\\
\caption{Intervening/Conditioning on Bald label in Causal Graph 1. Since $Male\rightarrow Bald$ in Causal Graph 1, we do not expect $do(Bald=1)$ to affect the probability of $Male=1$, i.e., $\pr{Male=1|do(Bald=1)}=\pr{Male=1} = 0.42$. Accordingly, the top row shows both bald males and bald females. The bottom row of images sampled from the conditional distribution $\pr{.|Bald=1}$ shows only male images because in the dataset $\pr{Male=1|Bald=1}\approx 1$.}
\end{figure}

\begin{figure}[ht!]
\centering
\subfloat[ Intervening vs Conditioning on Mouth Slightly Open, Top: Intervene Mouth Slightly Open=1, Bottom: Condition Mouth Slightly Open=1 ]
{\label{fig:began_mso1}
\includegraphics[width=0.8\linewidth]%{./began_pictures/celebA_0810_224425_beganworking_bcg/label_intv_v_conditioning/190001_intvcond_Mouth_Slightly_Open=1_2x10.pdf}}\\
{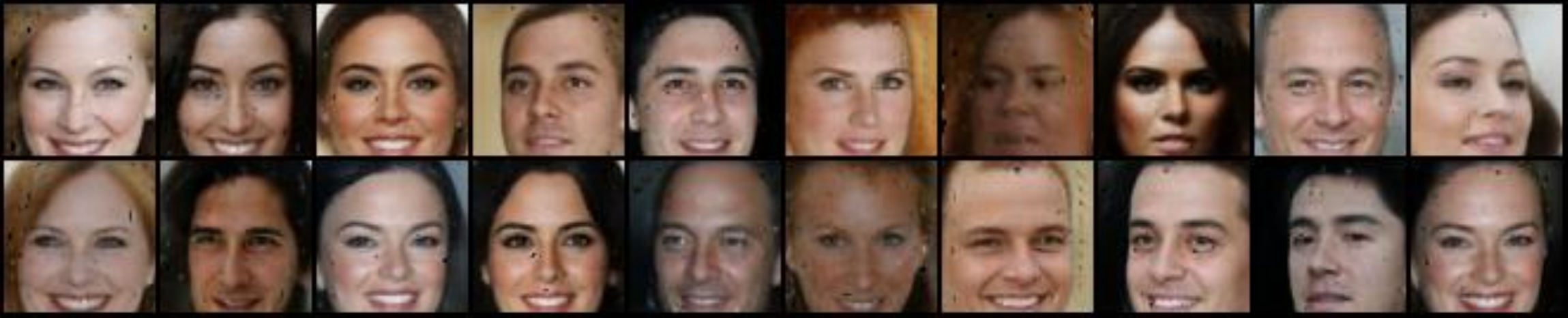}}\\
\caption{Intervening/Conditioning on Mouth Slightly Open label in Causal Graph 1. Since $Smiling\rightarrow Mouth Slightly Open$ in Causal Graph 1, we do not expect $do(\textit{Mouth Slightly Open}=1)$ to affect the probability of $Smiling=1$, i.e., $\pr{Smiling=1|do(\textit{Mouth Slightly Open}=1)}=\pr{Smiling=1}=0.48$. However on the bottom row, conditioning on $\textit{Mouth Slightly Open}=1$ increases the proportion of smiling images ($0.48\rightarrow 0.76$ in the dataset), although 10 images may not be enough to show this difference statistically.}
\end{figure}

\begin{figure}[ht!]
\centering
\subfloat[ Intervening vs Conditioning on Narrow Eyes, Top: Intervene Narrow Eyes=1, Bottom: Condition Narrow Eyes=1 ]
{\label{fig:began_narrow1}
\includegraphics[width=0.8\linewidth]{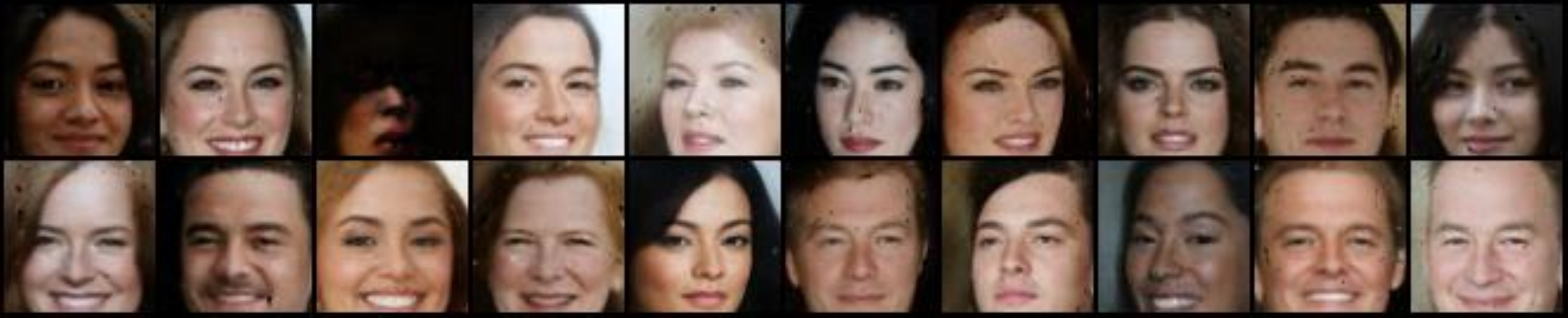}}\\
%{./began_pictures/celebA_0810_224425_beganworking_bcg/label_intv_v_conditioning/190001_intvcond_Narrow_Eyes=1_2x10.pdf}}\\
\caption{Intervening/Conditioning on Narrow Eyes label in Causal Graph 1. Since $Smiling\rightarrow \textit{Narrow Eyes}$ in Causal Graph 1, we do not expect $do(\textit{Narrow Eyes}=1)$ to affect the probability of $Smiling=1$, i.e., $\pr{Smiling=1|do(\textit{Narrow Eyes}=1)}=\pr{Smiling=1}=0.48$. However on the bottom row, conditioning on $\textit{Narrow Eyes}=1$ increases the proportion of smiling images ($0.48\rightarrow 0.59$ in the dataset), although 10 images may not be enough to show this difference statistically. As a rare artifact, in the dark image in the third column the generator appears to rule out the possibility of $\textit{Narrow Eyes}=0$ instead of demonstrating $\textit{Narrow Eyes}=1$.}
\end{figure}

\section{Conclusion}
We proposed a novel generative model with label inputs. In addition to being able to create samples \textit{conditional} on labels, our generative model can also sample from the \textit{interventional} distributions. 
Our theoretical analysis provides provable guarantees about correct sampling under such interventions and conditionings. 
The difference between these two sampling mechanisms is the key for causality. Interestingly, causality leads to generative models that are more creative since they can produce samples that are different from their training samples in multiple ways. We have illustrated this point for two models (CausalGAN and CausalBEGAN) and numerous label examples.

\section*{Acknowledgements}
We thank Ajil Jalal for the helpful discussions.

{\small
\bibliographystyle{plain}
\bibliography{causalinferenceSimple}

\begin{thebibliography}{10}

\bibitem{Antipov2017}
Grigory Antipov, Moez Baccouche, and Jean-Luc Dugelay.
\newblock Face aging with conditional generative adversarial networks.
\newblock In {\em arXiv pre-print}, 2017.

\bibitem{Arjovsky2017}
Martin Arjovsky, Soumith Chintala, and L\'{e}on Bottou.
\newblock Wasserstein gan.
\newblock In {\em arXiv pre-print}, 2017.

\bibitem{Bahadori2017}
Mohammad~Taha Bahadori, Krzysztof Chalupka, Edward Choi, Robert Chen, Walter~F.
  Stewart, and Jimeng Sun.
\newblock Causal regularization.
\newblock In {\em arXiv pre-print}, 2017.

\bibitem{Berthelot2017}
David Berthelot, Thomas Schumm, and Luke Metz.
\newblock Began: Boundary equilibrium generative adversarial networks.
\newblock In {\em arXiv pre-print}, 2017.

\bibitem{Besserve2017}
Michel Besserve, Naji Shajarisales, Bernhard Sch\"{o}lkopf, and Dominik
  Janzing.
\newblock Group invariance principles for causal generative models.
\newblock In {\em arXiv pre-print}, 2017.

\bibitem{Bora2017}
Ashish Bora, Ajil Jalal, Eric Price, and Alexandros~G. Dimakis.
\newblock Compressed sensing using generative models.
\newblock In {\em ICML 2017}, 2017.

\bibitem{Chen2016}
Yan Chen, Xi~Duan, Rein Houthooft, John Schulman, Ilya Sutskever, and Pieter
  Abbeel.
\newblock Infogan: Interpretable representation learning by information
  maximizing generative adversarial nets.
\newblock In {\em Proceedings of NIPS 2016}, Barcelona, Spain, December 2016.

\bibitem{Donahue2017}
Chris Donahue, Akshay Balsubramani, Julian McAuley, and Zachary~C. Lipton.
\newblock Semantically decomposing the latent spaces of generative adversarial
  networks.
\newblock In {\em arXiv pre-print}, 2017.

\bibitem{Donahue2016}
Jeff Donahue, Philipp Kr\"{a}henb\"{u}hl, and Trevor Darrell.
\newblock Adversarial feature learning.
\newblock In {\em ICLR}, 2017.

\bibitem{Dumoulin2017}
Vincent Dumoulin, Ishmael Belghazi, Ben Poole, Olivier Mastropietro, Alex Lamb,
  Martin Arjovsky, and Aaron Courville.
\newblock Adversarially learned inference.
\newblock In {\em ICLR}, 2017.

\bibitem{EberhardtThesis}
Frederick Eberhardt.
\newblock Phd thesis.
\newblock {\em Causation and Intervention (Ph.D. Thesis)}, 2007.

\bibitem{Etesami2016}
Jalal Etesami and Negar Kiyavash.
\newblock Discovering influence structure.
\newblock In {\em IEEE ISIT}, 2016.

\bibitem{Goodfellow2014}
Ian~J. Goodfellow, Jean Pouget-Abadie, Mehdi Mirza, Bing Xu, David
  Warde-Farley, Sherjil Ozair, Aaron Courville, and Yoshua Bengio.
\newblock Generative adversarial nets.
\newblock In {\em Proceedings of NIPS 2014}, Montreal, CA, December 2014.

\bibitem{Graham2017}
Matthew Graham and Amos Storkey.
\newblock {Asymptotically exact inference in differentiable generative models}.
\newblock In Aarti Singh and Jerry Zhu, editors, {\em Proceedings of the 20th
  International Conference on Artificial Intelligence and Statistics},
  volume~54 of {\em Proceedings of Machine Learning Research}, pages 499--508,
  Fort Lauderdale, FL, USA, 20--22 Apr 2017. PMLR.

\bibitem{Gulrajani2017}
Ishaan Gulrajani, Faruk Ahmed, Martin Arjovsky, Vincent Dumoulin, and Aaron
  Courville.
\newblock Improved training of wasserstein gans.
\newblock In {\em arXiv pre-print}, 2017.

\bibitem{hauser2014two}
Alain Hauser and Peter B{\"u}hlmann.
\newblock Two optimal strategies for active learning of causal models from
  interventional data.
\newblock {\em International Journal of Approximate Reasoning}, 55(4):926--939,
  2014.

\bibitem{Hoyer2008}
Patrik~O Hoyer, Dominik Janzing, Joris Mooij, Jonas Peters, and Bernhard
  Sch\"{o}lkopf.
\newblock Nonlinear causal discovery with additive noise models.
\newblock In {\em Proceedings of NIPS 2008}, 2008.

\bibitem{Hyttinen2013}
Antti Hyttinen, Frederick Eberhardt, and Patrik Hoyer.
\newblock Experiment selection for causal discovery.
\newblock {\em Journal of Machine Learning Research}, 14:3041--3071, 2013.

\bibitem{Kocaoglu2017a}
Murat Kocaoglu, Alexandros~G. Dimakis, and Sriram Vishwanath.
\newblock Cost-optimal learning of causal graphs.
\newblock In {\em ICML'17}, 2017.

\bibitem{Kocaoglu2017}
Murat Kocaoglu, Alexandros~G. Dimakis, Sriram Vishwanath, and Babak Hassibi.
\newblock Entropic causal inference.
\newblock In {\em AAAI'17}, 2017.

\bibitem{Kontoyiannis2016}
Ioannis Kontoyiannis and Maria Skoularidou.
\newblock Estimating the directed information and testing for causality.
\newblock {\em IEEE Trans. Inf. Theory}, 62:6053--6067, Aug. 2016.

\bibitem{Liu2016}
Ming-Yu Liu and Tuzel Oncel.
\newblock Coupled generative adversarial networks.
\newblock In {\em Proceedings of NIPS 2016}, Barcelona,Spain, December 2016.

\bibitem{liu2015faceattributes}
Ziwei Liu, Ping Luo, Xiaogang Wang, and Xiaoou Tang.
\newblock Deep learning face attributes in the wild.
\newblock In {\em Proceedings of International Conference on Computer Vision
  (ICCV)}, December 2015.

\bibitem{Paz2015a}
David Lopez-Paz, Krikamol Muandet, Bernhard Sch\"olkopf, and Ilya Tolstikhin.
\newblock Towards a learning theory of cause-effect inference.
\newblock In {\em Proceedings of ICML 2015}, 2015.

\bibitem{Paz2017}
David Lopez-Paz, Robert Nishihara, Soumith Chintala, Bernhard Sch\"{o}lkopf,
  and L\'{e}on Bottou.
\newblock Discovering causal signals in images.
\newblock In {\em Proceedings of CVPR 2017}, Honolulu, CA, July 2017.

\bibitem{Paz2016}
David Lopez-Paz and Maxime Oquab.
\newblock Revisiting classifier two-sample tests.
\newblock In {\em arXiv pre-print}, 2016.

\bibitem{Mirza2014}
Mehdi Mirza and Simon Osindero.
\newblock Conditional generative adversarial nets.
\newblock In {\em arXiv pre-print}, 2016.

\bibitem{Mohamed2016}
Shakir Mohamed and Balaji Lakshminarayanan.
\newblock Learning in implicit generative models.
\newblock In {\em arXiv pre-print}, 2016.

\bibitem{Naesseth2017}
Christian Naesseth, Francisco Ruiz, Scott Linderman, and David Blei.
\newblock {Reparameterization Gradients through Acceptance-Rejection Sampling
  Algorithms}.
\newblock In Aarti Singh and Jerry Zhu, editors, {\em Proceedings of the 20th
  International Conference on Artificial Intelligence and Statistics},
  volume~54 of {\em Proceedings of Machine Learning Research}, pages 489--498,
  Fort Lauderdale, FL, USA, 20--22 Apr 2017. PMLR.

\bibitem{Odena2016}
Augustus Odena, Christopher Olah, and Jonathon Shlens.
\newblock Conditional image synthesis with auxiliary classifier gans.
\newblock In {\em arXiv pre-print}, 2016.

\bibitem{Pearl2009}
Judea Pearl.
\newblock {\em Causality: Models, Reasoning and Inference}.
\newblock Cambridge University Press, 2009.

\bibitem{Quinn2015}
Christopher Quinn, Negar Kiyavash, and Todd Coleman.
\newblock Directed information graphs.
\newblock {\em IEEE Trans. Inf. Theory}, 61:6887--6909, Dec. 2015.

\bibitem{Radford2015}
Alec Radford, Luke Metz, and Soumith Chintala.
\newblock Unsupervised representation learning with deep convolutional
  generative adversarial networks.
\newblock In {\em arXiv pre-print}, 2015.

\bibitem{Salimans2016}
Tim Salimans, Ian Goodfellow, Wojciech Zaremba, Vicki Cheung, Alec Radford, and
  Xi~Chen.
\newblock Improved techniques for training gans.
\newblock In {\em NIPS'16}, 2016.

\bibitem{Shanmugam2015}
Karthikeyan Shanmugam, Murat Kocaoglu, Alex Dimakis, and Sriram Vishwanath.
\newblock Learning causal graphs with small interventions.
\newblock In {\em NIPS 2015}, 2015.

\bibitem{Spirtes2001}
Peter Spirtes, Clark Glymour, and Richard Scheines.
\newblock {\em Causation, Prediction, and Search}.
\newblock A Bradford Book, 2001.

\bibitem{Vondrick2016nips}
Carl Vondrick, Hamed Pirsiavash, and Antonio Torralba.
\newblock Generating videos with scene dynamics.
\newblock In {\em Proceedings of NIPS 2016}, Barcelona, Spain, December 2016.

\bibitem{Wu2017}
Lijun Wu, Yingce Xia, Li~Zhao, Fei Tian, Tao Qin, Jianhuang Lai, and Tie-Yan
  Liu.
\newblock Adversarial neural machine translation.
\newblock In {\em arXiv pre-print}, 2017.

\end{thebibliography}
}

\newpage

\section{Appendix}
\subsection{Proof of Lemma \ref{lem:labeler_r}}
\label{subsec:opt_labeler}
The proof follows the same lines as in the proof for the optimal discriminator. Consider the objective 
\begin{align}
& \rho\expectdo{\log(D_{LR}(x))} + (1-\rho)\expectdz{\log(1-D_{LR}(x)}\nonumber \\
& = \int \rho p_r(x|l=1) \log (D_{LR}(x)) + (1-\rho)  p_r(x | l=0) \log (1-D_{LR}(x))dx \label{eq:labeler_loss2}
\end{align}
Since $0<D_{LR}<1$, $D_{LR}$ that maximizes (\ref{eq:labeler_r_loss}) is given by
\begin{equation}
D_{LR}^*(x) = \frac{\rho p_r(x|l=1)}{p_r(x|l=1)\rho + p_r(x|l=0)(1-\rho)} = \frac{\rho p_r(x|l=1)}{p_r(x)} = p_r(l=1|x)
\end{equation}

\subsection{Proof of Theorem \ref{thm:main}}
\label{sec:main_proof}
Define $C(G)$ as the generator loss for when discriminator, Labeler and Anti-Labeler are at their optimum. $p_{data}$, $p_r$,  $\mathbb{P}_{data}$ and $\mathbb{P}_r$ are used exchangeably for the data distribution. Then we have,
\begin{align}
C(G)&=\expectd{\log(D^*(x))} + \expectg{\log(1-D^*(x))} -  \expectg{\log(D^*(x))} \nonumber\\
&\hspace{0.5in} -(1-\rho)\expectgz{\log(1-D_{LR}(x))} - \rho\expectgo{\log(D_{LR}(x))}  \nonumber\\
&\hspace{0.5in} +(1-\rho)\expectgz{\log(1-D_{LG}(x))} + \rho\expectgo{\log(D_{LG}(x))}\nonumber\\
&= \expectd{\log(D^*(x))} + \expectg{\log(1-D^*(x))} -  \expectg{\log(D^*(x))} \nonumber\\
&\hspace{0.5in} -(1-\rho)\expectgz{\log(\prr{l=0|x})} - \rho\expectgo{\log(\prr{l=1|x})}  \nonumber\\
&\hspace{0.5in} +(1-\rho)\expectgz{\log(\prg{l=0|x})} + \rho\expectgo{\log(\prg{l=1|x})}\nonumber\\
& = \expectd{\log\left(\frac{p_{data}(x)}{p_{data}(x)+p_{g}(x)}\right)} + \expectg{\log\left(\frac{p_{g}(x)}{p_{data}(x)}\right)}\nonumber\\
&\hspace{0.5in} -(1-\rho)\expectgz{\log(\prr{l=0|x})} - \rho\expectgo{\log(\prr{l=1|x})}\nonumber\\
&\hspace{0.5in} +(1-\rho)\expectgz{\log(\prg{l=0|x})} + \rho\expectgo{\log(\prg{l=1|x})}
\end{align}

Using Bayes' rule, we can write $\pr{l=1|x} = \frac{\pr{x|l=1}\rho}{\pr{x}}$ and $\pr{l=0|x} = \frac{\pr{x|l=0}(1-\rho)}{\pr{x}}$. Then we have the following:
\begin{align*}
C(G) &= -1 + \KL{p_r}{\frac{p_r+p_g}{2}} + \KL{p_g}{p_r} + H(\rho)  \nonumber\\
&+ (1-\rho) \KL{p_g^0}{p_r^0} + \rho\KL{p_g^1}{p_r^1} - (1-\rho)\KL{p_g^0}{p_r} -\rho\KL{p_g^1}{p_r}\nonumber\\
& +(1-\rho)\expectgz{\log\left(\frac{p_g^0 (1-\rho)}{p_g}\right)} + \rho\expectgo{\log\left(\frac{p_g^1 \rho}{p_g}\right)},
\end{align*}
where $H(\rho)$ stands for the binary entropy function. Notice that we have%- (1-\rho)\KL{p_g^0}{p_r} &-\rho\KL{p_g^1}{p_r} =
\begin{align*}
&- (1-\rho)\KL{p_g^0}{p_r} -\rho\KL{p_g^1}{p_r} &&\nonumber\\
 &=-(1-\rho)\int p_g^0(x)\log(p_g^0(x))dx  - \rho \int p_g^1\log(p_g^1(x))dx & &+ (1-\rho)\int p_g^0(x)\log(p_r(x))dx \\
 & & &+  \rho \int p_g^1(x)\log(p_r(x))dx\\
&=  -(1-\rho)\int p_g^0(x)\log(p_g^0(x))dx  - \rho \int p_g^1\log(p_g^1(x))dx && + \int p_g(x)\log(p_r(x))dx \\
&=  -(1-\rho)\int p_g^0(x)\log(p_g^0(x))dx  - \rho \int p_g^1\log(p_g^1(x))dx
 &&- \KL{p_g}{p_r} +\int p_g(x) \log(p_g(x))dx .
\end{align*}
Also notice that we have
\begin{align*}
& (1-\rho)\expectgz{\log\left(\frac{p_g^0 (1-\rho)}{p_g}\right)} + \rho\expectgo{\log\left(\frac{p_g^1 \rho}{p_g}\right)}\nonumber\\
& = -\int p_g(x) \log(p_g(x))dx - H(\rho) + (1-\rho) \int p_g^0(x)\log(p_g^0(x))dx + \rho \int p_g^1(x)\log(p_g^1(x))dx
\end{align*}
%  asdf
 %= -\KL{p_g}{p_r}& &
Substituting this into the above equation and combining terms, we get
\begin{equation*}
C(G) = -1 + \KL{p_r}{\frac{p_r+p_g}{2}} + (1-\rho) \KL{p_g^0}{p_r^0} + \rho\KL{p_g^1}{p_r^1} \\
\end{equation*}
Observe that for $p_g^0 = p_r^0$ and $p_g^1 = p_g^1$, we have $p_g = p_r$, yielding $C(G) = -1$. Finally, since KL divergence is always non-negative we have $C(G)\geq -1 $, concluding the proof.
\hfill \qed

\subsection{Proof of Corollary \ref{cor:main}}
Since $C$ is a causal implicit generative model for the causal graph $D$, by definition it is consistent with the causal graph $D$. Since in a conditional GAN, generator $G$ is given the noise terms and the labels, it is easy to see that the concatenated generator neural network $G(C(Z_1),Z_2)$ is consistent with the causal graph $D'$, where $D' = (\mathcal{V} \cup \{Image\}, E \cup \{(V_1,Image),(V_2,Image),\hdots (V_n,Image)\})$. Assume that $C$ and $G$ are perfect, i.e., they sample from the true label joint distribution and conditional image distribution. Then the joint distribution over the generated labels and image is the true distribution since $\pr{Image,Label} = \pr{Image|Label}\pr{Label}$. By Proposition \ref{prop:causal}, the concatenated model can sample from the true observational and interventional distributions. Hence, the concatenated model is a causal implicit generative model for graph $D'$.

\subsection{CausalGAN Architecture and Loss for Multiple Labels}
\label{sec:generalized_proof}
In this section, we explain the modifications required to extend the proof to the case with multiple binary labels, or a label variable with more than 2 states in general. $p_{data}$, $p_r$,  $\mathbb{P}_{data}$ and $\mathbb{P}_r$ are used exchangeably for the data distribution in the following.

Consider Figure \ref{fig:architecture} in the main text. Labeler outputs the scalar $D_{LR}(x)$ given an image $x$. With the given loss function in (\ref{eq:labeler_r_loss}), i.e., when there is a single binary label $l$, when we show in Section \ref{subsec:opt_labeler} that the optimum Labeler $D^*_{LR}(x)=p_r(l=1|X=x)$. We first extend the Labeler objective as follows: Suppose we have $d$ binary labels. Then we allow the Labeler to output a $2^d$ dimensional vector $D_{LR}(x)$, where $D_{LR}(x)[i]$ is the $i^{th}$ coordinate of this vector. The Labeler then solves the following optimization problem:
\begin{equation}
\label{eq:labeler_r_loss_generalized}
\max\limits_{D_{LR}} \sum_{j=1}^{2^d} \rho_j\mathbb{E}_{p_r^j}{\log(D_{LR}(x)[j])},
\end{equation}
\noindent
where $p_r^j(x) \coloneqq \mathbb{P}_{r}(X=x|l=j) $ and $\rho_j = \mathbb{P}_{r}(l=j)$. We have the following Lemma:
\begin{lemma}
\label{lem:opt_labeler_generalized}
Consider a Labeler $D_{LR}$ that outputs the $2^d$-dimensional vector $D_{LR}(x)$ such that $\sum_{j=1}^{2^d}D_{LR}(x)[j] = 1$, where $x\sim p_r(x,l)$. Then the optimum Labeler with respect to the loss in (\ref{eq:labeler_r_loss_generalized}) has $D_{LR}^*(x)[j]= p_r(l=j|x)$.
\end{lemma}
\begin{proof}
Suppose $p_r(l=j|x)=0$ for a set of (label, image) combinations. Then $p_r(x,l=j)=0$, hence these label combinations do not contribute to the expectation. Thus, without loss of generality, we can consider only the combinations with strictly positive probability. We can also restrict our attention to the functions $D_{LR}$ that are strictly positive on these (label,image) combinations; otherwise, loss becomes infinite, and as we will show we can achieve a finite loss. Consider the vector $D_{LR}(x)$ with coordinates $D_{LR}(x)[j]$ where $j\in[2^d]$. Introduce the discrete random variable $Z_x\in [2^d]$, where $\pr{Z_x = j} = D_{LR}(x)[j]$. The Labeler loss can be written as
\begin{align}
&\min -\mathbb{E}_{(x,l)\sim p_r(x,l)}\log(\pr{Z_x=j})\\
&=\min\mathbb{E}_{x\sim p_r(x)} \KL{L_x}{Z_x} - H(L_x),
\end{align}
where $L_x$ is the discrete random variable such that $\pr{L_x = j} = \prr{l=j|x}$. $H(L_x)$ is the Shannon entropy of $L_x$, and it only depends on the data. Since KL divergence is greater than zero and $p(x)$ is always non-negative, the loss is lower bounded by $-H(L_x)$. Notice that this minimum can be achieved by satisfying $\pr{Z_x = j} = \prr{l=j|x}$. Since KL divergence is minimized if and only if the two random variables have the same distribution, this is the unique optimum, i.e., $D_{LR}^*(x)[j] = \prr{l=j|x}$.
%Hence, the optimum $D_{LR}^*(x)$ should minimize each KL divergence term inside the expectation, due to the fact that . . Therefore, the optimum $D_{LR}^*(x)$ satisfies 

\end{proof}
The lemma above simply states that the optimum Labeler network will give the posterior probability of a particular label combination, given the observed image. In practice, the constraint that the coordinates sum to 1 could be satisfied by using a softmax function in the implementation. Next, we have the corresponding loss function and lemma for the Anti-Labeler network. The Anti-Labeler solves the following optimization problem
\begin{equation}
\label{eq:labeler_g_loss_generalized}
\max\limits_{D_{LG}} \sum_{j=1}^{2^d} \rho_j\mathbb{E}_{p_g^j}{\log(D_{LG}(x)[j])},
\end{equation}
\noindent
where $p_g^j(x) \coloneqq \pr{G(z,l)=x|l=j} $ and $\rho_j = \pr{l=j}$. We have the following Lemma:
\begin{lemma}
The optimum Anti-Labeler has $D_{LG}^*(x)[j]= \prg{l=j|x}$.
\end{lemma}
\begin{proof}
The proof is the same as the proof of Lemma \ref{lem:opt_labeler_generalized}, since Anti-Labeler does not have control over the joint distribution between the generated image and the labels given to the generator, and cannot optimize the conditional entropy of labels given the image under this distribution.
\end{proof}

For a fixed discriminator, Labeler and Anti-Labeler, the generator solves the following optimization problem:
\begin{align}
\label{eq:gen_loss_multiple_labels}
&\min\limits_{G} \expectd{\log(D(x))} + \expectg{\log\left(\frac{1-D(x)}{D(x)}\right)} \nonumber\\
&- \sum_{j=1}^{2^d}\rho_j\expect{x\sim p_g^j(x)}{ \log(D_{LR}(X)[j])} \nonumber \\
+& \sum_{j=1}^{2^d}\rho_j \expect{x\sim p_g^j(x)}{  \log(D_{LG}(X)[j])}.
\end{align}
We then have the following theorem, that shows that the optimal generator samples from the class conditional image distributions given a particular label combination:

\begin{theorem}[Theorem 1 formal for multiple binary labels]
\label{thm:main_generalized}
Define $C(G)$ as the generator loss for when discriminator, Labeler and Anti-Labeler are at their optimum obtained from (\ref{eq:gen_loss_multiple_labels}). The global minimum of the virtual training criterion $C(G)$  is achieved if and only if $p_g^j=p_{\text{data}}^j, \forall j\in[2^d]$, i.e., if and only if given a $d$-dimensional label vector $l$, generator samples from the class conditional image distribution, i.e., $\pr{G(z,l)=x} =p_{data}(x|l)$.
\end{theorem}
\begin{proof}
Substituting the optimum values for the Discriminator, Labeler and Anti-Labeler networks, we get the virtual training criterion $C(G)$ as
\begin{align}
C(G)&=\expectd{\log(D^*(x))} + \expectg{\log(1-D^*(x)} -  \expectg{\log(D^*(x)} \nonumber\\
&\hspace{0.5in}- \sum_{j=1}^{2^d}\rho_j \mathbb{E}_{x\sim p_g^j(x)}\log(D_{LR}^*(x)[j]))  \nonumber\\
&\hspace{0.5in} +\sum_{j=1}^{2^d}\rho_j \mathbb{E}_{x\sim p_{g}^j(x)}\log(D_{LG}^*(x)[j])\nonumber\\
& = \expectd{\log\left(\frac{p_{data}(x)}{p_{data}(x)+p_{g}(x)}\right)} + \expectg{\log\left(\frac{p_{g}(x)}{p_{data}(x)}\right)}\nonumber\\
&\hspace{0.5in}- \sum_{j=1}^{2^d}\rho_j \mathbb{E}_{x\sim p_g^j(x)}\log(p_r(l=j|X=x))  \nonumber\\
&\hspace{0.5in} +\sum_{j=1}^{2^d}\rho_j \mathbb{E}_{x\sim p_{g}^j(x)}\log(p_g(l=j|X=x))\nonumber\\
\end{align}

Using Bayes' rule, we can write $\pr{l=j|x} = \frac{\pr{x|l=j}\rho_j}{\pr{x}}$. Then we have the following:
\begin{align*}
C(G) &= \expectd{\log\left(\frac{p_{data}(x)}{p_{data}(x)+p_{g}(x)}\right)} + \expectg{\log\left(\frac{p_{g}(x)}{p_{data}(x)}\right)}\nonumber\\
&\hspace{0.5in} -\sum_{j=1}^{2^d}\rho_j \mathbb{E}_{x\sim p_g^j(x)}\log\left(\frac{p_r^j(x)\rho_j}{p_r(x)}\right)  \nonumber\\
&\hspace{0.5in} +\sum_{j=1}^{2^d} \rho_j\mathbb{E}_{x\sim p_{g}^j(x)}\log\left(\frac{p_g^j(x)\rho_j}{p_g(x)}\right),\nonumber\\
 &= -1 + \KL{p_r}{\frac{p_r+p_g}{2}} + \KL{p_g}{p_r} + H(l)  \nonumber\\
&+ \sum_{j=1}^{2^d}\rho_j \KL{p_g^j}{p_r^j} - \sum_{j=1}^{2^d} \rho_j\KL{p_g^j}{p_r} \nonumber\\
& +\sum_{j=1}^{2^d}\rho_j \mathbb{E}_{x\sim p_{g}^j(x)}\log\left(\frac{p_g^j(x)\rho_j}{p_g(x)}\right).
\end{align*}
Notice that we have
\begin{align*}
&- \sum_{j=1}^{2^d} \rho_j\KL{p_g^j}{p_r}\nonumber\\
& =-\sum_{j=1}^{2^d} \rho_j \int p_g^{j}\log(p_g^j)dx - \KL{p_g}{p_r} + \int p_g(x)\log(p_g(x))dx
\end{align*}
Also notice that we have
\begin{align*}
& \sum_{j=1}^{2^d}\rho_j \mathbb{E}_{x\sim p_{g}^j(x)}\log(\frac{p_g^j(x)\rho_j}{p_g(x)})\nonumber\\
& = -\int p_g(x) \log(p_g(x))dx - H(l) + \sum_{j=1}^{2^d}\rho_j \int p_g^j(x)\log(p_g^j(x))dx
\end{align*}
Substituting this into the above equation and combining terms, we get
\begin{equation*}
C(G) = -1 + \KL{p_r}{\frac{p_r+p_g}{2}} + \sum_{j=1}^{2^d} \rho_j\KL{p_g^j}{p_r^j}\\
\end{equation*}
Observe that for $p_g^j = p_r^j, \forall j\in [d]$, we have $p_g = p_r$, yielding $C(G) = -1$. Finally, since KL divergence is always non-negative we have $C(G)\geq -1 $, concluding the proof.

\end{proof}

\subsection{Alternate CausalGAN Architecture for $d$ Labels}
\label{sec:alternate_practical_architecture}
In this section, we provide the theoretical guarantees for the implemented CausalGAN architecture with $d$ labels. Later we show that these guarantees are sufficient to prove that the global optimal generator samples from the class conditional distributions for a practically relevant class of distributions.

First, let us restate the loss functions more formally. Note that $D_{LR}(x), D_{LG}(x)$ are $d-$dimensional vectors. The Labeler solves the following optimization problem:
\begin{equation}
\label{eq:labeler_r_loss_alternative}
\max\limits_{D_{LR}} \rho_j\mathbb{E}_{x\sim p_r^{j1}}\log(D_{LR}(x)[j]) + (1-\rho_j)\mathbb{E}_{x\sim p_r^{j0}}\log(1-D_{LR}(x)[j]).
\end{equation}
\noindent
where $p_r^{j0}(x) \coloneqq \pr{X=x|l_j=0} $, $p_r^{j0}(x) \coloneqq \pr{X=x|l_j=0} $ and $\rho_j = \pr{l_j=1}$. For a fixed generator, the Anti-Labeler solves the following optimization problem:
\begin{equation}
\label{eq:labeler_g_loss_alternative}
\max\limits_{D_{LG}} \rho_j\mathbb{E}_{p_g^{j1}}\log(D_{LG}(x)[j]) + (1-\rho_j)\mathbb{E}_{p_g^{j0}}\log(1-D_{LG}(x)[j]),
\end{equation}
where $p_g^{j0}(x) \coloneqq \mathbb{P}_g(x|l_j=0) $, $p_g^{j0}(x) \coloneqq \mathbb{P}_g(x|l_j=0) $. For a fixed discriminator, Labeler and Anti-Labeler, the generator solves the following optimization problem:
\begin{align}
\label{eq:gen_loss_alternative}
&\min\limits_{G} \expectd{\log(D(x))} + \expectg{\log\left(\frac{1-D(x)}{D(x)}\right)} \nonumber\\
&- \frac{1}{d}\sum_{j=1}^d\rho_j\expect{x\sim p_g^{j1}(x)}{ \log(D_{LR}(X)[j])} - (1-\rho_j)\expect{x\sim p_g^{j0}(x)}{\log(1-D_{LR}(X)[j])}\nonumber \\
+& \frac{1}{d}\sum_{j=1}^d \rho_j \expect{x\sim p_g^{j1}(x)}{  \log(D_{LG}(X)[j])} + (1-\rho_j) \expect{x\sim p_g^{j0}(x)}{ \log(1-D_{LG}(X)[j])}.
\end{align}

We have the following proposition, which characterizes the optimum generator, for optimum Labeler, Anti-Labeler and Discriminator:
\begin{proposition}
\label{prop:alternate}
Define $C(G)$ as the generator loss for when discriminator, Labeler and Anti-Labeler are at their optimum obtained from (\ref{eq:gen_loss_alternative}). The global minimum of the virtual training criterion $C(G)$  is achieved if and only if $p_g(x|l_i) = p_r(x|l_i), \forall i\in[d]$ and $p_g(x)=p_r(x)$.
\end{proposition}
\begin{proof}
Proof follows the same lines as in the proof of Theorem \ref{thm:main} and Theorem \ref{thm:main_generalized} and is omitted.
\end{proof}
%Although the implemented version does not come with the guarantee that the global optimum generator samples from class conditional distribution given the label vector, the above theorem states that it comes with the following guarantees: 

Thus we have
\begin{equation}
p_r(x,l_i) = p_g(x,l_i), \forall i\in [d]\text{  and  } p_r(x)=p_g(x).
\end{equation}
\noindent
However, this does not in general imply $p_r(x,l_1,l_2,\hdots,l_d) = p_g(x,l1,l_2,\hdots,l_d)$, which is equivalent to saying the generated distribution samples from the class conditional image distributions.  To guarantee the correct conditional sampling given all labels, we introduce the following assumption: We assume that the image $x$ determines all the labels. This assumption is very relevant in practice. For example, in the CelebA dataset, which we use, the label vector, e.g., whether the person is a male or female, with or without a mustache, can be thought of as a deterministic function of the image. When this is true, we can say that $p_r(l_1,l_2,\hdots, l_n|x) = p_r(l_1|x)p_r(l_2|x)\hdots p_r(l_n|x)$. 

We need the following lemma, where kronecker delta function refers to the functions that take the value of 1 only on a single point, and 0 everywhere else:
\begin{lemma}
\label{lem:helper_alternate}
Any discrete joint probability distribution, where all the marginal probability distributions are kronecker delta functions is the product of these marginals.
\end{lemma}
\begin{proof}
Let $\delta_{\{x-u\}}$ be the kronecker delta function which is 1 if $x=u$ and is 0 otherwise. Consider a joint distribution $p(X_1,X_2,\hdots,X_n)$, where $p(X_i) = \delta_{\{X_i-u_i\}}, \forall i\in [n]$, for some set of elements $\{u_i\}_{i\in [n]}$. We will show by contradiction that the joint probability distribution is zero everywhere except at $(u_1,u_2,\hdots,u_n)$. Then, for the sake of contradiction, suppose for some $v=(v_1,v_2,\hdots,v_n)\neq (u_1,u_2,\hdots,u_n)$, $p(v_1,v_2,\hdots,v_n)\neq 0$. Then $\exists j\in[n]$ such that $v_j\neq u_j$. Then we can marginalize the joint distribution as
\begin{equation}
p(v_j)=\sum_{X_1,\hdots,X_{j-1},X_j,\hdots,X_n}p(X_1,\hdots,X_{j-1},v_j,X_{j+1},\hdots,X_n) > 0,
\end{equation}
where the inequality is due to the fact that the particular configuration $(v_1,v_2,\hdots, v_n)$ must have contributed to the summation. However this contradicts with the fact that $p(X_j) = 0, \forall X_j\neq u_j$. Hence, $p(.)$ is zero everywhere except at $(u_1,u_2,\hdots,u_n)$, where it should be 1.
\end{proof}
We can now simply apply the above lemma on the conditional distribution $p_g(l_1,l_2,\hdots,l_d|x)$. Proposition \ref{prop:alternate} shows that the image distributions and the marginals $p_g(l_i|x)$ are true to the data distribution due to Bayes' rule. Since the vector $(l_1,\hdots,l_n)$ is a deterministic function of $x$ by assumption, $p_r(l_i|x)$ are kronecker delta functions, and so are $p_g(l_i|x)$ by Proposition \ref{prop:alternate}. Thus, since the joint $p_g(x,l_1,l_2,\hdots,l_d)$ satisfies the condition that every marginal distribution $p(l_i|x)$ is a kronecker delta function, then it must be a product distribution by Lemma \ref{lem:helper_alternate}. Thus we can write 
\begin{align*}
p_g(l_1,l_2,\hdots,l_d|x) = p_g(l_1|x)p_g(l_2|x)\hdots p_g(l_n|x).
\end{align*}
Then we have the following chain of equalities.

\begin{align*}
p_r(x,l_1,l_2,\hdots,l_d) &= p_r(l_1,\hdots,l_n|x)p_r(x)\\
& = p_r(l_1|x)p_r(l_2|x)\hdots p_r(l_n|x)p_r(x)\\
& = p_g(l_1|x)p_g(l_2|x)\hdots p_g(l_n|x)p_g(x)\\
& = p_g(l_1,l_2,\hdots,l_d|x)p_g(x)\\
& = p_g(x,l_1,l_2,\hdots,l_d).
\end{align*}

Thus, we also have $p_r(x|l_1,l_2,\hdots,l_n) = p_g(x|l_1,l_2,\hdots,l_n)$ since $p_r(l_1,l_2,\hdots,l_n)=p_g(l_1,l_2,\hdots,l_n)$, concluding the proof that the optimum generator samples from the class conditional image distributions.
%In practice, it can be shown that these statements are sufficient to match the true joint distribution, as long as there is a deterministic map from the image to the labels, which seems to be the case for the specific dataset we have used. 

\subsection{Additional Simulations for Causal Controller}
\label{app:cc}
First, we evaluate the effect of using the wrong causal graph on  an artificially generated dataset. Figure \ref{fig:syn_scatter} shows the scatter plot for the two coordinates of a three dimensional distribution. As we observe, using the correct graph gives the closest scatter plot to the original data, whereas using the wrong Bayesian network, collider graph, results in a very different distribution.

\begin{figure}[t]
\captionsetup{justification=centering}

\subfloat[ $X_1\rightarrow X_2 \rightarrow X_3$]
{\label{fig:syn_data}
\includegraphics[width=0.19\linewidth]%{synthetic/0818_072052/x1x3_notitledata.pdf}}
{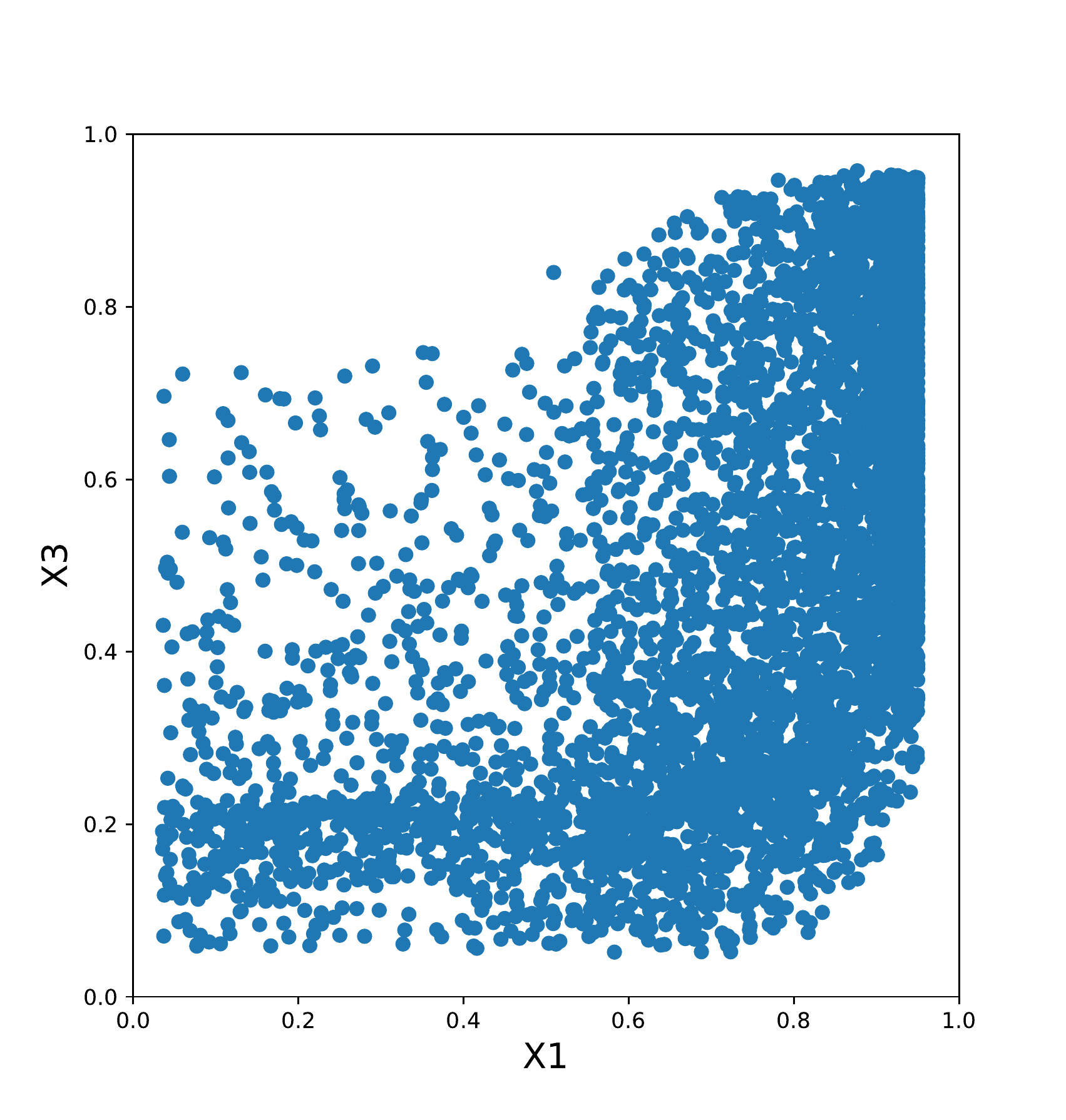}}
\subfloat[$X_1\rightarrow X_2 \rightarrow X_3$]
{\label{fig:syn_linear}
\includegraphics[width=0.19\linewidth]%{synthetic/0818_072052/x1x3_notitlelinear.pdf}}
{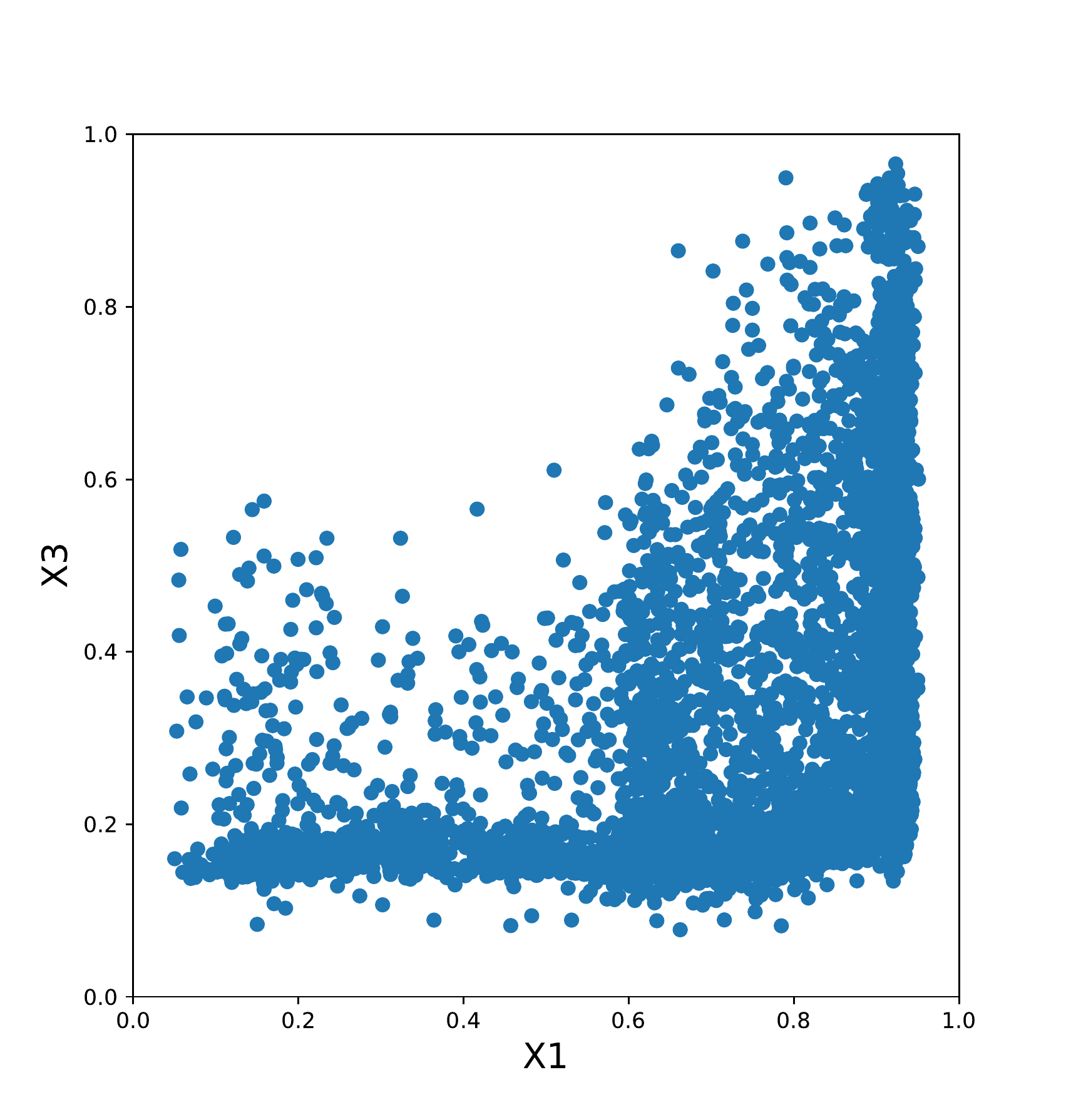}}
\subfloat[$X_1\rightarrow X_2\rightarrow X_3$ $X_1\rightarrow X_3$]
{\label{fig:syn_complete}
\includegraphics[width=0.19\linewidth]%{synthetic/0818_072052/x1x3_notitlecomplete.pdf}}
{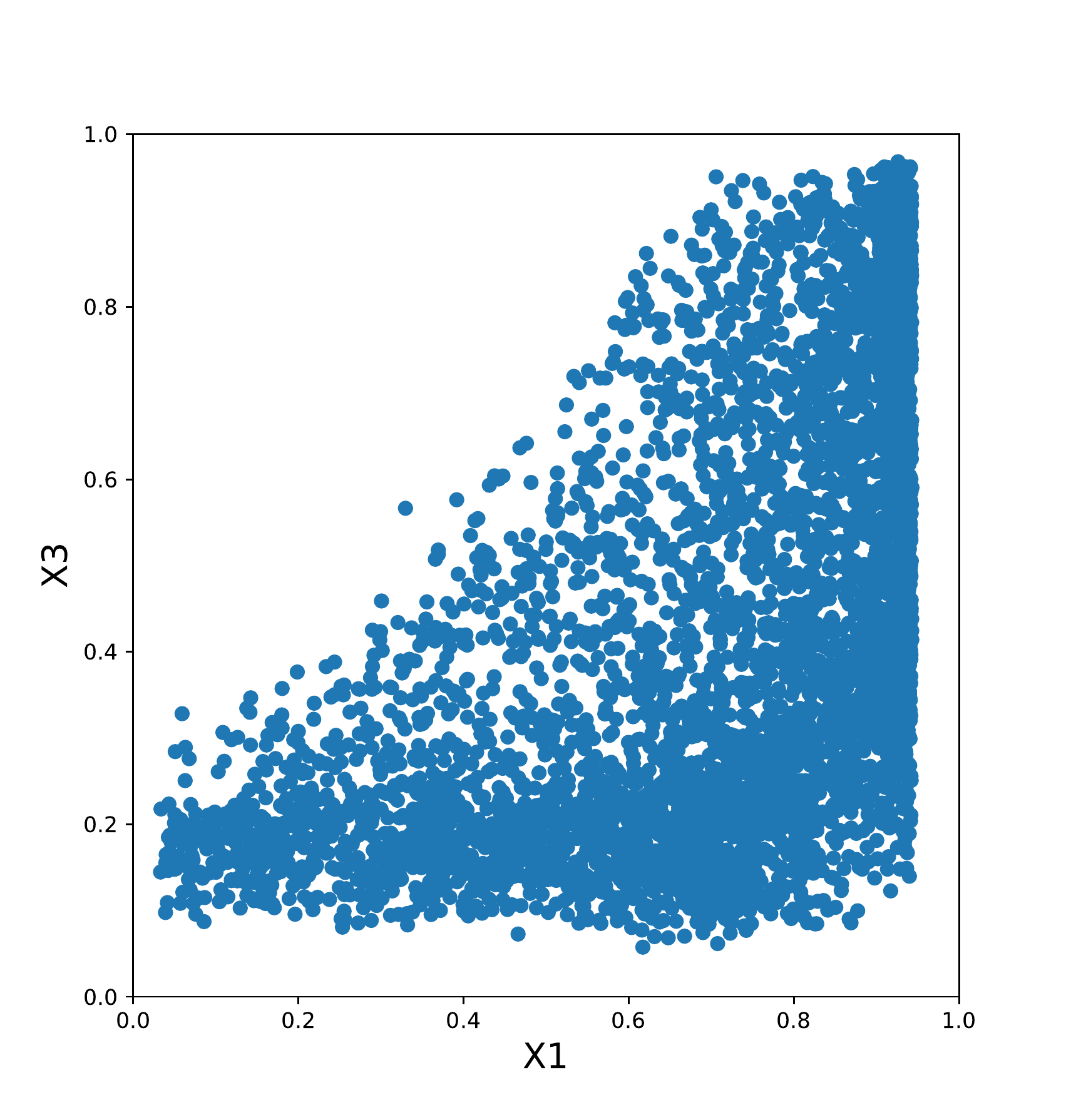}}%\endgroup
\subfloat[$X_1\rightarrow X_2\leftarrow X_3$]
{\label{fig:syn_collider}
\includegraphics[width=0.19\linewidth]%{synthetic/0818_072052/x1x3_notitlecollider.pdf}}
{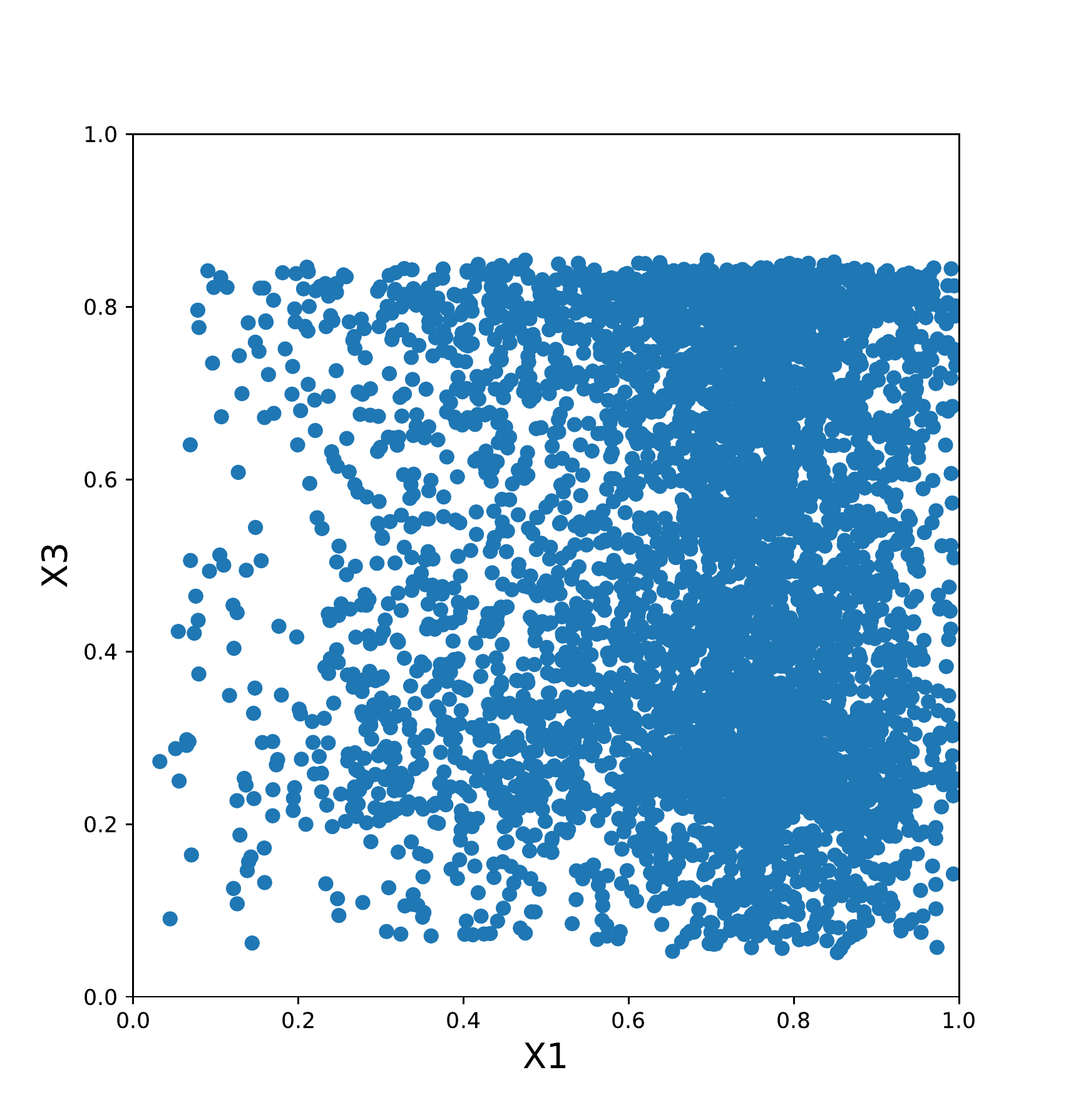}}
\subfloat[Fully connected]
{\label{fig:syn_fc}
\includegraphics[width=0.19\linewidth]%{synthetic/0818_072052/x1x3_notitlefc5.pdf}}
{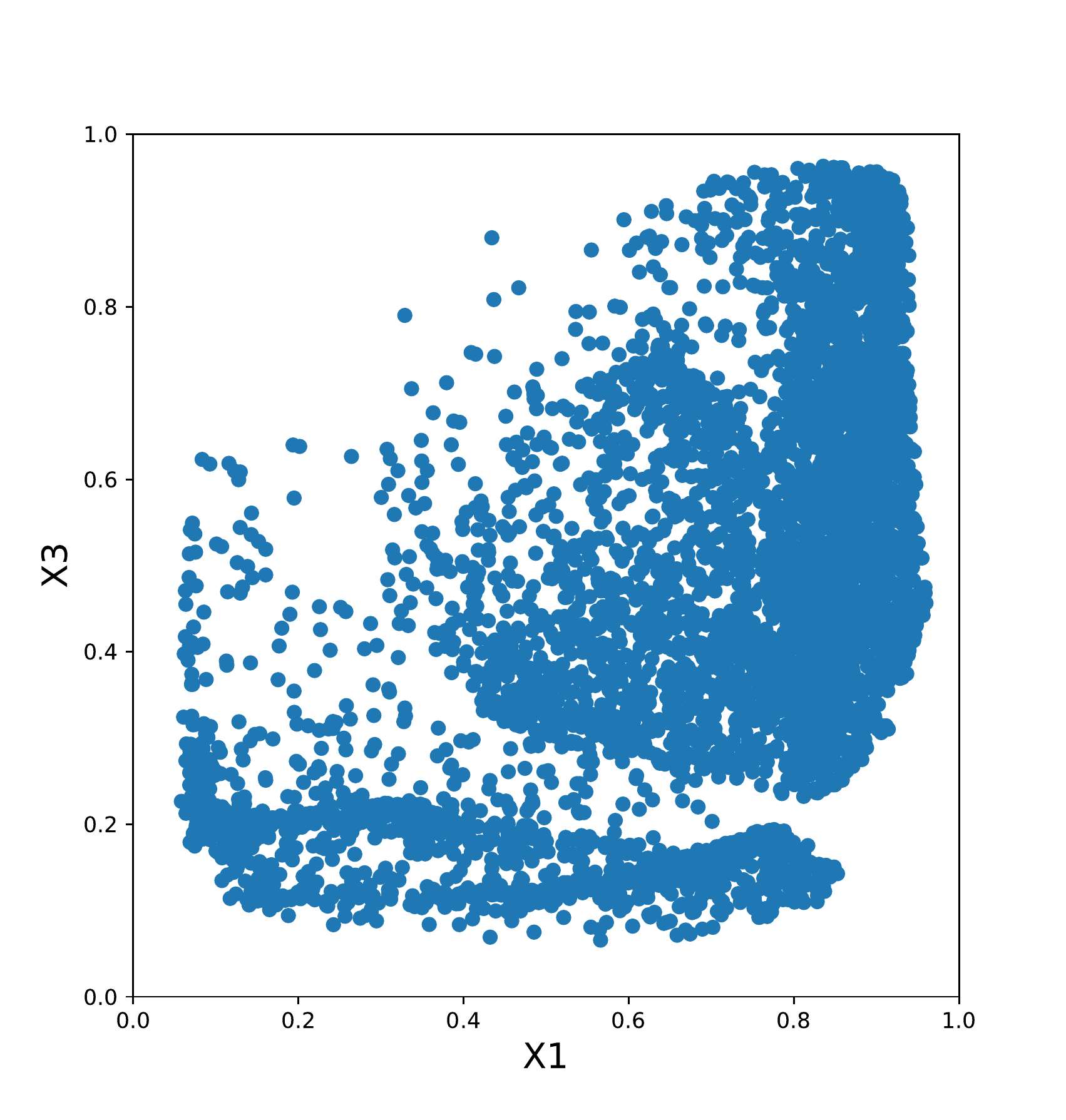}}
\caption{Synthetic data experiments: (a) Scatter plot for actual data. Data is generated using the causal graph $X_1\rightarrow X_2 \rightarrow X_3$.  (b) Generated distribution when generator causal model is $X_1\rightarrow X_2 \rightarrow X_3$. (c) Generated distribution when generator causal model is $X_1\rightarrow X_2 \rightarrow X_3$ $X_1\rightarrow X_3$. (d) Generated distribution when generator causal model is $X_1\rightarrow X_2 \leftarrow X_3$. (e) Generated distribution when generator is from a fully connected last layer of a 5 layer FF neural net.}
\label{fig:syn_scatter}
\end{figure}

Second, we expand on the causal graphs used for experiments for the CelebA dataset. The graph Causal Graph 1 (G1) is as illustrated in Figure \ref{fig:big_causal_graph}. The graph cG1, which is a completed version of G1, is the complete graph associated with the ordering: Young, Male, Eyeglasses, Bald, Mustache, Smiling, Wearing Lipstick, Mouth Slightly Open, Narrow Eyes. For example, in cG1 Male causes Smiling because Male comes before Smiling in the ordering. The graph rcG1 is associated with the reverse ordering.

Next, we check the effect of using the incorrect Bayesian network for the data. The causal graph G1 generates Male and Young independently, which is incorrect in the data. Comparison of pairwise distributions in Table \ref{tab:pretrain_pairMY} demonstrate that for G1 a reasonable approximation to the true distribution is still learned for \{Male, Young\} jointly. For cG1 a nearly perfect distributional approximation is learned. Furthermore we show that despite this inaccuracy, both graphs G1 and cG1 lead to Causal Controllers that never output the label combination \{Female,Mustache\}, which will be important later.

\begin{table}
\centering
\begin{tabular}{c|c|c|c|}
\hline
\multicolumn{2}{|c|}{Label} & \multicolumn{2}{ c| }{Male} \\ \cline{3-4}
\multicolumn{2}{|c|}{Pair} & 0 & 1\\ \hline
\multicolumn{1}{|c|}{\multirow{2}{*}{Young} } & 0 & 0.14[0.07](0.07) & 0.09[0.15](0.15) \\ 
\multicolumn{1}{|c|}{} & 1 &0.47[0.51](0.51) & 0.29[0.27](0.26) \\ \hline
\multicolumn{1}{|c|}{ \multirow{2}{*}{Mustache} } & 0 & 0.61[0.58](0.58) & 0.34[0.38](0.38) \\
\multicolumn{1}{|c|}{} & 1 & 0.00[0.00](0.00) & 0.04[0.04](0.04) \\ \hline

\end{tabular}
\caption{Pairwise marginal distribution for select label pairs when Causal Controller is trained on $G1$ in plain text, its completion $cG1$[square brackets], and the true pairwise distribution(in parentheses). Note that $G1$ treats Male and Young labels as independent, but does not completely fail to generate a reasonable (product of marginals) approximation. Also note that when an edge is added $Young\rightarrow Male$, the learned distribution is nearly exact. Note that both graphs contain the edge $Male\rightarrow Mustache$ and so are able to learn that women have no mustaches.}

\label{tab:pretrain_pairMY}
\end{table}

Wasserstein GAN in its original form (with Lipshitz discriminator) assures convergence in distribution of the Causal Controller output to the discretely supported distribution of labels. We use a slightly modified version of Wasserstein GAN with a penalized gradient\cite{Gulrajani2017}. We first demonstrate that learned outputs actually have "approximately discrete" support. In Figure \ref{fig:cc_discrete}, we sample the joint label distribution 1000 times, and make a histogram of the (all) scalar outputs corresponding to any label.

Although Figure \ref{fig:TVDiters} demonstrates conclusively good convergence for both graphs, TVD is not always intuitive. For example, "how much can each marginal be off if there are 9 labels and the TVD is 0.14?". To expand upon Figure \ref{tab:pretrain_marg} where we showed that the causal controller learns the correct distribution for a pairwise subset of nodes, here we also show that both Causal Graph 1 (G1) and the completion we define (cG1) allow training of very reasonable marginal distributions for all labels (Table \ref{tab:pretrain_pairMY}) that are not off by more than 0.03 for the worst label. $\mathbb{P}_D{(L=1)}$ is the probability that the label is 1 in the dataset, and $\mathbb{P}_G{(L=1)}$ is the probability that the generated label is (around a small neighborhood of ) 1.

\begin{table}
\centering
    \begin{tabular}{ | c | c | c | c |}
    \hline
    \textbf{Label, L}  & $\mathbb{P}_{G1}{(L=1)}$ & $\mathbb{P}_{cG1}{(L=1)}$ & $\mathbb{P}_D{(L=1)}$ \\ \hline  \hline
    Bald & 0.02244 & 0.02328 & 0.02244 \\ \hline
    Eyeglasses & 0.06180 & 0.05801 & 0.06406  \\ \hline
    Male & 0.38446 & 0.41938 & 0.41675\\ \hline
    Mouth Slightly Open & 0.49476 & 0.49413 & 0.48343 \\ \hline
    Mustache & 0.04596 & 0.04231 & 0.04154 \\ \hline
    Narrow Eyes & 0.12329 & 0.11458 & 0.11515 \\ \hline
    Smiling & 0.48766 & 0.48730 & 0.48208 \\ \hline
    Wearing Lipstick & 0.48111 & 0.46789 & 0.47243 \\ \hline
    Young & 0.76737 & 0.77663 & 0.77362 \\ \hline
    \end{tabular}
\caption{Marginal distribution of pretrained Causal Controller labels when Causal Controller is trained on Causal Graph 1($P_{G1}$) and its completion($P_{cG1}$), where $cG1$ is the (nonunique) largest DAG containing $G1$ (see appendix). The third column lists the actual marginal distributions in the dataset}
\label{tab:pretrain_marg}
\end{table}

\subsection{Additional Simulations for CausalGAN}
\label{dcgan:cig}
In this section, we provide additional simulations for CausalGAN. In Figures \ref{fig:bald_sweep}-\ref{fig:eyeglasses_sweep}, we show the conditional image generation properties of CausalGAN by sweeping a single label from 0 to 1 while keeping all other inputs/labels fixed. In Figure \ref{fig:dcgan_diversity}, to examine the degree of mode collapse and show the image diversity, we show 256 randomly sampled images.

\begin{figure}
\centering
\subfloat[Interpolating Bald label]
{\label{fig:bald_sweep}
\includegraphics[width=0.45\linewidth]%{./dcgan_pictures/celebA_0829_110623_no_goz/label_interpolation/45507_G_interp_Bald.pdf}
{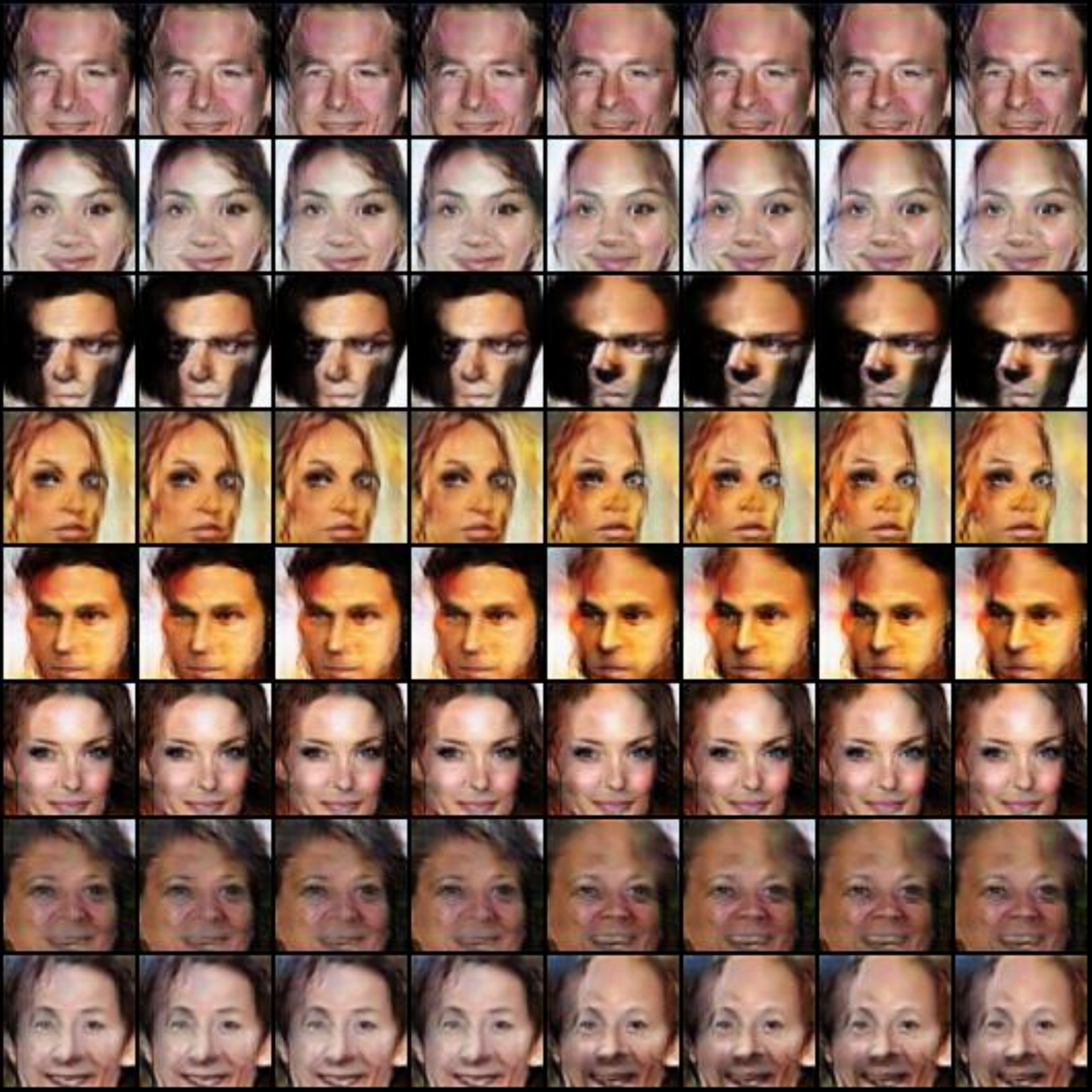}}
\subfloat[Interpolating Male label]
{\label{fig:male_sweep}
\includegraphics[width=0.45\linewidth]%{./dcgan_pictures/celebA_0829_110623_no_goz/label_interpolation/45507_G_interp_Male.pdf}
{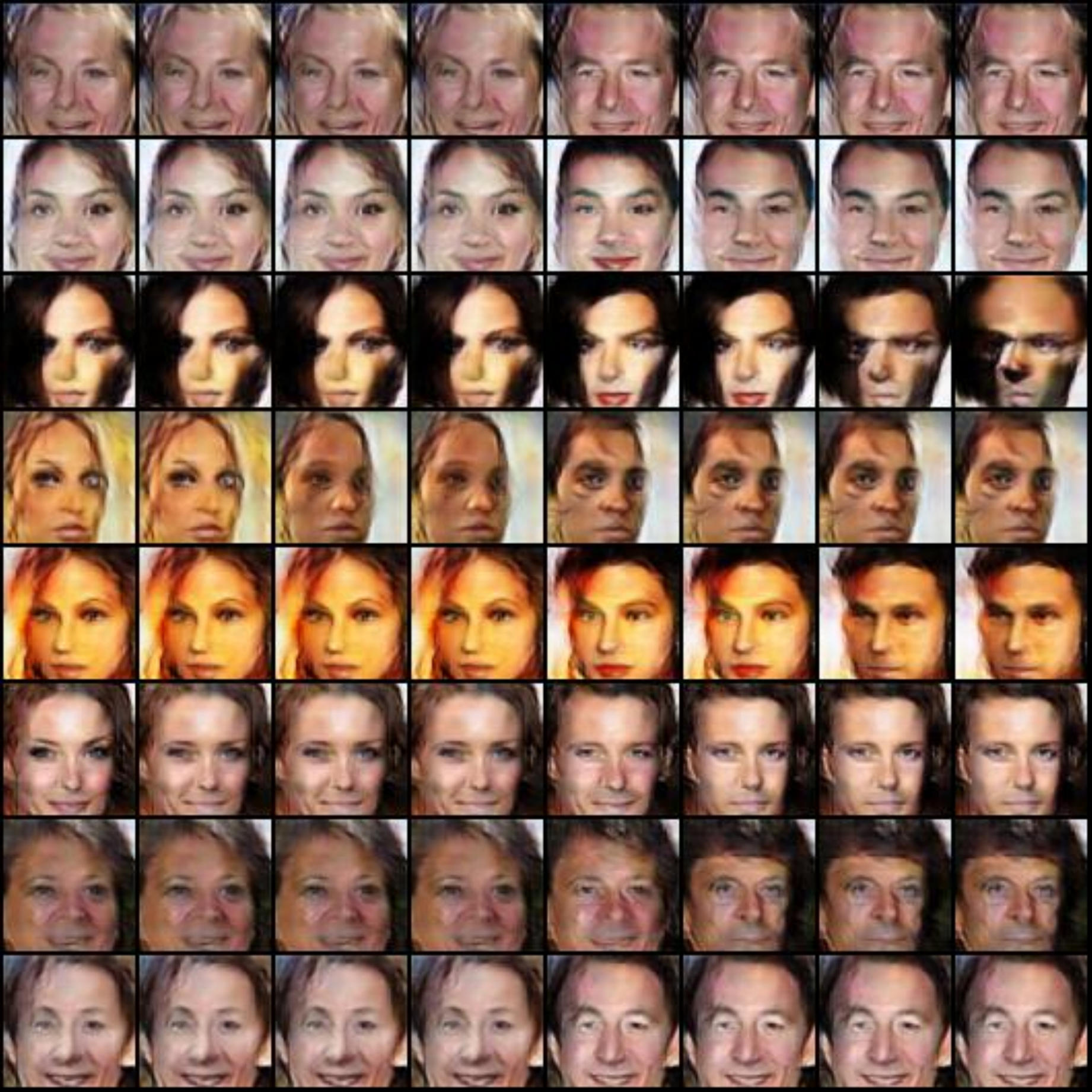}}
\\
\subfloat[Interpolating Young label]
{\label{fig:young_sweep}
\includegraphics[width=0.45\linewidth]%{./dcgan_pictures/celebA_0829_110623_no_goz/label_interpolation/45507_G_interp_Young.pdf}
{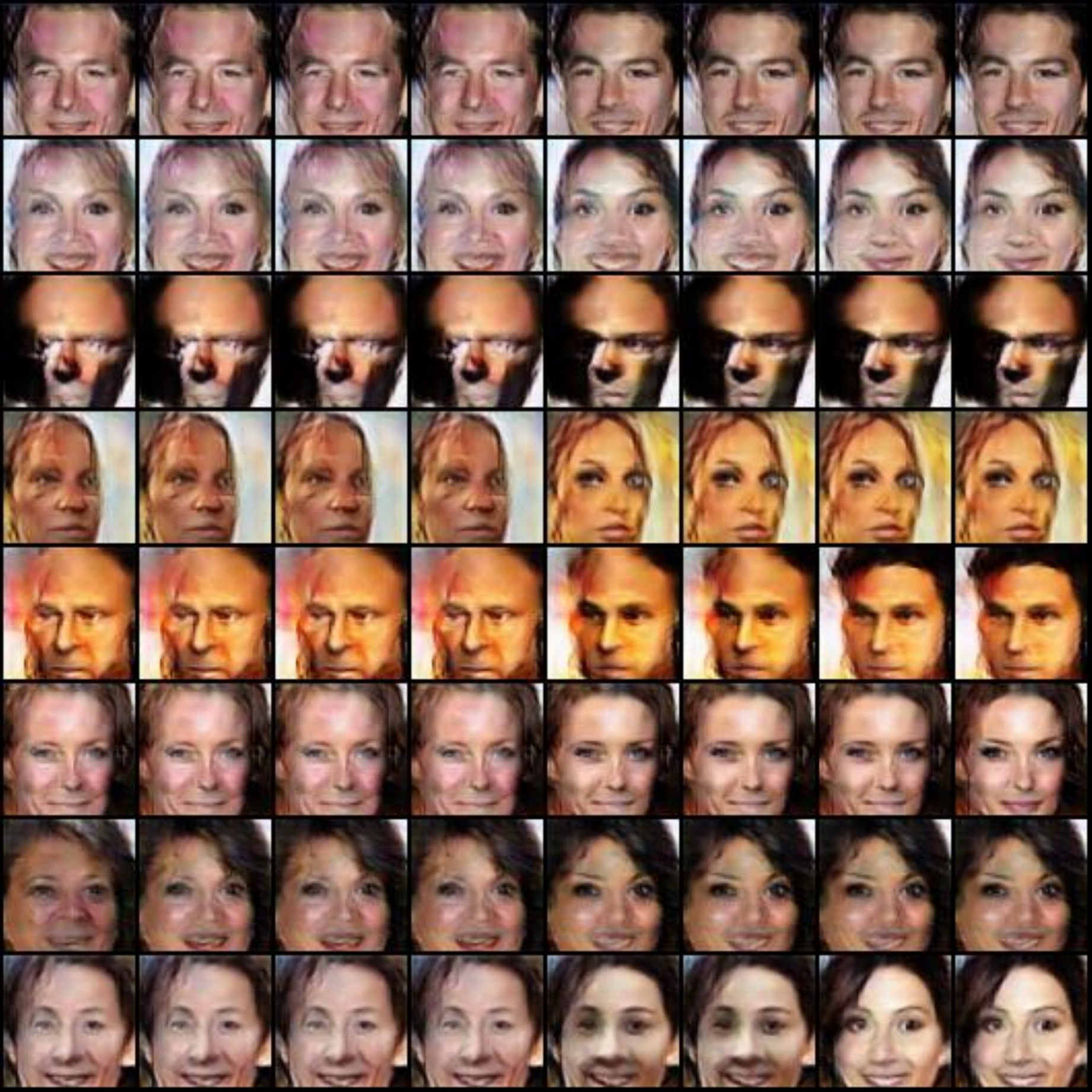}
}
\subfloat[Interpolating Eyeglasses label]
{\label{fig:eyeglasses_sweep}
\includegraphics[width=0.45\linewidth]%{./dcgan_pictures/celebA_0829_110623_no_goz/label_interpolation/45507_G_interp_Eyeglasses.pdf}
{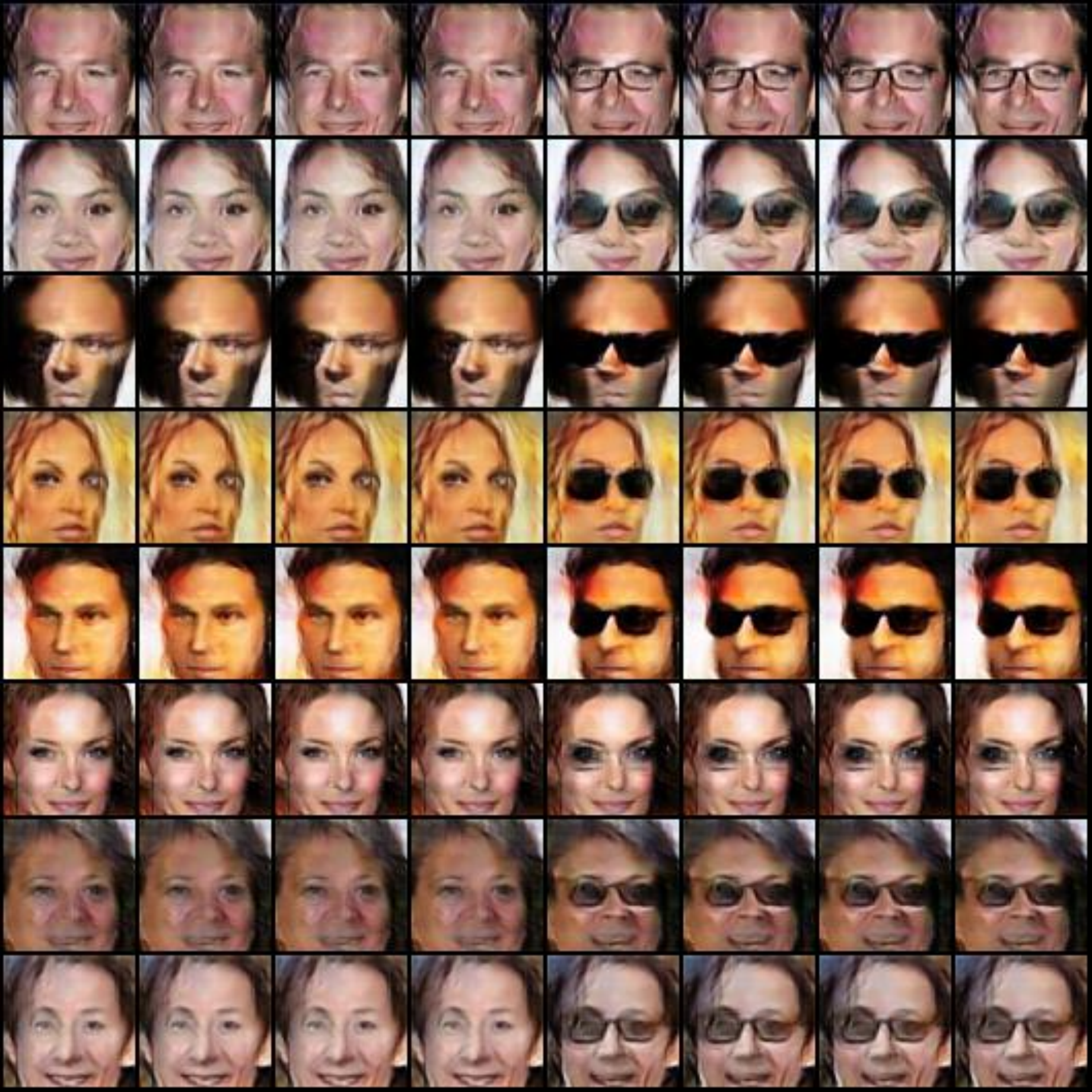}
}
\caption{The effect of interpolating a single label for CausalGAN, while keeping the noise terms and other labels fixed. }
\end{figure}

\begin{figure}
\includegraphics[width=\textwidth]%{./dcgan_pictures/celebA_0829_110623_no_goz/image_diversity/45507_G_diversity.pdf}
{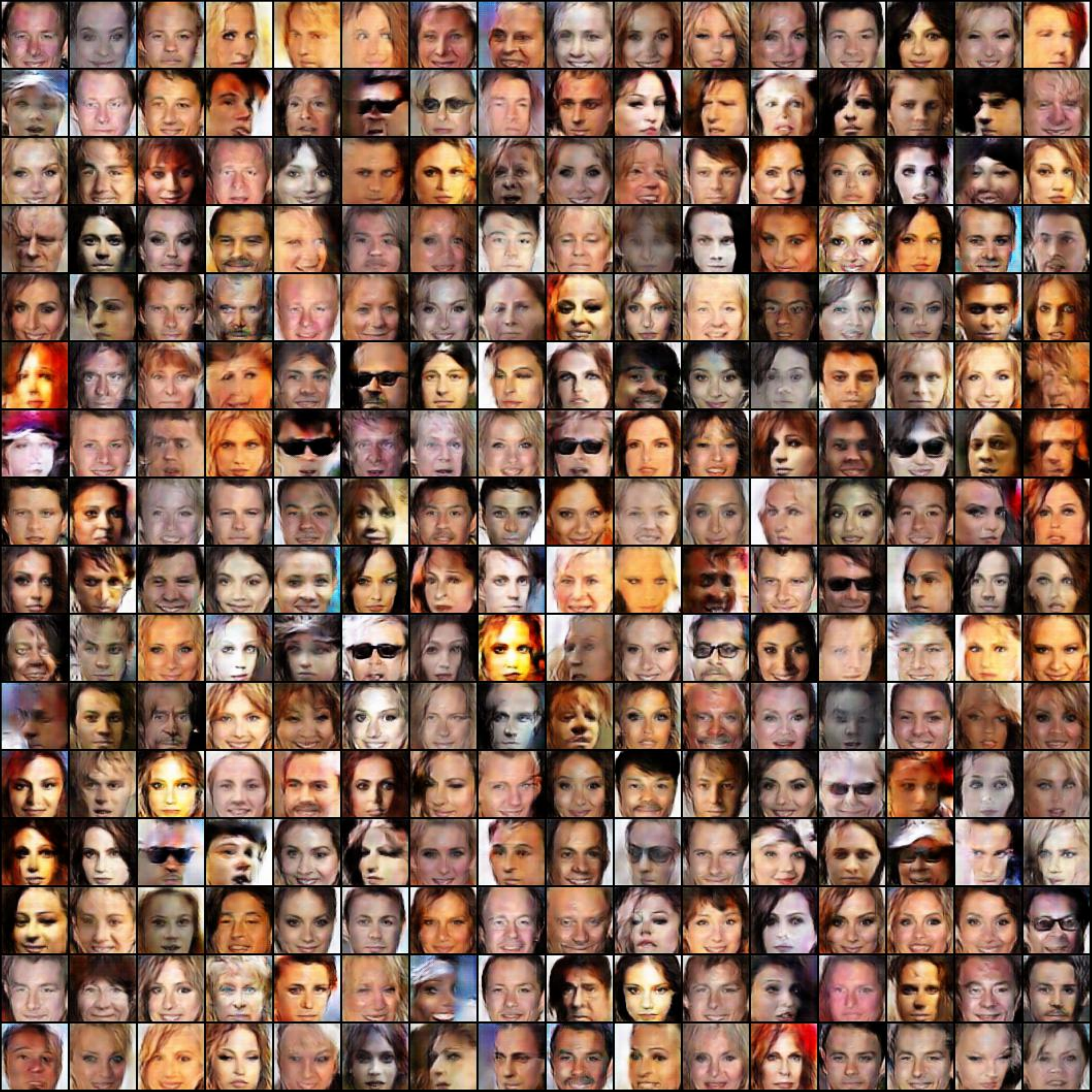}
\caption{Diversity of the proposed CausalGAN showcased with 256 samples.} 
\label{fig:dcgan_diversity}
\end{figure}

\subsection{Additional CausalBEGAN Simulations}
In this section, we provide additional simulation results for CausalBEGAN. First we show that although our third margin term $b_3$ introduces complications, it can not be ignored. Figure \ref{fig:no3margin} demonstrates that omitting the third margin on the image quality of rare labels.

Furthermore just as the setup in BEGAN permitted the definiton of a scalar "$\mathcal{M}$", which was monotonically decreasing during training, our definition permits an obvious extension $\mathcal{M}_{complete}$ (defined in \ref{eqn:M}) that preserves these properties. See Figure \ref{fig:convergence_of_causalBEGAN} to observe  $\mathcal{M}_{complete}$ decreaing monotonically during training.

We also show the conditional image generation properties of CausalBEGAN by using "label sweeps" that move a single label input from 0 to 1 while keeping all other inputs fixed (Figures \ref{fig:began_bald_sweep} -\ref{fig:began_eyeglasses_sweep} ). It is interesting to note that while generators are often implicitly thought of as continuous functions, the generator in this CausalBEGAN architecture learns a discrete function with respect to its label input parameters. (Initially there is label interpolation, and later in the optimization label interpolation becomes more step function like (not shown)). Finally, to examine the degree of mode collapse and show the image diversity, we show a random sampling of 256 images (Figure \ref{fig:began_diversity}).

\begin{figure}[ht!]
\includegraphics[width=\textwidth]%{./began_pictures/no3margin_90001_perMustache_intv.pdf}
{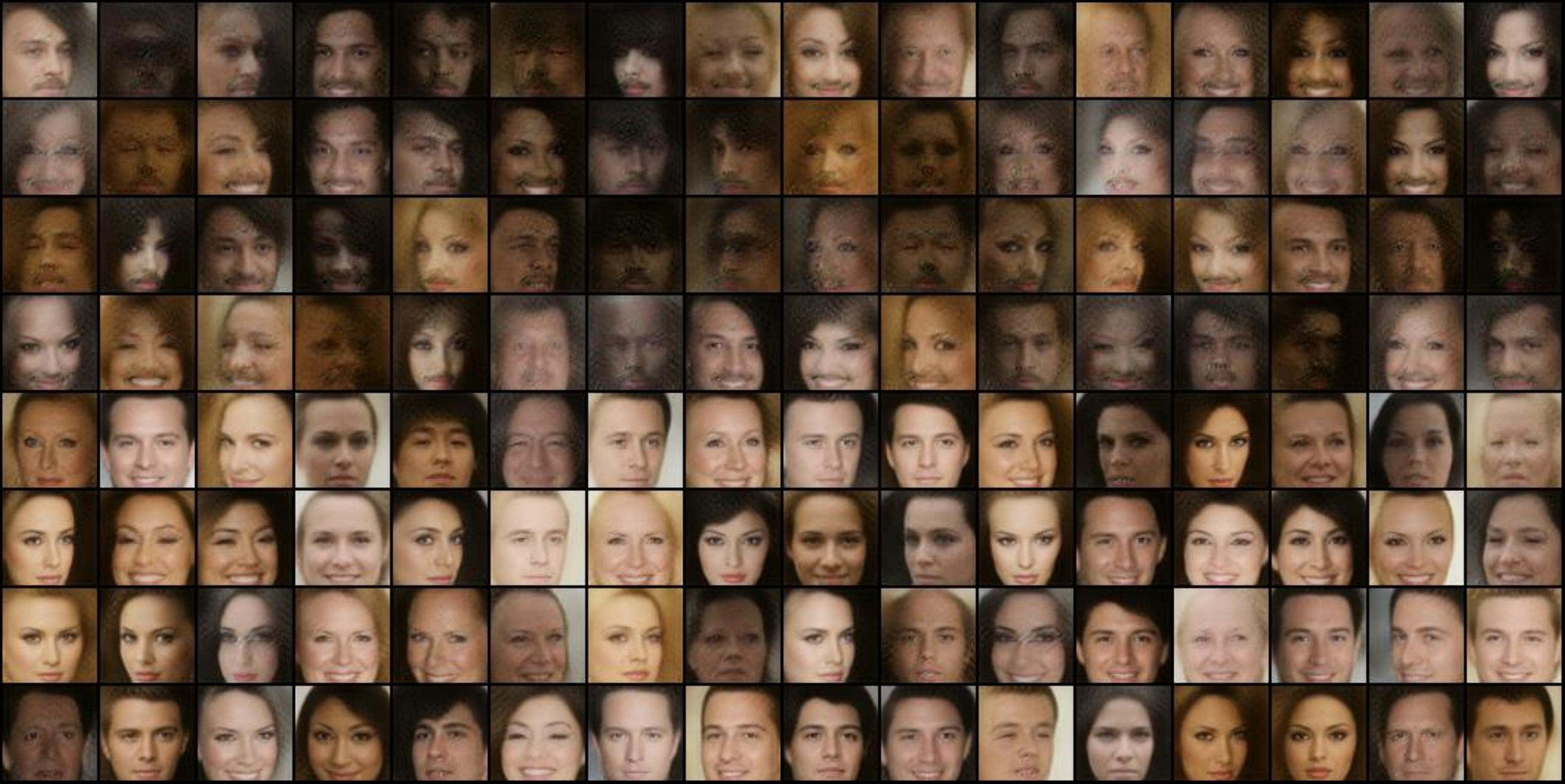}
\caption{Omitting the nonobvious margin $b_3 = \gamma_3*relu(b_1) - relu(b_2)$ results in poorer image quality particularly for rare labels such as mustache. We compare samples from two interventional distributions. Samples from $\pr{.|do(Mustache=1)}$ (top) have much poorer image quality compared to those under $\pr{.|do(Mustache=0)}$ (bottom).} 
\label{fig:no3margin}
\end{figure}

\begin{figure}[ht!]
\centering
\includegraphics[width=0.6\linewidth]%{./began_pictures/convergence.pdf}
{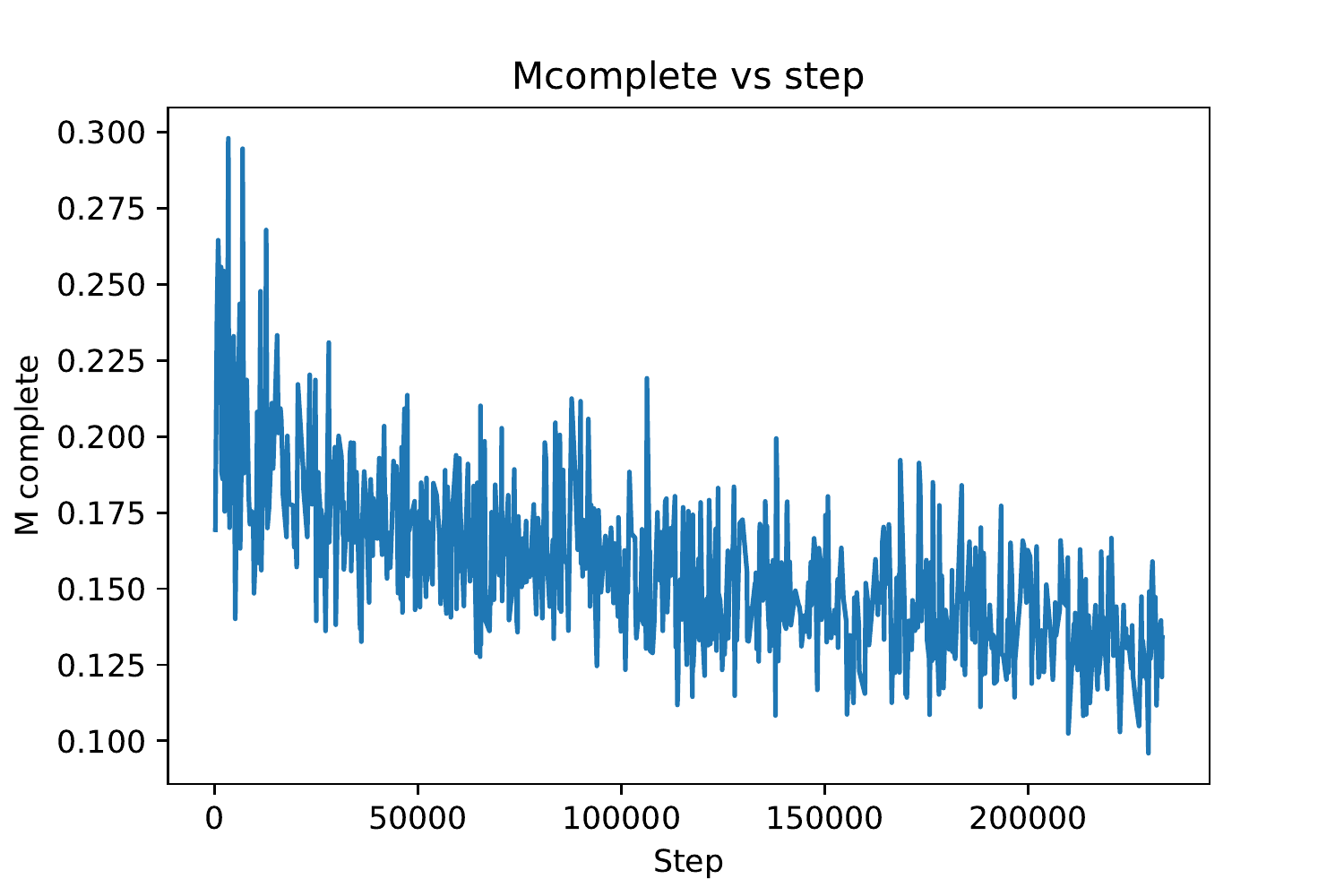}
\caption{Convergence of CausalBEGAN captured through the parameter $\mathcal{M}_{complete}$.}
\label{fig:convergence_of_causalBEGAN}
\end{figure}

\clearpage

\begin{figure}
\centering
\subfloat[Interpolating Bald label]
{\label{fig:began_bald_sweep}
\includegraphics[width=0.45\linewidth]%{./began_pictures/celebA_0810_224425_beganworking_bcg/label_interpolation/190001_G_interp_Bald.pdf}
{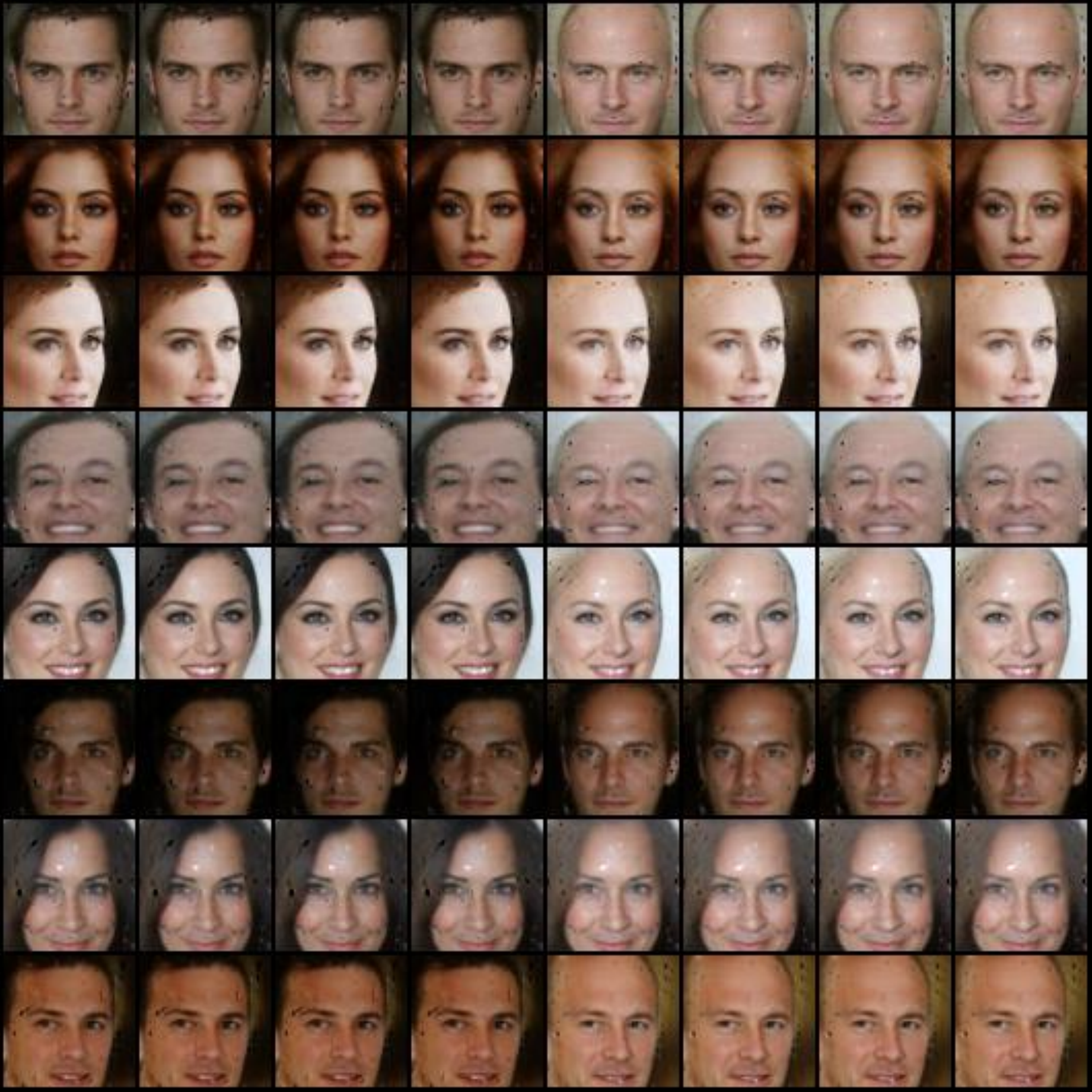}}
\subfloat[Interpolating Male label]
{\label{fig:began_male_sweep}
\includegraphics[width=0.45\linewidth]{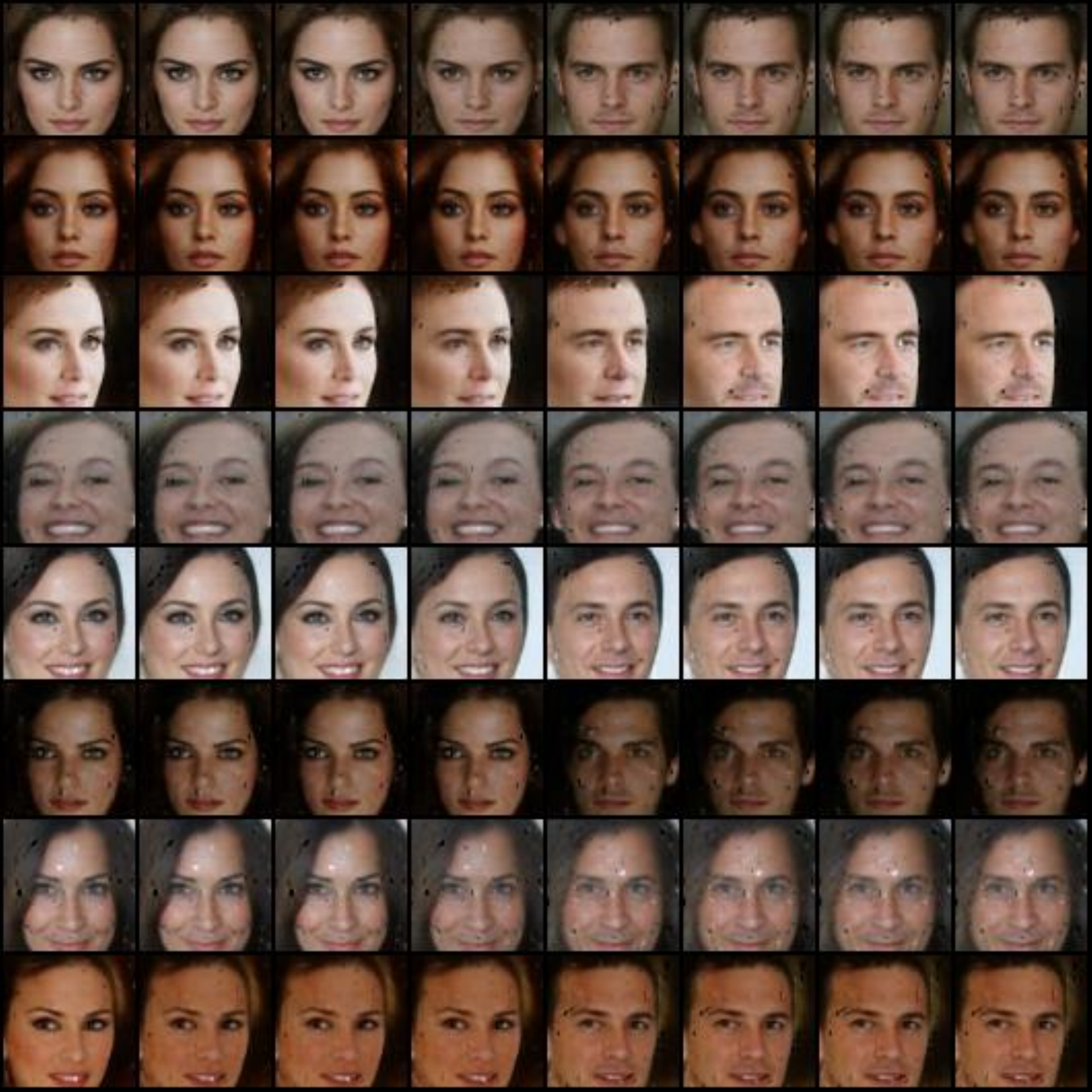}
}
\\
\subfloat[Interpolating Young label]
{\label{fig:began_young_sweep}
\includegraphics[width=0.45\linewidth]{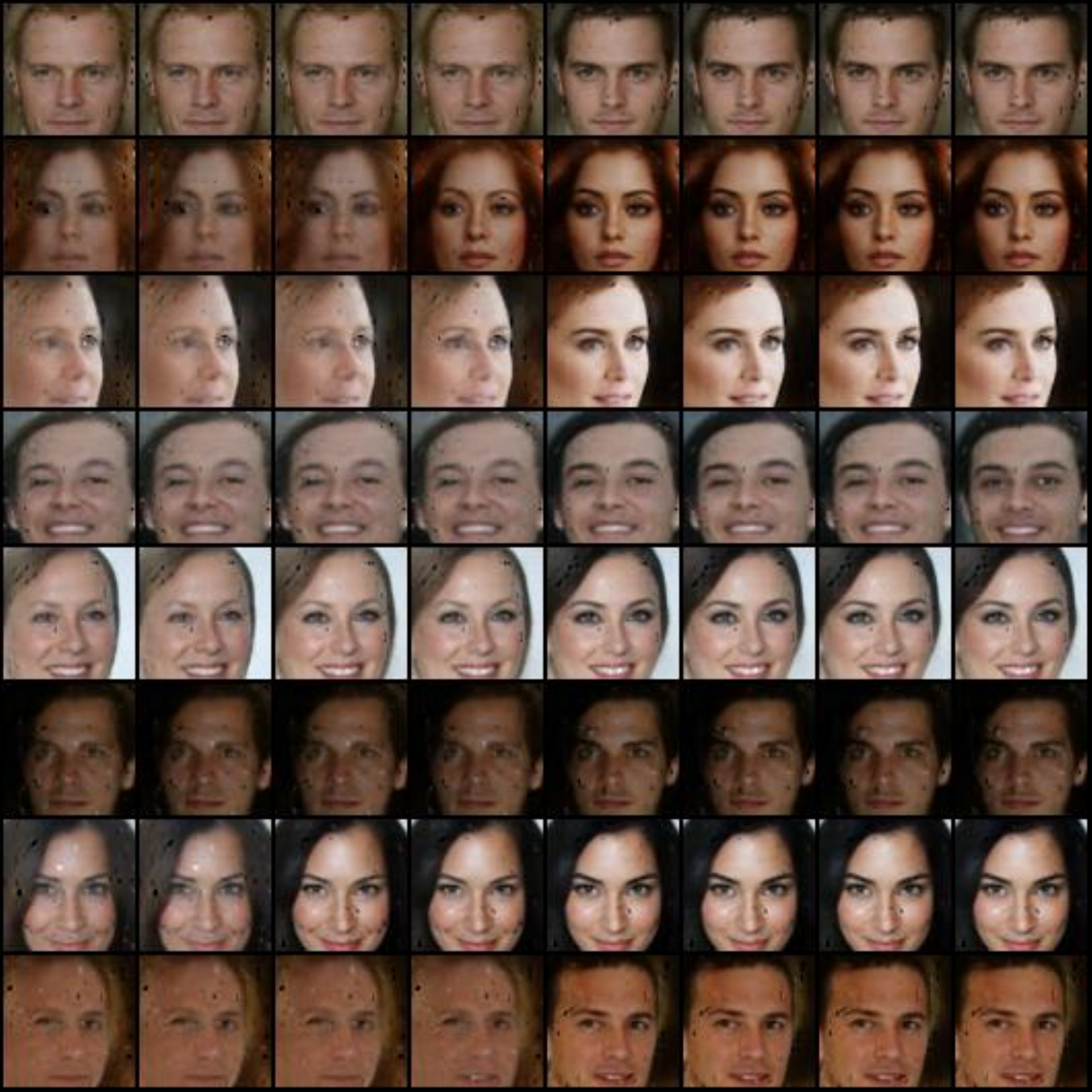}
}
\subfloat[Interpolating Eyeglasses label]
{\label{fig:began_eyeglasses_sweep}
\includegraphics[width=0.45\linewidth]{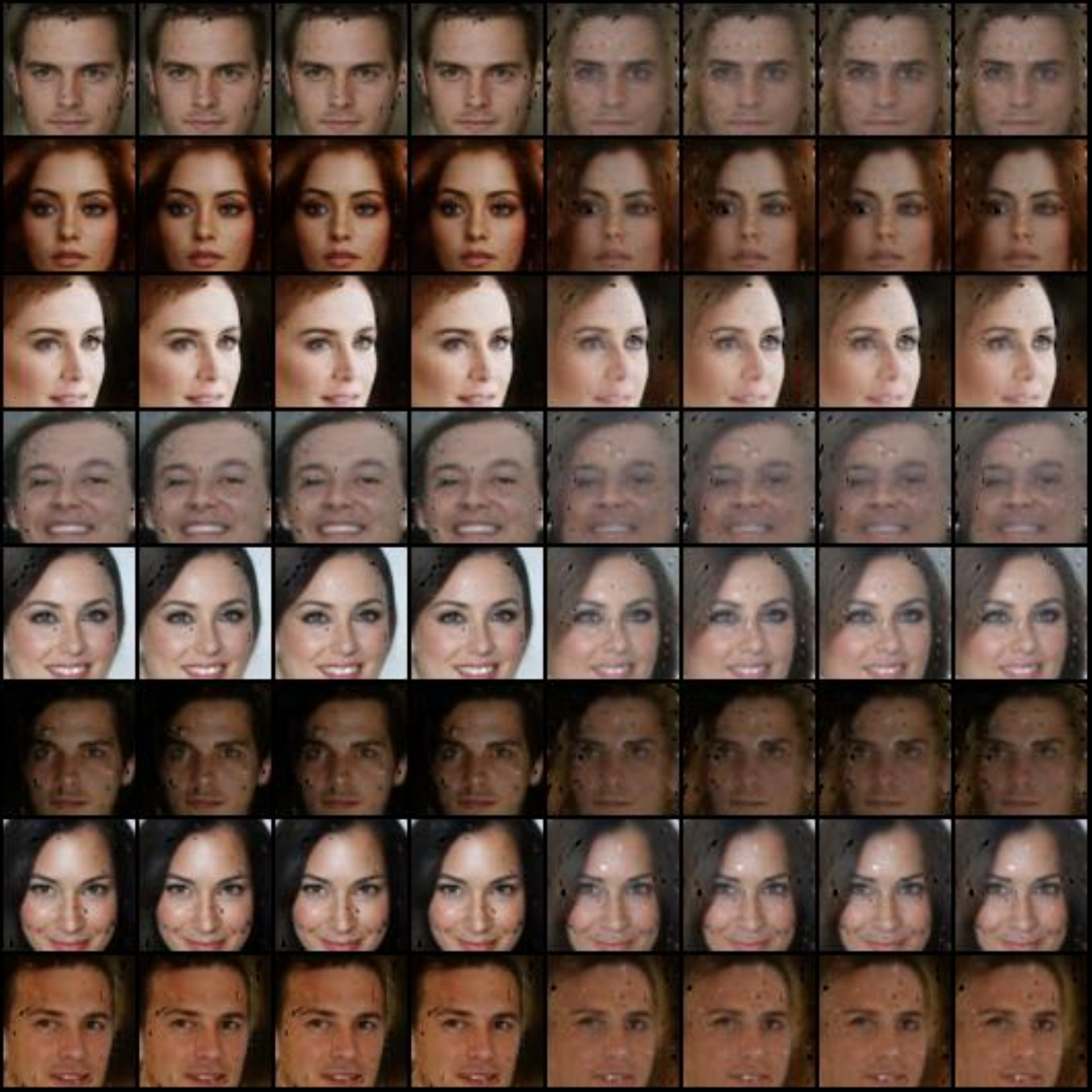}
}
\caption{The effect of interpolating a single label for CausalBEGAN, while keeping the noise terms and other labels fixed. Although most labels are properly captured, we see that eyeglasses label is not.}
\end{figure}

\begin{figure}
\includegraphics[width=\textwidth]%{./began_pictures/celebA_0810_224425_beganworking_bcg/image_diversity/190001_G_diversity.pdf}
{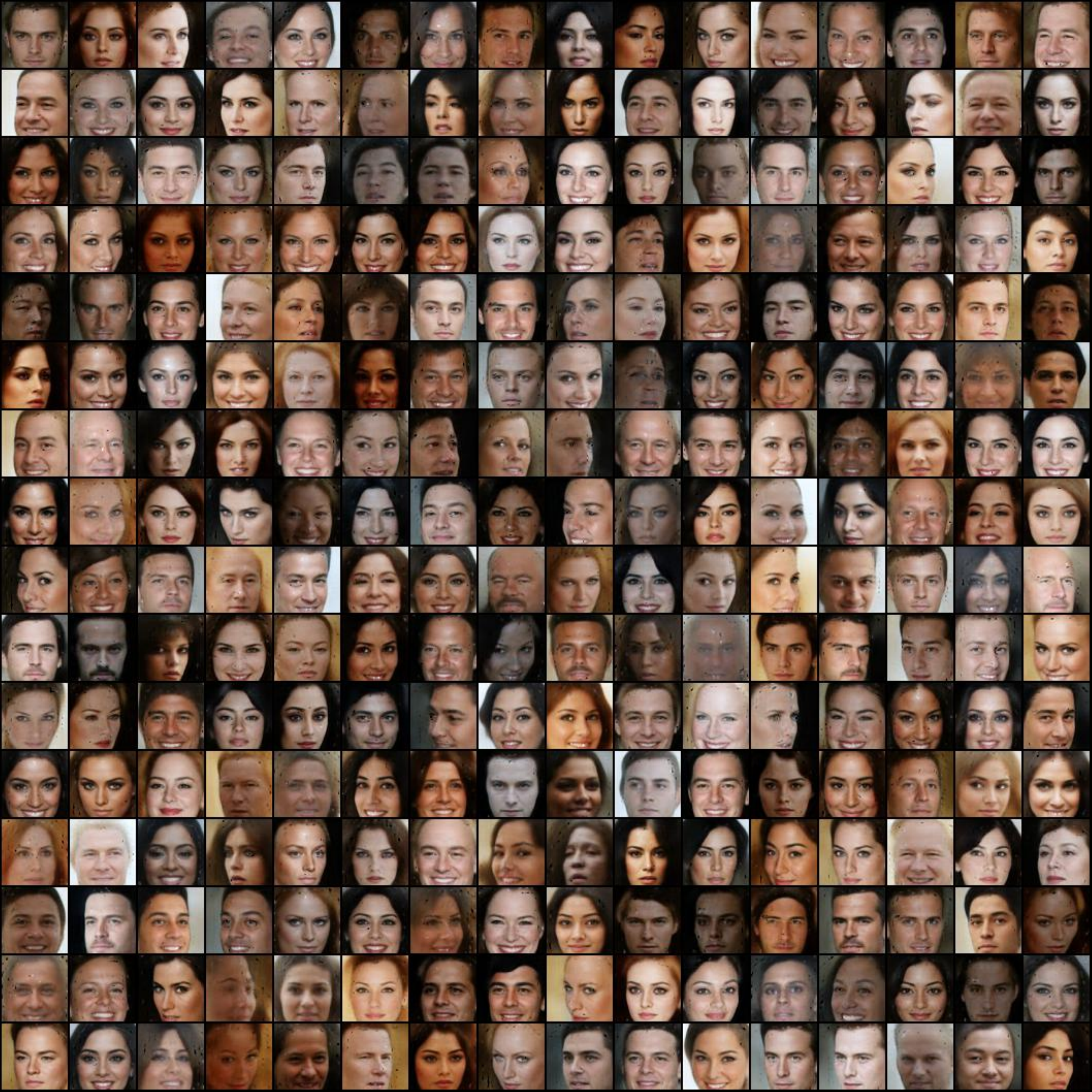}
\caption{Diversity of Causal BEGAN showcased with 256 samples.} 
\label{fig:began_diversity}
\end{figure}

\clearpage
\subsection{Directly Training Labels+Image Fails}
In this section, we present the result of attempting to jointly train an implicit causal generative model for labels and the image. This approach treats the image as part of the causal graph. It is not clear how exactly to feed both labels and image to discriminator, but one way is to simply encode the label as a constant image in an additional channel. We tried this for Causal Graph 1 and observed that the image generation is not learned (Figure \ref{channelgan:image}). One hypothesis is that the discriminator focuses on labels without providing useful gradients to the image generation.

\begin{figure}
\centering
\includegraphics[width=0.6\linewidth]{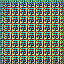}
\caption{Failed Image generation for simultaneous label and image generation after 20k steps.}
\label{channelgan:image}
\end{figure}

\end{document}